\documentclass{article}

\usepackage[margin=0.7in]{geometry}

\usepackage{amsmath,amssymb,amsfonts,mathtools}
\usepackage{amsthm}
\usepackage{upgreek}
\usepackage{bm}
\usepackage[dvipsnames]{xcolor}
\usepackage{natbib} 
\usepackage{hyperref}
\usepackage[capitalize,noabbrev]{cleveref}
\usepackage{subfig}
\usepackage{graphicx}
\usepackage{wrapfig}
\usepackage{booktabs}
\usepackage{enumitem}
\usepackage[most]{tcolorbox}

\newtcolorbox{mybox}{
  arc=2mm,
  boxrule=0.6pt,
  left=6pt,right=6pt,top=6pt,bottom=6pt
}

\usepackage{microtype}

\usepackage{amsmath,amsfonts,bm} 
\usepackage{mathrsfs}
\usepackage{enumitem}
\usepackage{stmaryrd}

\def\eqref#1{equation~\ref{#1}}

\def\1{\bm{1}}

\def\btheta{{\bm{\theta}}}

\def\beps{{\bm{\epsilon}}}

\def\rvf{{\mathbf{f}}}
\def\rvg{{\mathbf{g}}}

\def\rvu{{\mathbf{i}}}

\def\rvs{{\mathbf{s}}}

\def\rvu{{\mathbf{u}}}
\def\rvv{{\mathbf{v}}}

\def\rvx{{\mathbf{x}}}
\def\rvy{{\mathbf{y}}}
\def\rvz{{\mathbf{z}}}

\def\rmA{{\mathbf{A}}}

\def\rmI{{\mathbf{I}}}

\def\rmT{{\mathbf{T}}}

\DeclareMathAlphabet{\mathsfit}{\encodingdefault}{\sfdefault}{m}{sl}
\SetMathAlphabet{\mathsfit}{bold}{\encodingdefault}{\sfdefault}{bx}{n}

\newcommand{\E}{\mathbb{E}}

\newcommand*\diff{\mathop{}\!\mathrm{d}}

\usepackage{tikz}

\def\btheta{{\bm{\theta}}}

\def\beps{{\bm{\epsilon}}}

\usepackage{xcolor}

\definecolor{brightmaroon}{rgb}{0.76, 0.13, 0.28}
\definecolor{brown(web)}{rgb}{0.65, 0.16, 0.16}
\definecolor{customblue}{HTML}{5D94D1}

\definecolor{Pink}{rgb}{1.0, 0.44, 0.37}

\definecolor{PersB}{HTML}{FF7F0E} 
\definecolor{PersA}{HTML}{1F77B4} 
\definecolor{PersC}{HTML}{2CA02C} 

\theoremstyle{plain}
\newtheorem{theorem}{Theorem}
\newtheorem{lemma}{Lemma}
\newtheorem{proposition}{Proposition}
\newtheorem{assumption}{Assumption}
\crefname{assumption}{assumption}{assumptions}
\Crefname{assumption}{Assumption}{Assumptions}
\theoremstyle{definition}

\newtheorem{corollary}{Corollary}[theorem]

\theoremstyle{remark}

\title{\textbf{A Unified View of Score-Based and Drifting Models}

}
\author{
Chieh-Hsin Lai$^{1}$\thanks{Corresponding author: \href{mailto:chieh-hsin.lai@sony.com}{chieh-hsin.lai@sony.com}} \,\,\,\,
Bac Nguyen$^{1}$ \,\,\,\,\\
Naoki Murata$^{1}$ \,\,\,\,
Yuhta Takida$^{1}$ \,\,\,\,
Toshimitsu Uesaka$^{1}$\,\,\,\,\\
Yuki Mitsufuji$^{1,2}$ \,\,\,\,
Stefano Ermon$^{3}$\thanks{Equal supervision.} \,\,\,\,
Molei Tao$^{4}$\footnotemark[\value{footnote}]
}

\date{
$^{1}$Sony AI, $^{2}$Sony Group Corporation, $^{3}$Stanford University, $^{4}$Georgia Tech
}
\begin{document}
\maketitle

\begin{abstract}
Drifting models train one-step generators by optimizing a kernel-induced mean-shift discrepancy between the data and model distributions, with Laplace kernels used by default in practice. At each point, this discrepancy compares the kernel-weighted displacement toward nearby data samples with the corresponding displacement toward nearby model samples, thereby defining a transport direction for generated samples. In this paper, we show that drifting is more closely connected to score-based generative modeling than it may first appear, establishing a precise link to the score-matching principle underlying diffusion models. For Gaussian kernels, the population mean-shift field exactly equals the difference between the scores (i.e., the gradient-log-densities) of the Gaussian-smoothed data and model distributions. This identity follows from Tweedie’s formula, which links the score of a Gaussian-smoothed density to its conditional mean, and implies that Gaussian-kernel drifting is exactly a score-matching objective on smoothed distributions. More generally, we derive an exact decomposition for radial kernels in which mean shift equals a score-based field plus a residual term. For the practical Laplace kernel, we further show theoretically and empirically that this residual is negligible in high dimension, implying that the transport field used in practice is nearly score-based. Our results reveal a structural connection to diffusion models: both methods use score-mismatch transport directions, but drifting realizes the score nonparametrically through kernel-based estimates, whereas diffusion models learn it parametrically with neural networks.
\end{abstract}


\section{Introduction}
Diffusion and score-based generative models~\cite{sohl2015deep,ho2020denoising,song2019generative,song2020score,lai2025principles} generate data by transporting a simple noise distribution to the data distribution through many small steps, typically formulated as a time-indexed stochastic process or its ODE counterpart. This framework yields high sample quality, but often makes inference expensive because generation requires many neural network evaluations. Motivated by the need for faster sampling, recent work has therefore explored one-step or few-step generators~\cite{song2023consistency,kim2023consistency,geng2025mean,boffi2024flow,hu2025cmt} that directly push forward noise to data.

Drifting models~\cite{deng2026drifting} offer a different perspective on one-step generation. Instead of introducing a time-indexed noising process and learning its reverse dynamics, drifting fixes a kernel, Laplace by default, and constructs a transport rule directly from data samples. The central object is a local displacement field: at each location, the method aggregates nearby samples with kernel weights and computes their weighted average offset. This yields a mean-shift-type update~\cite{comaniciu2002mean} that moves points toward regions of higher sample density. Evaluating this local displacement across the ambient space defines a vector field, and generation proceeds by transporting samples along this field at one or a few kernel scales; see \Cref{fig:2d-vectors-gauss}-(a) for intuition.

This displacement field is closely connected to the \emph{score function} that underlies modern diffusion models.
The score is a log-density gradient and points toward higher-density regions.
In particular, the \emph{score-mismatch} between data and model defines a transport direction that steers model samples toward the data; \Cref{fig:2d-vectors-gauss}-(b) visualizes this behavior. Score-based diffusion models~\cite{song2019generative,song2020score} learn scores via score matching~\cite{hyvarinen2005estimation,vincent2011connection,lyu2012interpretation}, most often in a \emph{forward Fisher} form that averages the score-mismatch under the data distribution (which encourages mode coverage).
They then generate samples by applying denoising updates across noise levels.

In this article, we do not propose a new algorithm; instead, we show that drifting is fundamentally linked to score-based modeling on kernel-smoothed distributions. In the Gaussian case, the population mean-shift field coincides exactly with a variance-scaled \emph{score-mismatch} field between the Gaussian-smoothed data and model distributions. This identity follows from Tweedie's formula~\cite{efron2011tweedie}, which relates the conditional mean under additive Gaussian noise to the score of the corresponding smoothed marginal. The same principle underlies the denoising view of diffusion models, where the optimal denoiser is determined by the score of a Gaussian-smoothed distribution. We validate and visualize this exact correspondence in \Cref{fig:2d-vectors-gauss}-(c,d): up to the variance scaling, the mean-shift and score-mismatch fields coincide, making the bridge to score-based modeling exact at the level of the transport field.
\begin{figure}[!t]
    \vskip -0.3cm
    \centering
    \includegraphics[width=\linewidth]{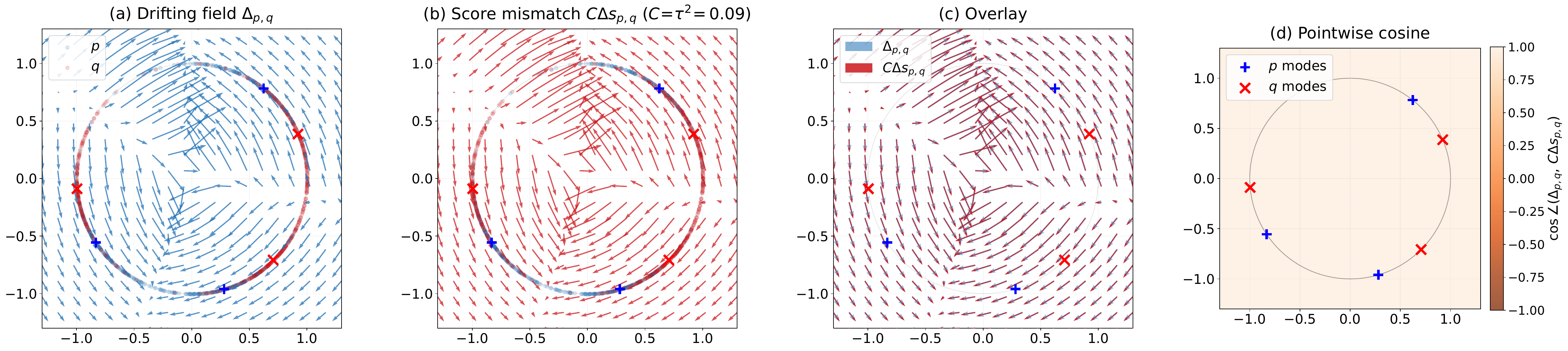}
\caption{\footnotesize{\textbf{2D visualization: Gaussian drifting exactly matches the score-mismatch direction.}
Under Gaussian smoothing, the mean-shift drifting field $\Delta_{p,q}$ in (a) is exactly parallel to the score-mismatch field $\Delta \rvs_{p,q}$ in (b), as shown in \Cref{thm:gaussian_meanshift_score_matching}; see \Cref{subsec:mean-shift-kernel-score} for the definitions. Panels (c,d) visualize this exact alignment. Both fields are estimated from finite samples using the same kernel-based Monte Carlo procedure. Here $p$ is a three-mode distribution on the unit circle, and $q$ is a fixed model distribution obtained by rotating the mode centers of $p$ by a fixed angle.}}
    \label{fig:2d-vectors-gauss}
    \vskip -0.5cm
\end{figure}
The corresponding drifting objective is a score-matching-style objective~\cite{weber2023score}, but in \emph{reverse Fisher} form: the pointwise score-mismatch is averaged under the \emph{model} distribution rather than the data distribution. This weighting encourages correcting the score field where the current model places mass (for example, suppressing spurious mass), and is complementary to forward Fisher, which emphasizes matching scores on data regions and thus promotes coverage under the data distribution.

This viewpoint also brings drifting closer to diffusion models at the objective level. Both methods are driven by score-mismatch transport directions, but they differ in how the score signal is obtained: drifting constructs it nonparametrically from local kernel neighborhoods, whereas diffusion models learn it parametrically with a neural network.

Beyond the Gaussian setting, we show that general radial kernels, including Laplace, admit an exact decomposition into a pre-conditioned smoothed-score term plus a covariance residual term that captures local neighborhood geometry. This result places drifting within a principled pre-conditioned score-matching framework and isolates precisely the correction introduced by non-Gaussian kernels.

Specializing this decomposition to the Laplace kernel used in standard drifting implementations, we show that drifting is approximately score matching in the practically relevant high-dimensional regime of modern representation spaces, including both pre-trained feature embeddings and latent-then-feature representations. As the embedding dimension $D$ grows, the mean-shift field approximates a scaled score-mismatch field and its direction becomes nearly parallel to the score-mismatch direction, with polynomially decaying error in both cases. Our experiments support this theory by verifying the predicted decay in $D$, showing that the covariance residual is typically small, and demonstrating that Gaussian and Laplace kernels yield broadly comparable generation quality.

For clarity, we summarize the main messages below:
\begin{mybox}
\begin{itemize}
  \item \textbf{Gaussian Kernel:} 
  \begin{itemize}
      \item[$\diamond$] Drifting's mean-shift direction 
$=$ score-mismatch direction (via Tweedie)
\item[$\diamond$] Drifting objective $=$ score-matching-style objective.
  \end{itemize}
  \item \textbf{General Radial Kernel:} mean-shift = rescaled score + residual
  \item \textbf{Laplace Kernel:} When dimension $D$ is large,
\[
\begin{aligned}
\text{residual} \approx 0 &\Rightarrow \text{drifting mean-shift direction } \propto \text{ score-mismatch direction, up to } \mathcal O(D^{-1}) \text{ error} \\
&\Rightarrow \text{objective, semi-gradient, and optima almost align with score-based counterparts.}
\end{aligned}
\]
\end{itemize}
\end{mybox}
Additional regimes where drifting becomes score-like, including low-temperature kernels and
implementation-aligned finite-sample settings, are discussed in \Cref{app:more-regimes}.

\section{Preliminaries}\label{sec:prelim}
\subsection{Score-Based Generative Models and Diffusion Models}
\paragraph{Score Functions and Fisher Divergences.}
Let $p$ be a distribution on $\mathbb R^D$ with density, still denoted by $p$. Its \emph{score function} is $\mathbf s_p(\rvx):=\nabla_{\rvx}\log p(\rvx)$. Given two distributions $p$ and $q$, we define the \emph{forward} and \emph{reverse} Fisher divergences by
\begin{equation*}
\mathcal D_{\mathrm{fF}}(p\|q)
:=
\mathbb E_{\rvx\sim p}\big[\|\mathbf s_p(\rvx)-\mathbf s_q(\rvx)\|_2^2\big],
\qquad
\mathcal D_{\mathrm{rF}}(p\|q)
:=
\mathbb E_{\rvx\sim q}\big[\|\mathbf s_p(\rvx)-\mathbf s_q(\rvx)\|_2^2\big].
\end{equation*}
 The two Fisher divergences share the same pointwise score-mismatch, $\|\mathbf s_p-\mathbf s_q\|_2^2$, but differ in the weighting measure: forward Fisher averages under $p$, whereas reverse Fisher averages under $q$. Consequently, forward Fisher emphasizes score accuracy where data concentrate and thus promotes mode coverage, while reverse Fisher emphasizes score accuracy where the model places mass and can suppress spurious modes. Standard score matching~\cite{hyvarinen2005estimation,song2020score} fits $q$ by minimizing the forward Fisher divergence, avoiding explicit normalizing constants while directly matching the score field.

\paragraph{Diffusion Models.}
A diffusion forward process takes the Gaussian form
$\rvx_t=\alpha_t\rvx_0+\sigma_t\beps$,
where $\rvx_0\sim p_{\mathrm{data}}$, $\beps\sim\mathcal N(\mathbf 0,\mathbf I)$, and $(\alpha_t,\sigma_t)$ is a fixed noise schedule satisfying $\alpha_t\ge 0$ and $\sigma_t>0$. Let $p_t$ denote the marginal density of $\rvx_t$. Then $p_t$ is the Gaussian smoothing of $p_{\mathrm{data}}$: $p_t(\rvx)
=
\int \mathcal N(\rvx;\alpha_t\rvx_0,\sigma_t^2\mathbf I)\,p_{\mathrm{data}}(\rvx_0)\,\mathrm d\rvx_0$. Its time-dependent score function is $\mathbf s_p(\rvx,t):=\nabla_{\rvx}\log p_t(\rvx)$. Diffusion models approximate the oracle score $\mathbf s_p(\rvx_t,t)$ by a neural network $\rvs_\btheta(\rvx_t,t)$ trained with the time-dependent score-matching objective~\cite{vincent2011connection,song2020score} $\mathbb E_t\,\mathbb E_{\rvx_t\sim p_t}\Big[\|\rvs_\btheta(\rvx_t,t)-\mathbf s_p(\rvx_t,t)\|_2^2\Big]$, which is precisely a forward Fisher objective along the diffusion marginals.

\subsection{Drifting Model's Fixed-Point Regression}
\label{sec:template}

In this section, we introduce a fixed-point regression framework with a pre-designed drift field, motivated by drifting models. We then present two realizations of this field: kernel-induced mean-shift and kernel-induced score-mismatch.

\paragraph{Training Objective of Drifting Model.}
Let $p:=p_{\mathrm{data}}$ denote the data distribution\footnote{
Although we present the construction in raw data space, it extends verbatim to any fixed pre-trained feature space. Given a frozen embedding $\Psi:\mathcal X\to\mathbb R^D$, let $p^\Psi:=\Psi_\# p$ and $q^\Psi:=\Psi_\# q$. Applying the same kernel, mean-shift field, and discrepancy construction to $p^\Psi$ and $q^\Psi$ gives the corresponding feature-space quantities, so all results below apply equally after replacing $\rvx$ by $\Psi(\rvx)$ and $(p,q)$ by $(p^\Psi,q^\Psi)$.
}.
We consider a pushforward generator
$\rvx=\rvf_\btheta(\beps)$ with $\beps\sim p_{\mathrm{prior}}$, where $p_{\mathrm{prior}}:=\mathcal N(\mathbf 0,\rmI)$, inducing the model distribution: $q_\btheta := (\rvf_\btheta)_\# p_{\mathrm{prior}}$. When the dependence on $\btheta$ is clear, we simply write $q$. The goal of one-step generation is to choose $\rvf_\btheta$ so that $q_\btheta \approx p_{\mathrm{data}}$ under a chosen statistical divergence, such as KL or Fisher divergence. Throughout, $p$ (or $p_{\mathrm{data}}$) denotes the data distribution, and $q$ (or $q_\btheta$) denotes the model distribution.

Drifting model starts from a one-step transport operator $\mathcal U_{p,q}(\mathbf x)=\mathbf x+\Delta_{p,q}(\mathbf x)$, where $\Delta_{p,q}(\mathbf x)$ is a designed drift field on $\mathbb R^D$ that measures discrepancy between $p$ and $q$, and satisfies the equilibrium condition $\Delta_{p,p}(\mathbf x)\equiv \mathbf 0$. Given $\mathcal U_{p,q}$, the drifting-model template realizes this transport through a fixed-point regression objective with a stop-gradient target:
\begin{equation}
\label{eq:fp_loss}
\min_\btheta~\mathcal L_{\mathrm{drift}}(\btheta)
:=
\mathbb E_{\beps\sim p_{\mathrm{prior}}}\Big[
\big\|\rvf_\btheta(\beps)-\mathrm{sg}\big(\mathcal U_{p,q_\btheta}(\rvf_\btheta(\beps))\big)\big\|_2^2
\Big].
\end{equation}
Intuitively, at each iteration the method first computes a frozen transported sample
$\tilde{\mathbf x}=\mathcal U_{p,q_\btheta}(\mathbf x)$
using the current model, and then fits the next generator to regress onto this transported cloud.

\paragraph{Objective-Level Equivalence.}
For any fixed $\btheta$, the stop-gradient freezes the target during backpropagation, but the value of the loss simplifies to
\begin{align}\label{eq:value_equiv}
\begin{aligned}
\mathcal L_{\mathrm{drift}}(\btheta)
&=
\mathbb E_{\beps}\Big[
\big\|\rvf_\btheta(\beps)-\mathrm{sg}\big(\rvf_\btheta(\beps)+\Delta_{p,q_\btheta}(\rvf_\btheta(\beps))\big)\big\|_2^2
\Big]
\\
&=
\mathbb E_{\beps}\big[\|\Delta_{p,q_\btheta}(\rvf_\btheta(\beps))\|_2^2\big]
=
\mathbb E_{\mathbf x\sim q_\btheta}\big[\|\Delta_{p,q_\btheta}(\mathbf x)\|_2^2\big].
\end{aligned}
\end{align}
Thus, at the objective level, \Cref{eq:fp_loss} is exactly the squared-norm functional of the field $\Delta_{p,q_\btheta}$ under $q_\btheta$. Since these expressions agree for every $\btheta$, they define the same objective function and therefore share the same global minimizers and optimal value.

\paragraph{Gradient-Level Remark.}
Although \Cref{eq:value_equiv} identifies the objective as $\mathbb E_{q_\btheta}\|\Delta_{p,q_\btheta}\|^2$, the stop-gradient solver does not backpropagate through the dependence of $\Delta_{p,q_\btheta}$ on $q_\btheta$. Differentiating \Cref{eq:fp_loss} treats the target as frozen and yields the semi-gradient
\begin{equation*}
\nabla_\btheta \mathcal L_{\mathrm{drift}}(\btheta)
=
-2\,\mathbb E_{\beps\sim p_{\mathrm{prior}}}\Big[
\mathbf J_{\rvf_\btheta}(\beps)^\top \Delta_{p,q_\btheta}(\mathbf x)
\Big],
\end{equation*}
where $\mathbf x=\rvf_\btheta(\beps)$ and $\mathbf J_{\rvf_\btheta}(\beps)$ denotes the Jacobian of $\rvf_\btheta$ with respect to $\btheta$.

Equivalently, the stop-gradient objective implements a transport--then--projection iteration~\cite{weber2023score}. At iteration $i$, the
current generator $\rvf_{\btheta_i}$ induces $q_{\btheta_i}$ and produces samples
$\rvx_i=\rvf_{\btheta_i}(\beps)$. We first transport these samples using the current drift field:
\[
\widetilde{\rvx}_i
=
\rvx_i+\Delta_{p,q_{\btheta_i}}(\rvx_i).
\]
We then fit the next generator to this transported sample cloud by regression:
\[
\btheta_{i+1}
\approx
\arg\min_{\btheta}
\mathbb E_{\beps\sim p_{\mathrm{prior}}}
\Big[
\big\|
\rvf_{\btheta}(\beps)-\mathrm{sg}(\widetilde{\rvx}_i)
\big\|_2^2
\Big].
\]
The stop-gradient makes the transported samples fixed targets during the regression step. Thus, drifting
alternates between moving the current model samples by a distribution-comparison field and projecting
the moved cloud back onto the generator family. The semi-gradient above is precisely the infinitesimal
form of this transport--then--projection update.

\subsection{Drifting Model's Mean-Shift versus Kernel-Induced Score Function}
\label{subsec:mean-shift-kernel-score}

\paragraph{Mean-Shift Operator.}
Drifting models realize the update field $\Delta_{p,q}(\mathbf x)$ through a kernel-induced mean-shift~\citep{comaniciu2002mean}. Let $k(\mathbf x,\mathbf y)\ge 0$ be a similarity kernel, and let $\uppi$ be any distribution on $\mathbb R^D$. Define the kernel-weighted barycenter
$\boldsymbol\mu_{\uppi,k}(\mathbf x)
:=
\frac{\mathbb E_{\mathbf y\sim \uppi}[k(\mathbf x,\mathbf y)\,\mathbf y]}
{\mathbb E_{\mathbf y\sim \uppi}[k(\mathbf x,\mathbf y)]}$,
and the associated mean-shift direction
\begin{equation*}
\mathbf V_{\uppi,k}(\mathbf x)
:=
\boldsymbol\mu_{\uppi,k}(\mathbf x)-\mathbf x
=
\frac{\mathbb E_{\mathbf y\sim \uppi}[k(\mathbf x,\mathbf y)\,(\mathbf y-\mathbf x)]}
{\mathbb E_{\mathbf y\sim \uppi}[k(\mathbf x,\mathbf y)]}.
\end{equation*}
Intuitively, $\boldsymbol\mu_{\uppi,k}(\rvx)$ is a kernel-weighted local average, so $\mathbf V_{\uppi,k}(\rvx)$ moves $\rvx$ toward regions better supported by $\uppi$ under the kernel.

For $\uppi\in\{p,q\}$, we write $\mathbf V_{p,k}(\mathbf x)$ and $\mathbf V_{q,k}(\mathbf x)$ for the corresponding data and model mean-shift fields. The resulting discrepancy field is
\begin{equation*}
\Delta_{p,q}(\mathbf x)
:=
\eta\big(\mathbf V_{p,k}(\mathbf x)-\mathbf V_{q,k}(\mathbf x)\big),
\end{equation*}
where $\eta>0$ is a step size. Intuitively, the kernel induces a local notion of support around each point $\mathbf x$. The field $\mathbf V_{p,k}(\mathbf x)$ pulls $\mathbf x$ toward nearby regions favored by the data, while $\mathbf V_{q,k}(\mathbf x)$ captures the analogous pull from the current model. Their difference therefore defines a local correction direction: it attracts samples toward regions that are under-supported by the model relative to the data, while counteracting motion toward regions where the model already places excessive mass.

\paragraph{Kernel-Induced Score Function.}
The same kernel smoothing also induces a complementary score field. Consider the kernel-smoothed mass profile
\begin{equation*}
\uppi_k(\mathbf x)
:=
\mathbb E_{\mathbf y\sim \uppi}[k(\mathbf x,\mathbf y)]
=
\int k(\mathbf x,\mathbf y)\,\uppi(\mathbf y)\,\mathrm d\mathbf y.
\end{equation*}
In the translation-invariant case $k(\mathbf x,\mathbf y)=\kappa(\mathbf x-\mathbf y)$, this reduces to convolution, $\uppi_k=\uppi * \kappa$; hereafter, we do not distinguish between $k$ and $\kappa$. The associated kernel-induced score is
$\mathbf s_{\uppi,k}(\mathbf x):=\nabla_{\mathbf x}\log \uppi_k(\mathbf x)$.
We allow $k$ to be unnormalized, but this causes no ambiguity, since both the score and the mean-shift field are invariant under positive rescaling of the kernel.


\section{Bridging Drifting Models and Score-Based Models}
\label{sec:instantiations}
This section develops the central theoretical message that drifting is more closely connected to score-based modeling than it may initially appear. We show that the mean-shift  $\mathbf V_{p,k}(\mathbf x)-\mathbf V_{q,k}(\mathbf x)$ is intrinsically linked to the score-mismatch $\mathbf s_{p,k}(\mathbf x)-\mathbf s_{q,k}(\mathbf x)$: for Gaussian kernels, mean-shift is exactly a score-mismatch up to a multiplicative constant (\Cref{sec:drifting}); for general radial kernels, this link extends via an exact score-mismatch-plus-residual decomposition (\Cref{sec:radial}); and for the Laplace kernel used in practice, the residual is small in high dimension (\Cref{sec:laplace}).

\subsection{Gaussian Kernels: Mean-Shift is Exactly a Score-Mismatch}
\label{sec:drifting}

We first specialize $k$ to the Gaussian kernel with temperature $\tau>0$:
\begin{equation*}
k_\tau(\rvx,\rvy)=\exp\Big(-\frac{\|\rvx-\rvy\|_2^2}{2\tau^2}\Big).
\end{equation*}
For any distribution $\uppi$, the associated smoothing field is
$\uppi_{k_\tau}(\rvx)
=
\mathbb E_{\rvy\sim \uppi}[k_\tau(\rvx,\rvy)]
=
(\uppi * k_\tau)(\rvx)$,
which we abbreviate as $\uppi_\tau(\rvx)$. The corresponding kernel-induced score is therefore:
$\mathbf s_{\uppi,k_\tau}(\rvx)
=
\nabla_{\rvx}\log \uppi_{k_\tau}(\rvx)
=
\nabla_{\rvx}\log \uppi_\tau(\rvx)$,
which we abbreviate as $\mathbf s_{\uppi,\tau}(\rvx)$.

For Gaussian kernels, the mean-shift field is exactly the smoothed score up to the factor $\tau^2$, so the drifting discrepancy field is exactly a score-mismatch field between the Gaussian-smoothed data and model distributions. Consequently, at the objective level, the drifting loss in \Cref{eq:value_equiv} is exactly a reverse-Fisher score-matching objective on the smoothed densities, up to the positive constant $\eta^2\tau^4$. Thus Gaussian-kernel drifting is precisely a positive rescaling of the clean-model-weighted smoothed score-mismatch functional and therefore has the same global minimizers. See \Cref{app:proof-precond-score} for the proof.

\begin{theorem}[Gaussian-Kernel Drifting Equals Smoothed-Score Matching]
\label{thm:gaussian_meanshift_score_matching}
Suppose that $k_\tau$ is a Gaussian kernel with a fix $\tau>0$. Let $p$ and $q$ be distributions on $\mathbb R^D$. Then, for all $\rvx$,
\begin{equation*}
\mathbf V_{\uppi,k_\tau}(\rvx)=\tau^2 \mathbf s_{\uppi,\tau}(\rvx),
\qquad
\text{for }\uppi\in\{p,q\}.
\end{equation*}
Consequently, the drifting discrepancy field satisfies
\begin{mybox}
\begin{equation*}
\Delta_{p,q}(\rvx)
=
\eta\tau^2\big(\mathbf s_{p,\tau}(\rvx)-\mathbf s_{q,\tau}(\rvx)\big).
\end{equation*}
\end{mybox}
Now consider a pushforward model $q_\btheta=(\rvf_\btheta)_\# p_{\mathrm{prior}}$ and the fixed-point regression loss in \Cref{eq:value_equiv}. At the objective level, Gaussian drifting is exactly reverse Fisher on smoothed distributions:
\begin{equation}
\label{eq:drift_value_equals_score}
\mathcal L_{\mathrm{drift}}(\btheta)
=
\eta^2\tau^4\,
\mathbb E_{\rvx\sim q_\btheta}\Big[\big\|\mathbf s_{p,\tau}(\rvx)-\mathbf s_{q_\btheta,\tau}(\rvx)\big\|_2^2\Big].
\end{equation}
\end{theorem}

\Cref{thm:gaussian_meanshift_score_matching} is a direct consequence of Tweedie's formula~\cite{efron2011tweedie}, which is also a central identity behind diffusion denoising. Under Gaussian smoothing, the kernel barycenter is exactly the Bayes-optimal denoiser:
\begin{equation*}
\boldsymbol{\mu}_{\uppi,k_\tau}(\rvx)
=
\mathbb E[\mathbf X | \tilde{\mathbf X}=\rvx],
\qquad
\tilde{\mathbf X}=\mathbf X+\tau \mathbf Z,
\ \ 
\mathbf Z\sim\mathcal N(\mathbf 0,\mathbf I).
\end{equation*}
Hence the mean-shift direction is precisely the denoising residual, and Tweedie's formula gives
\begin{equation*}
\mathbf V_{\uppi,k_\tau}(\rvx)
=
\boldsymbol{\mu}_{\uppi,k_\tau}(\rvx)-\rvx
=
\tau^2\nabla_{\rvx}\log \uppi_\tau(\rvx).
\end{equation*}
Thus, for Gaussian kernels, mean shift is not merely score-like: it equals $\tau^2$ times the score of the Gaussian-smoothed distribution. Viewed as a transport field, the resulting Gaussian-smoothed score-mismatch is closely related to the score-difference flow of \cite{weber2023score}, which derives the same normalized-kernel-barycenter field as a KL-decreasing direction.

Equivalently, the discrepancy $\mathbf V_{p,k_\tau}-\mathbf V_{q,k_\tau}$ is exactly the score-mismatch between the smoothed densities $p_\tau$ and $q_\tau$, up to the same factor. This places Gaussian-kernel drifting squarely within a score-matching framework. The only distinction is the weighting measure: the drifting objective in \Cref{eq:drift_value_equals_score} is in reverse-Fisher form and averages under $q_\btheta$, rather than under $p$. As a result, learning is concentrated where the current model already places mass. In particular, if $q_\btheta$ assigns negligible mass to a mode of $p$, then that region contributes little to the gradient, which can make drifting more vulnerable to mode dropping~\cite{lu2025adversarial,minka2005divergence}, especially early in training.


\subsection{General Radial Kernels: Mean-Shift is a Pre-Conditioned Score-Mismatch Plus Residual}
\label{sec:radial}
In practice, however, drifting models use non-Gaussian kernels, with Laplace as the default choice. We now isolate what changes beyond the Gaussian case and show that mean-shift still follows the score up to an explicit pre-conditioning factor and a residual term.

We consider the family of translation-invariant radial kernels
\begin{equation}
\label{eq:radial_kernel}
k_\tau(\rvx,\rvy)
=
\exp\Big(-\rho\Big(\frac{\|\rvx-\rvy\|_2}{\tau}\Big)\Big),
\qquad \tau>0,
\end{equation}
where $\rho:[0,\infty)\to\mathbb R$ is differentiable. This family includes the Gaussian kernel, with $\rho(u)=\tfrac12 u^2$, and the Laplace kernel, with $\rho(u)=u$. As in \Cref{subsec:mean-shift-kernel-score}, we write $\mathbf V_{\uppi,k_\tau}$ for the mean-shift field and $\mathbf s_{\uppi,k_\tau}=\nabla\log \uppi_{k_\tau}$ for the kernel-induced score, where $\uppi_{k_\tau}(\rvx)=\mathbb E_{\rvy\sim\uppi}[k_\tau(\rvx,\rvy)]$.

For radial kernels, the score still takes the form of a local average of displacements, but with an additional radius-dependent reweighting. Writing $r:=\|\rvx-\rvy\|_2$, differentiation of $\log k_\tau(\rvx,\rvy)=-\rho(r/\tau)$ yields
$\nabla_{\rvx}\log k_\tau(\rvx,\rvy)
=
\frac{1}{\tau^2}\,b_\tau(r)\,(\rvy-\rvx)$,
where $b_\tau(r):=\frac{\rho'(r/\tau)}{r/\tau}$.
Substituting this identity gives
\begin{equation}
\label{eq:radial_weighted_score_ratio}
\tau^2\,\mathbf s_{\uppi,k_\tau}(\rvx)
=
\frac{\mathbb E_{\rvy\sim\uppi}\big[k_\tau(\rvx,\rvy)\,b_\tau(\|\rvx-\rvy\|_2)\,(\rvy-\rvx)\big]}
{\mathbb E_{\rvy\sim\uppi}\big[k_\tau(\rvx,\rvy)\big]}.
\end{equation}
The Gaussian kernel is therefore the unique case in which the extra weight disappears: for $\rho(u)=\tfrac12u^2$, one has $b_\tau(r)\equiv 1$, and mean-shift becomes exactly proportional to the smoothed score.

To compare mean-shift and score under the same local law, define the kernel-reweighted distribution
\begin{equation*}
\uppi_\tau(\rvy| \rvx)
:=
\frac{k_\tau(\rvx,\rvy)\,\uppi(\rvy)}
{\mathbb E_{\rvy\sim\uppi}[k_\tau(\rvx,\rvy)]}.
\end{equation*}
Then \Cref{eq:radial_weighted_score_ratio} can be written as
\begin{equation*}
\tau^2\,\mathbf s_{\uppi,k_\tau}(\rvx)
=
\mathbb E_{\rvy\sim\uppi_\tau(\cdot| \rvx)}
\Big[b_\tau\big(\|\rvx-\rvy\|_2\big)(\rvy-\rvx)\Big].
\end{equation*}
The next theorem, proved in \Cref{app:proof-precond-score}, gives an exact decomposition of mean-shift as score-mismatch.
\begin{theorem}[Pre-conditioned-Score Decomposition of Mean-Shift for General Radial Kernels]
\label{prop:radial_precond_decomp}
Fix $\tau>0$ and let $k_\tau$ be defined by \Cref{eq:radial_kernel}. Then for any $\rvx\in\mathbb R^D$,
\begin{equation*}
\mathbf V_{\uppi,k_\tau}(\rvx)
=
\tau^2\,\alpha_{\uppi}(\rvx)\,\mathbf s_{\uppi,k_\tau}(\rvx)
+
\boldsymbol{\delta}_{\uppi}(\rvx),
\end{equation*}
where $\alpha_{\uppi}(\rvx):=
\mathbb E_{\rvy\sim\uppi_\tau(\cdot| \rvx)}
\Big[b_\tau^{-1}(\|\rvx-\rvy\|_2)\Big]$, and $\boldsymbol{\delta}_{\uppi}(\rvx)
:=
\operatorname{Cov}_{\rvy\sim\uppi_\tau(\cdot| \rvx)}
\Big(b_\tau^{-1}(\|\rvx-\rvy\|_2),\,b_\tau(\|\rvx-\rvy\|_2)(\rvy-\rvx)\Big)\in\mathbb R^D$. Consequently, for $\Delta_{p,q}(\rvx)=\eta(\mathbf V_{p,k_\tau}(\rvx)-\mathbf V_{q,k_\tau}(\rvx))$,
\begin{mybox}
\[
\Delta_{p,q}(\rvx)
=
\eta\tau^2\Big(\alpha_{p}(\rvx)\mathbf s_{p,k_\tau}(\rvx)-\alpha_{q}(\rvx)\mathbf s_{q,k_\tau}(\rvx)\Big)
+
\eta\big(\boldsymbol{\delta}_{p}(\rvx)-\boldsymbol{\delta}_{q}(\rvx)\big).
\]
\end{mybox}
\end{theorem}

\Cref{prop:radial_precond_decomp} isolates exactly what changes beyond the Gaussian case. The mean-shift field still follows the kernel-induced score, but now through a scalar pre-conditioner $\alpha_{\uppi}(\rvx)$ and an additional residual $\boldsymbol{\delta}_{\uppi}(\rvx)$ caused by radius-dependent reweighting. In the Gaussian case, $b_\tau\equiv 1$, so $\alpha_{\uppi}\equiv 1$ and $\boldsymbol{\delta}_{\uppi}\equiv \mathbf 0$, exactly recovering \Cref{thm:gaussian_meanshift_score_matching}.
\begin{wrapfigure}{r}{0.3\textwidth}
    \centering
    \vspace{-0.6em}
    \includegraphics[width=0.87\linewidth]{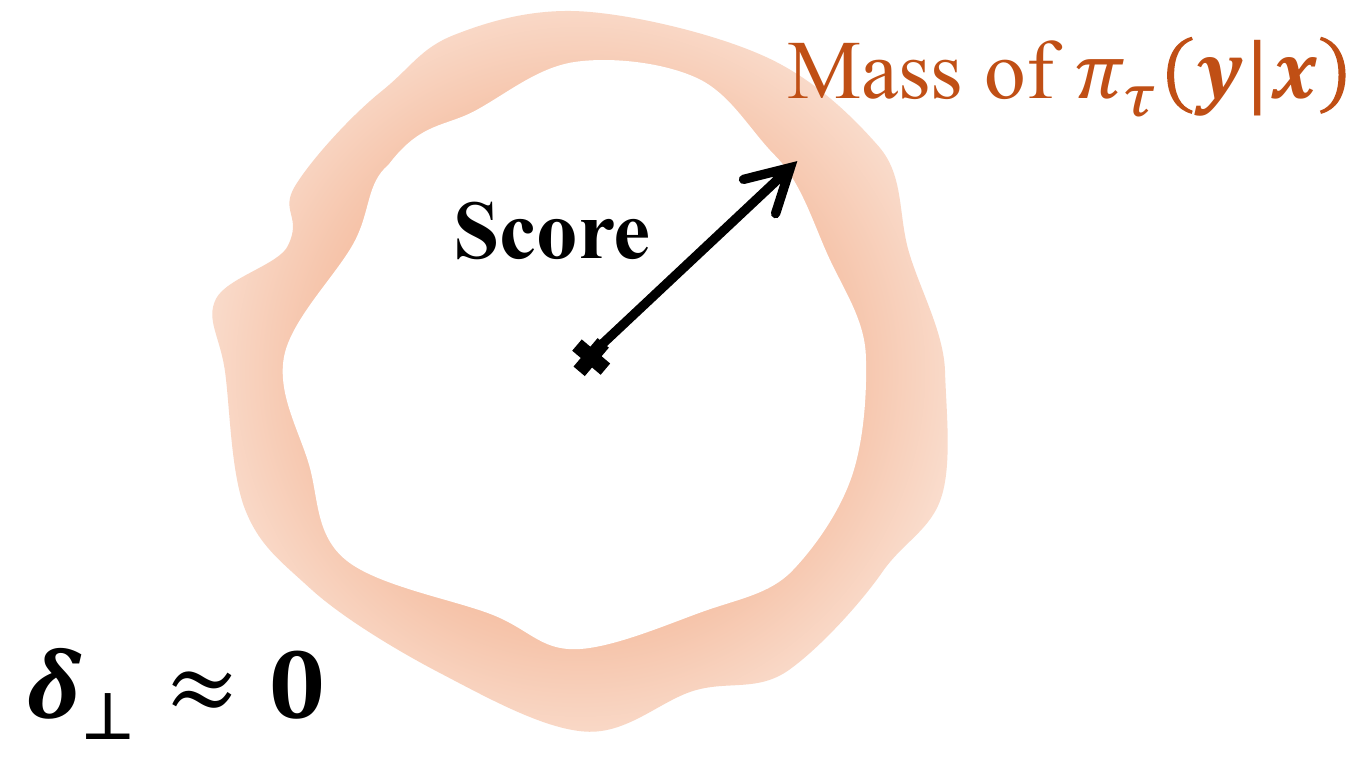}
    \caption{\footnotesize{\textbf{Illustration of $\boldsymbol\delta_\perp(\rvx)$ for the Laplace case.} Here $\boldsymbol\delta_\perp(\rvx)\approx \bm{0}$ because tangential contributions from different directions nearly cancel.}}
\label{fig:illustration-of-delta-laplace}
\end{wrapfigure}
For the Laplace kernel, the residual has a simple geometric interpretation; we defer the derivation to \Cref{app:radial_geometry} but highlight the main idea here. Write the displacement from $\rvx$ to a local neighbor $\rvy$ as $\rvy-\rvx=r\,\rvu$, where $r:=\|\rvy-\rvx\|_2$ is the distance to the neighbor and $\rvu$ is the corresponding unit direction.
Let $\rvu_\perp$ denote the component of $\rvu$ orthogonal to the score $\mathbf s_{\uppi,k_\tau}(\rvx)$. Then the score-orthogonal component of the residual satisfies
\begin{equation}
\label{eq:laplace_delta_perp_main}
\boldsymbol{\delta}_\perp(\rvx)
=
\mathbb E_{\rvy\sim\uppi_\tau(\cdot|\rvx)}\big[r\,\rvu_\perp\big].
\end{equation}

In other words, for the Laplace kernel, mean-shift departs from the score direction only through a radius-weighted average tangential bias. This makes the high-dimensional mechanism transparent. If the kernel-reweighted radii concentrate so that $r$ is nearly constant (see \Cref{fig:illustration-of-delta-laplace}), then the weighting in \Cref{eq:laplace_delta_perp_main} is nearly uniform, again making $\boldsymbol{\delta}_\perp(\rvx)$ small. In high dimension, this concentration is typical, so the residual is suppressed and mean-shift remains nearly aligned with the score. This is exactly the mechanism behind the alignment results in \Cref{sec:laplace}.

\Cref{prop:radial_precond_decomp} also explains why, beyond the Gaussian, drifting may admit multiple optima: the additional terms can make the drift vanish even when the model and data distributions differ; see \Cref{app:identifiability}.

\subsection{Laplace Kernels: Mean-Shift is Approximately a Score-Mismatch in High Dimension}\label{sec:laplace}

Let $\rvf$ be a generator, and let $q_{\rvf} := \rvf_\# \mathcal{N}(\mathbf{0}, \rmI)$ denote its induced model distribution on the same space $\mathbb{R}^D$ as the data distribution $p$. Consider the Laplace kernel
$k_\tau(\rvx,\rvy)=\exp\left(-\frac{\|\rvx-\rvy\|_2}{\tau}\right)$.
Following \cite{deng2026drifting}, we use the dimension-aware temperature
$\tau:=\bar\tau D^a$, where $\bar\tau>0$ and $a\ge 0$ are hyper-parameters, so that the Laplace kernel stays on a meaningful pairwise-distance scale as $D$ grows.
Define the mean-shift discrepancy by
$\Delta_{p,q_\rvf}(\rvx):=\mathbf V_{p,k_\tau}(\rvx)-\mathbf V_{q_\rvf,k_\tau}(\rvx)$,
and the score-mismatch by
$\Delta \rvs_{p,q_\rvf}(\rvx):=\rvs_{p,\tau}(\rvx)-\rvs_{q_\rvf,\tau}(\rvx)$.

We study the regime of fixed $\bar\tau>0$ and large dimension $D$. This regime is relevant both in raw data space and, more practically, in pre-trained feature spaces, whose dimension is often on the order of $10^3$. In this setting, the connection between drifting and score matching becomes explicit.

Under standard high-dimensional regularity conditions, including shell concentration and controlled inner products (\Cref{ass:shell_R0}--\Cref{ass:bounded_norm}), we show that the drifting discrepancy field is well approximated by a scaled score-mismatch field, with mean-square error of order $\mathcal O(D^{-1})$.

\begin{theorem}[(Informal) Large-$D$ field alignment at $1/D$ rate]
\label{thm:field_alignment_radial_full_xy}
Suppose that $k_\tau$ is a Laplace kernel. Assume the distributions under consideration concentrate on a common-radius shell, and that their inner products and moments are suitably controlled. Let $C=C(D^a)>0$ denote the known scaling factor. Then for all $\rvf$ and all sufficiently large $D$,
\begin{mybox}
\[
\mathbb E_{\rvx\sim q_{\rvf}}
\Big\|
\Delta_{p,q_\rvf}(\rvx)-C\,\Delta \rvs_{p,q_\rvf}(\rvx)
\Big\|_2^2
=
\mathcal O(D^{-1}).
\]
\end{mybox}
\end{theorem}
Thus, the drifting objective is approximated by a reverse Fisher divergence between kernel-smoothed scores, up to vanishing $\mathcal O(D^{-1})$ error; in particular, $C$ is constant if $a=0$.

The same estimate has a simple geometric consequence: the drifting discrepancy field and the score-mismatch field become asymptotically parallel. That is, in high dimension, the two fields are not only close in mean square, but also point in essentially the same direction.

\begin{corollary}[(Informal) Large-$D$ cosine similarity of discrepancy fields]
\label{cor:field_cosine_alignment_largeD}
Assume the conditions of \Cref{thm:field_alignment_radial_full_xy}. Suppose moreover that, in averaged norm under $q_{\rvf}$, at least one of the two discrepancy fields remains non-vanishing. Then for all sufficiently large $D$,
\begin{mybox}
\[
\mathbb E_{\rvx\sim q_\rvf}
\Big[
\big\langle
\overline{\Delta}_{p,q_\rvf}(\rvx),
\overline{\Delta \rvs}_{p,q_\rvf}(\rvx)
\big\rangle
\Big]
=
1-\mathcal O(D^{-1}).
\]
\end{mybox}
Here, $\overline{\Delta}_{p,q_\rvf}$ and $\overline{\Delta \rvs}_{p,q_\rvf}$ are the normalized discrepancy fields
\[
\overline{\Delta}_{p,q_\rvf}(\rvx)
:=
\Delta_{p,q_\rvf}(\rvx)\Big/
\Big(\mathbb E_{\rvx\sim q_\rvf}\|\Delta_{p,q_\rvf}(\rvx)\|_2^2\Big)^{\scriptscriptstyle \frac{1}{2}},
\,\,\,
\overline{\Delta \rvs}_{p,q_\rvf}(\rvx)
:=
\Delta \rvs_{p,q_\rvf}(\rvx)\Big/
\Big(\mathbb E_{\rvx\sim q_\rvf}\|\Delta \rvs_{p,q_\rvf}(\rvx)\|_2^2\Big)^{\scriptscriptstyle \frac{1}{2}}.
\]
Equivalently, the cosine similarity between the drifting mean-shift field and the score-mismatch field converges to $1$ at rate $\mathcal O(D^{-1})$.
\end{corollary}
Even at low dimension, the Laplace mean-shift direction can already be fairly well aligned with the score-mismatch direction; see \Cref{subsec:emp-toy}. \Cref{cor:field_cosine_alignment_largeD} shows that this alignment becomes asymptotically exact in high dimension.

We defer two further consequences of this high-dimensional picture to \Cref{app:high-D}. First, we prove \emph{algorithmic gradient alignment}: the implemented stop-gradient update of drifting becomes asymptotically identical to the corresponding score-transport update. Second, we prove \emph{minimizer alignment}: population minimizers of the drifting objective achieve vanishing score-mismatch, so the two principles agree not only at the field level but also at the level of optima. Because kernel-based estimation is delicate in high dimension, practical implementations of drifting models are necessarily more carefully engineered; implementation-aligned counterparts of these high-dimensional theorems are given in \Cref{app:impl_highD}. Beyond the high-dimensional regime, we also analyze other settings in which drifting becomes secretly score-based, including the low-temperature regime in \Cref{app:low-tem}. Taken together, these results show that drifting behaves as a score-based model in disguise.


\section{Empirical Study}\label{sec:emp}

\subsection{Examinations with Oracles}\label{subsec:emp-toy}

We empirically test the high-dimensional alignment prediction in our theory: as the ambient dimension grows, the Laplace drifting discrepancy should become increasingly well approximated by a scaled score discrepancy. Concretely, we examine whether
\[
\Delta_{p,q}(\mathbf{x}) \approx C\,\Delta \mathbf{s}_{p,q}(\mathbf{x}),
\quad
\text{and}
\quad
\cos\angle\big(\Delta_{p,q}(\mathbf{x}),\,\Delta \mathbf{s}_{p,q}(\mathbf{x})\big)\to 1
\]
as $D\to\infty$. To do so, we treat $p$ and $q$ as oracle samplers and consider two synthetic families. \emph{Ring MoG} is a controlled setting in which $p$ and $q$ share the same shell structure and differ primarily by an angular offset, whereas \emph{Raw MoG} is a broader stress test with mismatched mode layouts and more heterogeneous norms. At each query point $\mathbf{x}\sim q$, we estimate both the Laplace mean-shift field and the corresponding kernel-score field nonparametrically from finite reference samples using the same Laplace kernel, with no training involved. Exact dataset constructions, temperature choice, and estimator formulas are deferred to \Cref{app:toy-setup}.

\begin{figure}
    \centering
    \includegraphics[width=0.8\linewidth]{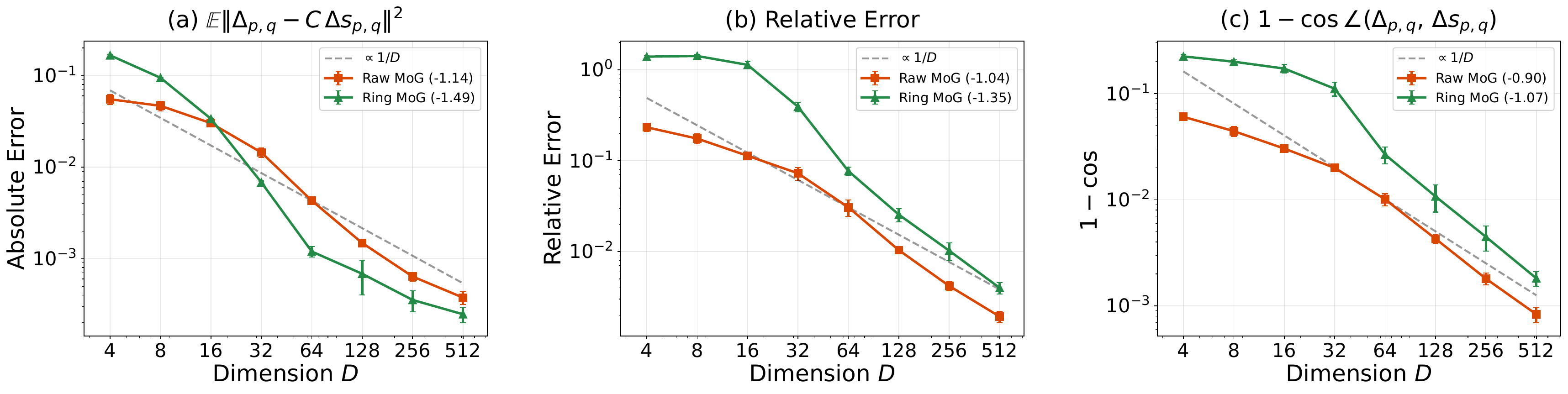}
\caption{\footnotesize{\textbf{Empirical validity of drifting--score alignment as dimension grows.}
Field alignment between the drifting discrepancy $\Delta_{p,q}(\mathbf{x})$ and the score discrepancy $\Delta \mathbf{s}_{p,q}(\mathbf{x})$ as the dimension $D$ increases, evaluated on Ring MoG and Raw MoG.
(a) Absolute alignment error $\mathbb{E}_q\|\Delta_{p,q}(\mathbf{x})-C^*\Delta \mathbf{s}_{p,q}(\mathbf{x})\|_2^2$, where $C^*$ is the best global least-squares scaling defined in \Cref{app:toy-diagnostics}.
(b) Relative error normalized by the field energy $\mathbb{E}_q\|\Delta_{p,q}(\mathbf{x})\|^2$.
(c) Directional misalignment $1-\cos\angle(\Delta_{p,q},\Delta \mathbf{s}_{p,q})$ on a log scale.
For both datasets, alignment improves monotonically with dimension: the errors decay close to the predicted $1/D$ rate, while the cosine similarity approaches $1$, supporting the theory in \Cref{sec:laplace}.}}
\vspace{-1.5em}
    \label{fig:emp-D-decay}
\end{figure}

\begin{wrapfigure}{r}{0.23\textwidth}
\vspace{-3em}
    \centering
    \includegraphics[width=0.85\linewidth]{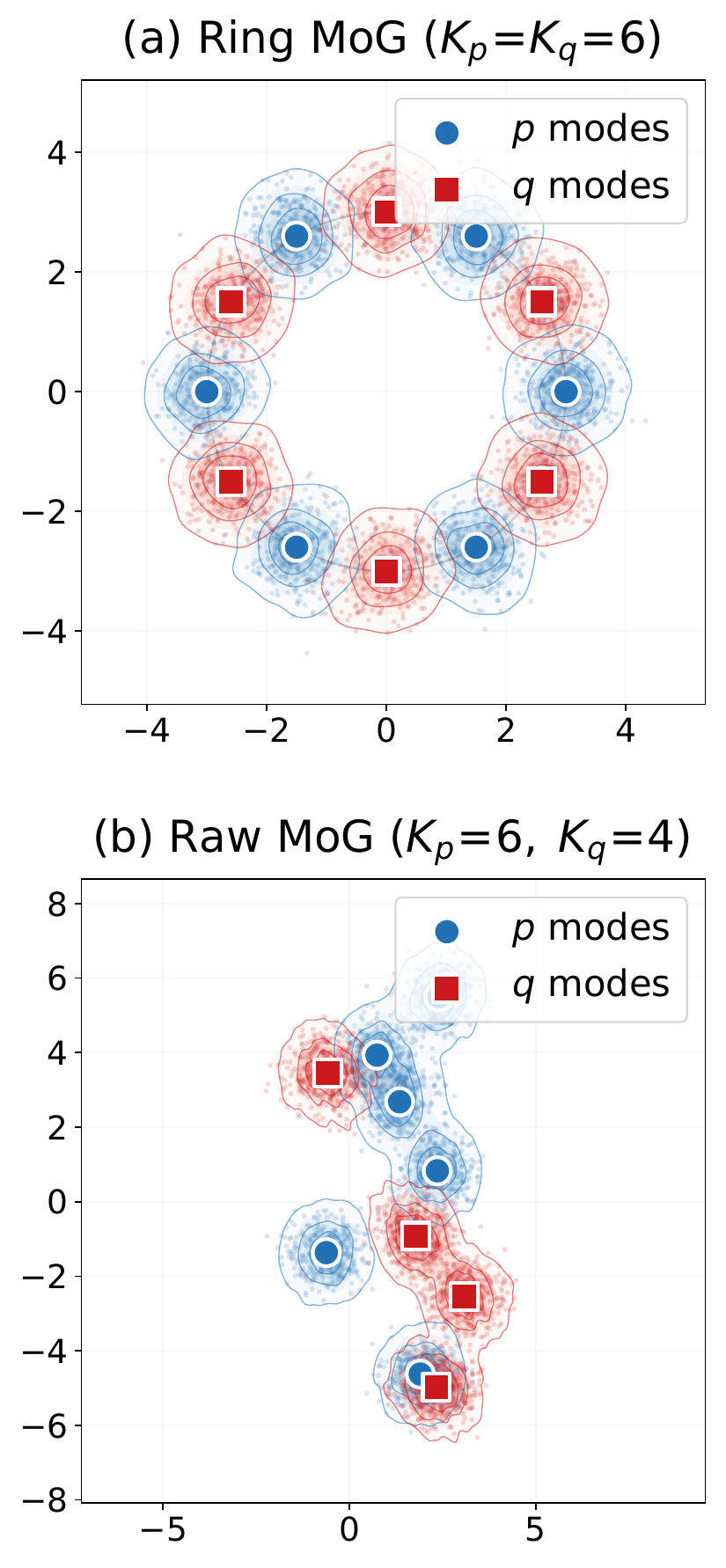}
    \caption{\footnotesize{\textbf{Illustration of the 2D datasets.} Top row shows Ring MoG; bottom row shows Raw MoG.}}
    \vspace{-2em}
\end{wrapfigure}

\paragraph{Empirical Results of Alignment as $D$ Increases.} We evaluate four complementary alignment metrics as $D$ increases; the results are shown in \Cref{fig:emp-D-decay}. Panels (a) and (b) report the absolute and relative approximation errors between $\Delta_{p,q}$ and the theory-predicted rescaling of $\Delta \mathbf{s}_{p,q}$. For both Ring MoG and Raw MoG, these errors decrease steadily with dimension. To make the rate explicit, we annotate the curves in (a) and (b) with their log--log regression slopes; the measured slopes are close to $-1$, in agreement with the $\mathcal O(D^{-1})$ prediction of \Cref{thm:field_alignment_radial_full_xy}. Panels (c) and (d) isolate directional alignment through the cosine similarity and its complement $1-\cos\angle$. In both datasets, the cosine similarity increases toward $1$ as $D$ grows, while the directional misalignment decreases correspondingly. Thus, the empirical picture matches the theory from two sides: the fields become close in magnitude after the predicted rescaling, and they become increasingly parallel in direction.

The geometric intuition follows from \Cref{prop:radial_precond_decomp} and is supported by the diagnostics in \Cref{app:toy-diagnostics}. As $D$ grows, the averaged pre-conditioners $\bar\alpha_p:=\E_{q}[\alpha_p(\rvx)]$ and $\bar\alpha_q:=\E_{q}[\alpha_q(\rvx)]$ converge to the same scale, while the covariance residuals $\boldsymbol{\delta}_{p}$ and $\boldsymbol{\delta}_{q}$ vanish. Thus, in high dimension, the drifting discrepancy for Laplace kernels converges to a scaled score discrepancy.

\subsection{Examinations with Trained Models}\label{subsec:emp-train}

\paragraph{Laplace (Drifting) vs. Gaussian (Score-Based) in Practice.}
A practical question is whether the Laplace-specific pre-conditioning and residual terms materially affect generation quality. To examine this, we compare one-step generators on CIFAR-10 within an otherwise identical drifting pipeline, varying only the kernel: Gaussian, for which mean shift is exactly aligned with the smoothed score-mismatch, and Laplace, the default choice in drifting implementations. Under matched training budgets, Gaussian achieves FID 8.38 versus 8.66 for Laplace in the ConvNeXt feature space~\cite{liu2022convnet,woo2023convnextv2}, and 4.72 versus 4.79 in the MAE feature space~\cite{he2022masked}. Since the training pipeline, architecture, and feature map are otherwise fixed, this comparison isolates the effect of the non-Gaussian pre-conditioning and residual terms. The near-parity across both feature spaces suggests that, in these settings, the Laplace-specific corrections have limited impact on end-to-end sample quality. This is also consistent with the concurrent findings of \cite{li2026long}, who report similar FIDs for Laplace and Gaussian kernels on CelebA-HQ. More experimental details are provided in \Cref{app:real-setup}.

\paragraph{Post-hoc Analysis on CIFAR-10.}
To understand why the two kernels behave so similarly in practice, we test whether the same mechanism predicted by our theory also appears in realistically trained drifting models on CIFAR-10. We record checkpoints from Laplace-drifting models and perform the same post-hoc analysis in the training feature spaces, namely ConvNeXt ($D=1024$) and MAE ($D=5120$), using the estimators from \Cref{app:toy-setup,app:toy-diagnostics}. \Cref{fig:cifar10-posthoc} reports four diagnostics over training. Panels (a) and (b) test the field-level prediction that the drifting discrepancy is well approximated by a scaled score discrepancy: in both feature spaces, the relative error decreases and stabilizes, while the cosine similarity remains high, showing increasing alignment between the learned drifting field and its score-based counterpart. This effect appears for both encoders and is even stronger in the higher-dimensional MAE space. Panels (c) and (d) probe the mechanism in \Cref{prop:radial_precond_decomp}: the empirical pre-conditioners move closer over training, while the averaged residual gap remains uniformly small, indicating that the covariance correction is minor in practice. Taken together, these results explain the near-parity in FID and support the same high-dimensional picture in realistic models: in the pre-trained feature spaces where drifting is optimized, Laplace drifting is nearly parallel to a score-based field, while the pre-conditioning increasingly matches and the residual terms remain small throughout training.

\paragraph{Post-hoc Analysis on ImageNet $256\times256$.}
Since training drifting models from scratch at high resolution is costly, we validate our theory using the officially released checkpoints. Unlike CIFAR-10, we consider a stronger latent$+$feature regime, using their best released ImageNet checkpoint: a Laplace model with SD-VAE latent generation and a pre-trained MAE feature-space loss, achieving FID $\approx 1.54$~\cite{deng2026drifting}. We then evaluate four theorem-facing quantities: the relative field error, the cosine similarity, the averaged pre-conditioner ratio $\bar\alpha_p/\bar\alpha_q$, and the averaged residual gap $\mathbb E_q\|\boldsymbol{\delta}_p-\boldsymbol{\delta}_q\|_2^2$. At the released checkpoint, these are $0.184$, $0.983$, $1.002$, and $1.6\times10^{-6}$, respectively, indicating strong field alignment, matched pre-conditioning, and a negligible residual correction.

Taken together, our theory and experiments show that the connection between drifting and score-based modeling is structural rather than via incidental engineering choices. It persists across feature spaces and across both pixel- and latent-space formulations, indicating that it is rooted in the underlying transport mechanism. In this sense, drifting is a score-based model in disguise.
\begin{figure}
\vspace{-1.em}
    \centering
    \includegraphics[width=\linewidth]{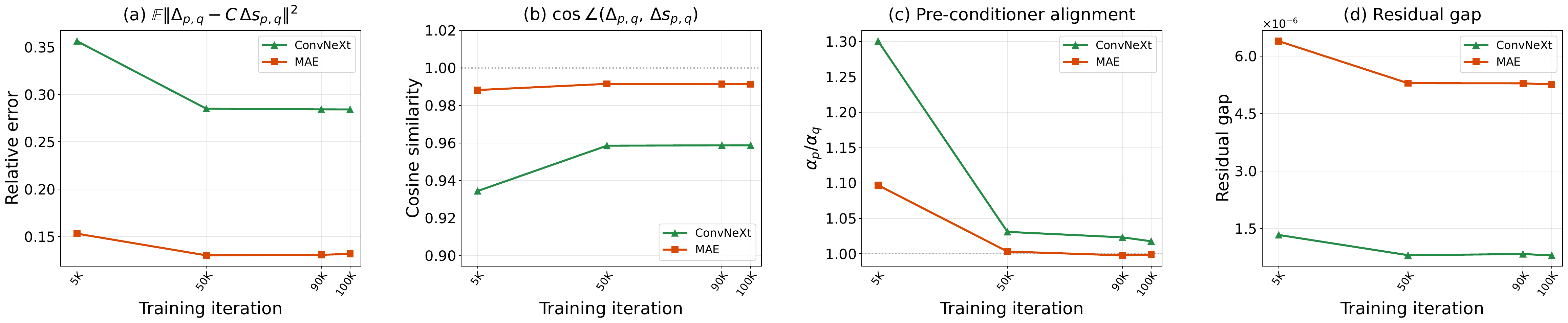}
\caption{\footnotesize{\textbf{Empirical validity of drifting--score alignment on CIFAR-10 under different pre-trained feature embeddings.}
Starting from Laplace-drifting checkpoints, we evaluate the decomposition-related quantities in the same frozen feature spaces used by training: the  ConvNeXt and the  MAE embeddings. (a) Relative error $\mathbb E_q\|\Delta_{p,q}-C^*\Delta \rvs_{p,q}\|_2^2 / \mathbb E_q\|\Delta_{p,q}\|_2^2$.
(b) Cosine similarity $\cos\angle(\Delta_{p,q},\Delta \rvs_{p,q})$.
(c) pre-conditioner matching $\bar\alpha_p/\bar\alpha_q$.
(d) Averaged residual gap $\mathbb E_q\|\bm{\delta}_p-\bm{\delta}_q\|_2^2$.
Across training, the approximation error decreases, the drifting and score discrepancies become more aligned, the pre-conditioners become better matched, and the residual gap remains small, again consistent with \Cref{prop:radial_precond_decomp,sec:laplace}.}}
\vspace{-1.5em}
    \label{fig:cifar10-posthoc}
\end{figure}


\section{Bridging Drifting Models and Other Existing Approaches}\label{app:drift-and-other}

In the previous sections, we showed that drifting is, at its core, score-based transport.
We now situate it among nearby one-step and two-sample generative methods. The closest prior framework is score-difference flow (SD-Flow)~\cite{weber2023score}, which already contains
several ingredients that reappear in Gaussian-kernel drifting: a Gaussian-smoothed score-difference
transport field, a normalized kernel-barycenter estimator, and a transport--then--regression generator
update mathematically equivalent to drifting's stop-gradient update at the semi-gradient level. We then discuss connections
to DMD~\cite{yin2023one}, GANs~\cite{goodfellow2014generative}, and
IGN~\cite{shocher2024idempotent}.

\subsection{Relationship to Score-Difference Flow and DMD}
\label{sec:sdf}
\label{sec:dmd}

Drifting model with Gaussian kernel is most directly connected to score-difference flow
(SD-Flow)~\cite{weber2023score}. SD-Flow predates drifting models and already develops the key ingredients that reappear in drifting models: an attractive--repulsive transport
field via scores, a normalized kernel-barycenter estimator, and a transport--then--regression generator update
closely aligned with drifting model's stop-gradient fixed-point update. DMD follows the same score-transport
logic, but realizes the score-mismatch parametrically through pre-trained diffusion teacher,
rather than nonparametrically through kernel estimates.

\paragraph{Score-Difference Flow from KL Descent.}
SD-Flow starts from a variational question: given a current model distribution $q$ and a target
distribution $p$, what infinitesimal transport direction moves $q$ most directly toward $p$?
Let particles $\rvx\sim q$ be transported by
\[
\rmT_\varepsilon(\rvx)=\rvx+\varepsilon \mathbf f(\rvx),
\]
and let $q_\varepsilon := (\rmT_\varepsilon)_\# q$ denote the pushed-forward distribution. The
first-order variation of the reverse KL satisfies
\[
\left.\frac{\diff}{\diff \varepsilon}
\mathcal D_{\mathrm{KL}}(q_\varepsilon\|p)
\right|_{\varepsilon=0}
=
-\E_{\rvx\sim q}
\left[
\big\langle
\mathbf s_p(\rvx)-\mathbf s_q(\rvx),\mathbf f(\rvx)
\big\rangle
\right],
\]
where
\[
\mathbf s_p(\rvx):=\nabla_{\rvx}\log p(\rvx),
\qquad
\mathbf s_q(\rvx):=\nabla_{\rvx}\log q(\rvx).
\]
Thus the steepest KL-decreasing direction, under the model-weighted $L^2(q)$ geometry, is the
score-difference field, or score-mismatch in our terminology,
\[
\mathbf f^*(\rvx)
=
\mathbf s_p(\rvx)-\mathbf s_q(\rvx).
\]
Moving particles along this field decreases the reverse KL at rate proportional to the reverse Fisher
divergence
\[
\E_{\rvx\sim q}
\big[
\|\mathbf s_p(\rvx)-\mathbf s_q(\rvx)\|_2^2
\big].
\]
This is the core principle of SD-Flow: the score-mismatch itself is the transport direction.

\paragraph{Gaussian Proxy Distributions and Normalized Kernel Barycenters.}
Since raw scores are typically unavailable or ill-defined, SD-Flow applies the same KL-descent
principle to Gaussian-smoothed proxy distributions. For the Gaussian kernel
\[
k_\tau(\rvx,\rvy)
=
\exp\left(-\frac{\|\rvx-\rvy\|_2^2}{2\tau^2}\right),
\]
define
\[
p_\tau(\rvx):=\E_{\rvy\sim p}[k_\tau(\rvx,\rvy)],
\qquad
q_\tau(\rvx):=\E_{\rvy\sim q}[k_\tau(\rvx,\rvy)].
\]
By Tweedie's formula, the score of the smoothed distribution can be written as
\[
\mathbf s_{\uppi,\tau}(\rvx)
=
\nabla_{\rvx}\log \uppi_\tau(\rvx)
=
\frac{1}{\tau^2}
\left(
\frac{\E_{\rvy\sim\uppi}[k_\tau(\rvx,\rvy)\rvy]}
     {\E_{\rvy\sim\uppi}[k_\tau(\rvx,\rvy)]}
-\rvx
\right),
\qquad
\uppi\in\{p,q\}.
\]
Consequently, the Gaussian-smoothed score difference is
\[
\mathbf s_{p,\tau}(\rvx)-\mathbf s_{q,\tau}(\rvx)
=
\frac{1}{\tau^2}
\left[
\frac{\E_{\rvy\sim p}[k_\tau(\rvx,\rvy)\rvy]}
     {\E_{\rvy\sim p}[k_\tau(\rvx,\rvy)]}
-
\frac{\E_{\rvy\sim q}[k_\tau(\rvx,\rvy)\rvy]}
     {\E_{\rvy\sim q}[k_\tau(\rvx,\rvy)]}
\right].
\]
Thus SD-Flow converts the KL-steepest-descent direction into a difference of two locally normalized
kernel barycenters. This local normalization is crucial: it is what turns a kernel force into a score-like
log-density-gradient signal.

\paragraph{Drifting Model with Gaussian Kernel is SD-Flow.}
The drifting model with Gaussian kernel uses the same normalized barycenter structure. By
\Cref{thm:gaussian_meanshift_score_matching},
\[
\Delta_{p,q}(\rvx)
=
\eta\tau^2
\big(\mathbf s_{p,\tau}(\rvx)-\mathbf s_{q,\tau}(\rvx)\big).
\]
Thus, at a fixed query point, the drifting model with Gaussian kernel and the kernel realization of
SD-Flow use the same population transport field, up to the scalar factor $\eta\tau^2$. SD-Flow
derives this field as the KL-decreasing direction between Gaussian-smoothed proxy distributions,
whereas drifting embeds it as the mean-shift discrepancy inside a fixed-point generator objective.

The tight connection also appears at the level of the generator update. In the model-optimization
view of SD-Flow, one first generates a particle
\[
\rvx=\rvf_{\btheta}(\rvz),
\qquad
\rvz\sim p_{\mathrm{prior}},
\]
moves it along the score-difference direction with stepsize $h>0$,
\[
\rvx^+
=
\rvx
+
h\big(\mathbf s_{p,\tau}(\rvx)-\mathbf s_{q_{\btheta},\tau}(\rvx)\big),
\]
and then updates the generator by regression toward the moved target. Written as a frozen-target
surrogate, this gives
\[
\mathcal L_{\mathrm{fp\mbox{-}score}}(\btheta)
:=
\E_{\rvz\sim p_{\mathrm{prior}}}
\left[
\left\|
\rvf_{\btheta}(\rvz)
-
\mathrm{sg}\left(
\rvf_{\btheta}(\rvz)
+
h\big(\mathbf s_{p,\tau}(\rvx)-\mathbf s_{q_{\btheta},\tau}(\rvx)\big)
\right)
\right\|_2^2
\right],
\qquad
\rvx=\rvf_{\btheta}(\rvz).
\]
Differentiating through the generator while freezing the transported target gives the semi-gradient
\[
\nabla_{\btheta}\mathcal L_{\mathrm{fp\mbox{-}score}}(\btheta)
=
-2h\,
\E_{\rvz\sim p_{\mathrm{prior}}}
\left[
\big(\partial_{\btheta}\rvf_{\btheta}(\rvz)\big)^\top
\big(\mathbf s_{p,\tau}(\rvx)-\mathbf s_{q_{\btheta},\tau}(\rvx)\big)
\right].
\]
For Gaussian kernels, choosing $h=\eta\tau^2$ yields the same first-order frozen-target update
direction as the drifting model with Gaussian kernel:
\[
\nabla_{\btheta}\mathcal L_{\mathrm{drift}}(\btheta)
=
-2\eta\tau^2\,
\E_{\rvz\sim p_{\mathrm{prior}}}
\left[
\big(\partial_{\btheta}\rvf_{\btheta}(\rvz)\big)^\top
\big(\mathbf s_{p,\tau}(\rvx)-\mathbf s_{q_{\btheta},\tau}(\rvx)\big)
\right].
\]
Thus the drifting model with Gaussian kernel and SD-Flow agree not only in the normalized-kernel
transport field, but also in the induced transport--then--regression semi-gradient.

\paragraph{Technical Comparison.}
The comparison above shows that, in the Gaussian population setting, the drifting model with Gaussian
kernel recovers several core ingredients already developed in SD-Flow: the attractive--repulsive
score-difference field, the normalized two-sample kernel-barycenter estimator, and the frozen-target
transport--then--regression semi-gradient.

The formulations differ in how this field is derived and used. SD-Flow is derived as a KL-gradient-flow
principle on Gaussian-smoothed proxy distributions; in its kernel form, the field is naturally evaluated
at noised proxy points, and the clean update is interpreted as aligning the smoothed proxies. Drifting,
instead, is formulated as a stop-gradient fixed-point regression objective and typically evaluates the
kernel field directly at generated samples or at their frozen feature embeddings. Thus, the field identity
is exact at a fixed query point, but the query law and training implementation differ.

\paragraph{DMD as Parametric Score Transport.}
DMD realizes a closely related score-transport principle, but obtains the score signal parametrically
rather than through kernel mean shift. Let $\rvf_\btheta(\rvz)$ be a deterministic one-step generator
with Gaussian prior $\rvz\sim p_{\mathrm{prior}}$, inducing $q_\btheta=(\rvf_\btheta)_\#p_{\mathrm{prior}}$.
Applying the same forward process to generated samples gives
\[
\hat{\rvx}_t
:=
\alpha_t\rvf_\btheta(\rvz)+\sigma_t\beps,
\qquad
\rvz\sim p_{\mathrm{prior}},\quad
\beps\sim\mathcal N(\mathbf 0,\mathbf I),
\]
with marginal density $q_{\btheta,t}$. DMD matches $q_{\btheta,t}$ to the noised data marginal $p_t$
through a reverse-KL objective. Its gradient takes the score-mismatch form
\[
\nabla_\btheta \mathcal L_{\mathrm{DMD}}(\btheta)
=
\E_{t,\rvz,\beps}
\Big[
\omega(t)\alpha_t
\big(\partial_\btheta \rvf_\btheta(\rvz)\big)^\top
\big(
\mathbf s_{q_\btheta}(\hat{\rvx}_t,t)
-
\mathbf s_p(\hat{\rvx}_t,t)
\big)
\Big].
\]
Equivalently, gradient descent transports generated samples along
\[
\mathbf s_p(\hat{\rvx}_t,t)-\mathbf s_{q_\btheta}(\hat{\rvx}_t,t).
\]
In practice, $\mathbf s_p(\cdot,t)$ is provided by a pre-trained diffusion teacher, while
$\mathbf s_{q_\btheta}(\cdot,t)$ is estimated by an auxiliary fake-score model trained on current
generator samples.

DMD can therefore be written in the same transport--then--projection language. Define
\[
\Delta\rvs_{p_t,q_{\btheta,t}}(\mathbf x)
:=
\frac{\omega(t)\alpha_t}{2}
\Big(
\mathbf s_p(\mathbf x,t)-\mathbf s_{q_\btheta}(\mathbf x,t)
\Big),
\]
and consider
\[
\mathcal L_{\mathrm{fp\mbox{-}score}}(\btheta)
:=
\E_{t,\rvz,\beps}
\Big[
\big\|
\rvf_\btheta(\rvz)
-
\mathrm{sg}\big(
\rvf_\btheta(\rvz)
+
\Delta\rvs_{p_t,q_{\btheta,t}}(\hat{\rvx}_t)
\big)
\big\|_2^2
\Big].
\]
Differentiating with stop-gradient gives the same direction as the DMD gradient above. Thus DMD and
drifting share the same model-weighted score-transport logic; they differ mainly in the source of the
score-mismatch signal.

Therefore, we summarize the relationship between drifting models, SD-Flow, and DMD as follows:
\begin{quote}
{\color{BrickRed}
\emph{
In the Gaussian case, the drifting model recovers the core SD-Flow ingredients inside a stop-gradient fixed-point generator objective: the same attractive--repulsive score-difference field, the same normalized kernel-barycenter estimator, the same reverse-Fisher objective, and the same transport--then--regression semi-gradient. DMD follows the same score-transport principle, but realizes the score mismatch parametrically through diffusion teacher/fake-score estimates, whereas drifting realizes it nonparametrically through kernel estimates.
}}
\end{quote}

\subsection{Relationship to GANs}\label{sec:gan}


 The drifting-model discrepancy field $\Delta_{p,q}$ has an \emph{attractive--repulsive} structure: it pulls model samples toward regions supported by the data distribution $p$ while pushing them away from regions heavily supported by the model distribution $q$. This is conceptually close to two-sample, kernel-based views of GANs~\cite{goodfellow2014generative}. The goal of this subsection is to highlight a simple but important distinction. Coulomb GANs~\cite{unterthiner2018coulomb} (and, more broadly, MMD-style kernel discrepancies~\cite{li2017mmd,li2015generative,unterthiner2018coulomb} discussed in \cite{deng2026drifting}) derive update directions by differentiating a \emph{global} objective that compares $p$ and $q$ through kernel interactions. Drifting models use the same two-sample ingredients, but introduce a \emph{local normalization} via mean shift that turns the kernel signal into a score-like (log-gradient) update, thereby connecting more naturally to (kernel-smoothed) score matching.

\paragraph{Coulomb GANs: A Potential Field Induced by A Global Discrepancy.}
Coulomb GANs begin with the signed density difference $p-q$ and define a kernel potential
\[
\Phi_{p,q}(\rvx)
:= \E_{\rvy\sim p}\big[k(\rvx,\rvy)\big]
   - \E_{\rvy\sim q}\big[k(\rvx,\rvy)\big],
\]
which appears in empirical (mini-batch) form as Equation~(8) in~\cite{unterthiner2018coulomb}.
Generator samples are transported along the induced \emph{force field}
(the direction followed by generator particles under the Coulomb-GAN charge convention),
given by the gradient of the potential:
\[
\mathbf{F}_{p,q}(\rvx)
:= \nabla_{\rvx}\Phi_{p,q}(\rvx)
= \E_{\rvy\sim p}\big[\nabla_{\rvx}k(\rvx,\rvy)\big]
  -\E_{\rvy\sim q}\big[\nabla_{\rvx}k(\rvx,\rvy)\big].
\]
This matches the electric-field expression in Equation~(34) of~\cite{unterthiner2018coulomb} up to a sign convention
(generator particles are treated as negative charges).
For radial kernels, $\nabla_{\rvx}k(\rvx,\rvy)$ is colinear with $(\rvy-\rvx)$, so $\mathbf{F}_{p,q}(\rvx)$ takes the familiar
``attract data, repel model'' form of a weighted displacement difference.

The key point is that this force field is not arbitrary: it is induced by an \emph{interaction energy}.
Coulomb GANs define
\[
\mathcal{E}(p,q)
:= \frac12 \int \big(p(\rvx)-q(\rvx)\big)\,\Phi_{p,q}(\rvx)\,\diff\rvx,
\]
and update particles by following the potential gradient $\mathbf{F}_{p,q}=\nabla_{\rvx}\Phi_{p,q}$.
Heuristically, transporting model mass in the direction of $\mathbf{F}_{p,q}$ decreases $\mathcal{E}$.
The energy is minimized when $p\equiv q$.
In this sense, the Coulomb-GAN update is driven by a \emph{global} discrepancy: the learning signal is determined by
how $p$ and $q$ interact through the kernel across the whole space.

To relate this to the drifting model, it is helpful to isolate the kernel-smoothed profile
\[
\uppi_k(\rvx) := \E_{\rvy\sim \uppi}\big[k(\rvx,\rvy)\big],
\qquad \uppi\in\{p,q\}.
\]
Intuitively, $\uppi_k(\rvx)$ summarizes how much probability mass of $\uppi$ lies near $\rvx$ in the sense of the kernel.
Note that $\uppi_k$ is generally not a normalized density; it is a kernel-smoothed mass profile (or ``density proxy'')
whose overall scale depends on the normalization of $k$.
With this notation,
\[
\Phi_{p,q}(\rvx)=p_k(\rvx)-q_k(\rvx),
\qquad
\mathbf{F}_{p,q}(\rvx)=\nabla_{\rvx} p_k(\rvx)-\nabla_{\rvx} q_k(\rvx),
\]
so Coulomb GAN transport is driven by gradients of these kernel-smoothed mass profiles.

\paragraph{Drifting Models: Same Ingredients but Locally Normalized.}
Drifting uses the same attraction--repulsion ingredients from $p$ and $q$, but replaces an unnormalized force by a \emph{locally normalized} mean-shift direction:
\[
\Delta_{p,q}(\rvx)
=\eta\big(\mathbf{V}_{p,k}(\rvx)-\mathbf{V}_{q,k}(\rvx)\big),
\qquad
\mathbf{V}_{\uppi,k}(\rvx)
=
\frac{\E_{\rvy\sim \uppi}\big[k(\rvx,\rvy)(\rvy-\rvx)\big]}
     {\E_{\rvy\sim \uppi}\big[k(\rvx,\rvy)\big]}
\quad (\uppi\in\{p,q\}).
\]
The numerator is a kernel-weighted displacement, while the denominator $\uppi_k(\rvx):=\E_{\rvy\sim\uppi}[k(\rvx,\rvy)]$ is the local kernel mass.
This normalization automatically rescales the step by local neighborhood density, while keeping the same qualitative behavior of ``attract $p$, repel $q$''.
In dense regions $\uppi_k(\rvx)$ is large, so the update is tempered; in sparse regions it is amplified, yielding a meaningful pull toward nearby neighbors even when few samples lie in the kernel neighborhood.
This is a score-like property: the update follows a direction determined by local geometry rather than a force whose magnitude is proportional to local mass.

Score matching is formulated in terms of log-gradients. The kernel-induced score is
\[
\mathbf{s}_{\uppi,k}(\rvx)
:= \nabla_{\rvx}\log \uppi_k(\rvx)
= \frac{\nabla_{\rvx}\uppi_k(\rvx)}{\uppi_k(\rvx)}.
\]
This makes the contrast explicit.
Kernel-based potential methods (e.g., Coulomb-style updates) act on the \emph{unnormalized} gradient $\nabla_{\rvx}\uppi_k(\rvx)$, which scales with local mass.
Drifting, through the mean-shift normalization, naturally aligns with the \emph{normalized} log-gradient $\nabla_{\rvx}\log \uppi_k(\rvx)$, namely the score of the kernel-smoothed profile.
As shown in \Cref{sec:instantiations}, this connection is exact for Gaussian kernels, where the drifting discrepancy reduces to a kernel-smoothed score-mismatch.

We summarize these points in the following takeaway:
\begin{quote}
{\color{BrickRed}
\emph{Mean-shift drifting is more score-like than kernel-discrepancy methods (e.g., Coulomb GANs/MMD), because its normalization by $\uppi_k(\rvx)$ turns mass gradients into log-gradients (scores).}
}
\end{quote}

\subsection{Relationship to IGNs}\label{sec:ign}

Idempotent Generative Networks (IGN)~\cite{shocher2024idempotent} provide a different route to one-step generation, but they can still be related to drifting through the same equilibrium-seeking viewpoint.

IGN starts from a self-map $\mathbf{f}_{\btheta}:\mathbb R^D\to\mathbb R^D$ and defines the residual field
\[
\Delta^{\mathrm{IGN}}_{\btheta}(\mathbf{u})
:=
\mathbf{f}_{\btheta}(\mathbf{u})-\mathbf{u},
\qquad
\mathcal U_{\btheta}(\mathbf{u})
=
\mathbf{u}+\Delta^{\mathrm{IGN}}_{\btheta}(\mathbf{u})
=
\mathbf{f}_{\btheta}(\mathbf{u}).
\]
Its equilibrium condition is simply the fixed-point relation
\[
\Delta^{\mathrm{IGN}}_{\btheta}(\mathbf{u})=\mathbf 0
\qquad\Longleftrightarrow\qquad
\mathbf{f}_{\btheta}(\mathbf{u})=\mathbf{u},
\]
with fixed-point set
\[
\mathcal M_{\btheta}
:=
\{\mathbf{u}\in\mathbb R^D:\mathbf{f}_{\btheta}(\mathbf{u})=\mathbf{u}\}.
\]

This gives IGN a natural geometric interpretation in the same equilibrium language.
The reconstruction loss
\[
\mathcal L_{\mathrm{rec}}
:=
\E_{\mathbf{x}\sim p}\big[\|\mathbf{f}_{\btheta}(\mathbf{x})-\mathbf{x}\|_2^2\big]
=
\E_{\mathbf{x}\sim p}\big[\|\Delta^{\mathrm{IGN}}_{\btheta}(\mathbf{x})\|_2^2\big]
\]
makes real data $\mathbf{x}\sim p$ lie near $\mathcal M_{\btheta}$ by enforcing $\mathbf{f}_{\btheta}(\mathbf{x})\approx \mathbf{x}$.
Likewise, if $\beps\sim p_{\mathrm{prior}}$ and $\hat{\mathbf{x}}:=\mathbf{f}_{\btheta}(\beps)$ denotes a generated output, the idempotence loss
\[
\mathcal L_{\mathrm{idem}}
:=
\E_{\beps\sim p_{\mathrm{prior}}}
\big[
\|\mathbf{f}_{\btheta}(\hat{\mathbf{x}})-\hat{\mathbf{x}}\|_2^2
\big]
=
\E_{\beps\sim p_{\mathrm{prior}}}
\big[
\|\Delta^{\mathrm{IGN}}_{\btheta}(\hat{\mathbf{x}})\|_2^2
\big]
\]
makes generated outputs also lie near the same set by enforcing
$\mathbf{f}_{\btheta}(\hat{\mathbf{x}})\approx \hat{\mathbf{x}}$.

In this sense, IGN may be viewed as using the residual field
\[
\Delta^{\mathrm{IGN}}_{\btheta}(\mathbf{u})=\mathbf{f}_{\btheta}(\mathbf{u})-\mathbf{u}
\]
as a restoring force that pulls both real data and generated samples toward a common data-anchored fixed-point set, while correcting points that remain off equilibrium.
This is the precise sense in which IGN is related to drifting.
Both are equilibrium-seeking, but the mechanism is different:
drifting, score-mismatch, and Coulomb/MMD-style updates are \emph{distribution-comparison fields}, which compare local statistics of $p$ and $q$ at a query point and produce an explicit attractive--repulsive transport rule.
IGN, by contrast, does not compare $p$ and $q$ through a subtractive two-sample field; instead, it uses a \emph{self-map residual field} that pulls inputs toward the fixed-point anchor manifold $\mathcal M_{\btheta}$.

\begin{table}[th!]
\centering
\small
\setlength{\tabcolsep}{5pt}
\renewcommand{\arraystretch}{1.15}
\caption{
Comparison of four update-field realizations for one-step generative modeling.
Drifting, score-mismatch, and Coulomb/MMD-style methods are all attractive--repulsive fields, while IGN is a self-map residual that enforces equilibrium through fixed points.
}
\resizebox{\textwidth}{!}{%
\begin{tabular}{p{2.6cm}p{2.4cm}p{4.8cm}p{2.4cm}p{3.4cm}}
\toprule
\textbf{Method}
& \textbf{Field Type}
& \textbf{Update Field}
& \textbf{Normalization?}
& \textbf{Main Interpretation}
\\
\midrule

\textbf{Drifting Model}
&
Mean-shift
&
$\displaystyle
\Delta_{p,q}(\rvx)
=
\eta\big(\mathbf V_{p,k}(\rvx)-\mathbf V_{q,k}(\rvx)\big)
$
&
$\checkmark$
&
Attractive--repulsive
\\

\textbf{Score-based Model}
&
Score-mismatch
&
$\displaystyle
\Delta \rvs_{p,q}(\rvx)
=
\mathbf s_{p,k}(\rvx)-\mathbf s_{q,k}(\rvx)
$
&
$\checkmark$ (automatic via score)
&
Attractive--repulsive
\\

\textbf{GAN} 
&
kernel-potential/ MMD-type kernel
&
$\displaystyle
\mathbf F_{p,q}(\rvx)
=
\nabla_{\rvx}p_k(\rvx)-\nabla_{\rvx}q_k(\rvx)
$
&
$\times$
&
Attractive--repulsive
\\

\textbf{IGN}
&
Self-map residual
&
$\displaystyle
\Delta^{\mathrm{IGN}}_{\btheta}(\rvu)
=
\rvf_{\btheta}(\rvu)-\rvu
$
&
$\times$
&
Data-anchored manifold-restoring
\\

\bottomrule
\end{tabular}
}
\label{tab:comparison_bridges}
\end{table}

\subsection{Summary and Comparison}\label{sec:summary-comparison}

Drifting is  directly connected to score-based modeling, since mean-shift drifting can be viewed as a score-mismatch discrepancy field, with exact equivalence in the Gaussian case. By comparison, its connections to Coulomb GAN and IGN are less direct: Coulomb GAN is related through kernel-induced force fields, while IGN is related through the shared equilibrium-seeking viewpoint, in which a data-anchored fixed-point restoring field pulls samples toward the anchored manifold. We summarize these relationships in \Cref{tab:comparison_bridges}.


\section{Conclusion}\label{sec:conclude}
We showed that drifting model is more inherently connected to score-based modeling than it may initially appear. For Gaussian kernels, the correspondence is exact: the mean-shift field coincides with a score-mismatch field, and the drifting objective becomes a reverse-Fisher score-matching objective. For general radial kernels, mean-shift decomposes into a pre-conditioned score term plus a residual that captures local geometry. For the Laplace kernel used in practice, we further proved that this residual becomes negligible in regimes of interest, yielding polynomial alignment of the drifting field, the implemented updates, and the population optima with their score-based counterparts. Our experiments support these predictions and show that Gaussian and Laplace kernels lead to broadly comparable generation quality. Overall, these results place drifting within the score-based generative modeling framework as a kernel-based, nonparametric realization of score-driven one-step generation. We hope that, by clarifying the connections between drifting models and existing approaches, our analysis helps distill the principles of one-step generation and guide the design of faster, more capable, and more stable generative methods.
\clearpage
\newpage
\bibliographystyle{unsrt}
\bibliography{ref}

\clearpage
\newpage
\appendix
\appendix
\tableofcontents

\begin{center}
   \section*{\LARGE Appendix}\label{appendix} 
\end{center}

\section{Discussion on Identifiability}\label{app:identifiability}

\paragraph{Definition of Identifiability.}
In the fixed-point training template of drifting models in \Cref{eq:fp_loss,eq:value_equiv}, even reaching a global minimum of the drifting loss does not by itself imply that the learned distribution matches the data. Indeed,
$\mathcal L_{\mathrm{drift}}(\btheta)=0$ only guarantees that the discrepancy field vanishes on model samples (that is, $\Delta_{p,q_\btheta}(\rvx)=\mathbf 0$ for $q_\btheta$-almost every $\rvx$). In principle, there can be multiple generators, and thus multiple model distributions, that satisfy this same fixed-point condition. This motivates an \emph{identifiability} question: is the data distribution the only equilibrium picked out by the discrepancy operator?

We say a discrepancy operator is \emph{identifiable}~\cite{deng2026drifting} on a class of distributions if the only way for its discrepancy to vanish is to have $q=p$. A natural criterion is
\begin{equation}
\label{eq:ident_criterion}
\Delta_{p,q}(\rvx)=\mathbf 0 \ \text{for $q$-almost every }\rvx
\quad\Longrightarrow\quad
q=p.
\end{equation}
If identifiability fails, the fixed-point condition can admit spurious equilibria with $q\neq p$, and driving the drifting loss to zero does not certify true distribution matching. Whether identifiability holds depends on the choice of kernel; we return to this question for Gaussian kernels in \Cref{sec:drifting} and for general radial kernels in \Cref{sec:radial}.

\paragraph{Identifiability of the Gaussian Kernel.}
Using a Gaussian kernel to define the discrepancy $\Delta_{p,q}(\rvx)$ (as in DMD-style objectives) has a clean
identifiability story. By \Cref{thm:gaussian_meanshift_score_matching} or the decomposition in \Cref{prop:radial_precond_decomp}, matching the mean-shift directions under a
Gaussian kernel is equivalent to matching the corresponding Gaussian-smoothed scores:
\[
\mathbf{s}_{p,\tau}(\rvx)=\mathbf{s}_{q,\tau}(\rvx)\qquad \text{for all $\rvx$}.
\]
Under mild regularity and the usual normalization (both $p_\tau$ and $q_\tau$ integrate to one), equality of scores
forces the smoothed densities to coincide: $\mathbf{s}_{p,\tau}\equiv \mathbf{s}_{q,\tau}$ implies
$p_\tau = q_\tau$ (the only ambiguity is a multiplicative constant, removed by normalization).
Finally, Gaussian smoothing is injective: if $p_\tau=q_\tau$ for some $\tau>0$, then necessarily $p=q$.
This yields the following identifiability statement.

\begin{proposition}[Identifiability of Gaussian Case (Idealized DMD)]
\label{prop:dmd_ident}
Assume $p_\tau$ and $q_\tau$ are the Gaussian-smoothed versions of $p$ and $q$ for some $\tau>0$.
If $\mathbf{s}_{p,\tau}(\mathbf{x})=\mathbf{s}_{q,\tau}(\mathbf{x})$ for all $\mathbf{x}$, then $p=q$.
In particular, any population objective that enforces $p_\tau=q_\tau$ for a single $\tau>0$ is identifiable.
\end{proposition}

In words, in the Gaussian case the DMD objective inherited from a diffusion teacher is identifiable under an idealized setting:
if the teacher is perfectly trained, then at the population level the objective has a unique intended optimum, namely $q=p$.
In practice, this guarantee is only approximate: it depends on the teacher score being sufficiently accurate and on student training
being well behaved, in the sense that the model class is expressive enough and optimization reaches the intended optimum rather than
an approximate or degenerate solution.

\paragraph{Identifiability of the General Radial Kernel.}
For the mean-shift--induced discrepancy field
\[
\Delta_{p,q}(\mathbf{x})
=
\eta\big(\mathbf V_{p,k}(\mathbf{x})-\mathbf V_{q,k}(\mathbf{x})\big),
\]
a natural identifiability question is whether driving the field to zero forces true distribution matching, namely,
\[
\Delta_{p,q}\equiv \mathbf 0
\quad \Longrightarrow \quad
p=q.
\]

For a general radial kernel, however, \Cref{prop:radial_precond_decomp} exposes extra degrees of freedom that can cancel the
score-mismatch. Concretely,
\begin{equation}
\label{eq:ident_decomp_key}
\Delta_{p,q}(\mathbf{x})=\mathbf 0
\quad\Longrightarrow\quad
\tau^2\Big(\alpha_{p}(\mathbf{x})\,\mathbf s_{p,k_\tau}(\mathbf{x})
-\alpha_{q}(\mathbf{x})\,\mathbf s_{q,k_\tau}(\mathbf{x})\Big)
+\big(\boldsymbol{\delta}_{p}(\mathbf{x})-\boldsymbol{\delta}_{q}(\mathbf{x})\big)=\mathbf 0.
\end{equation}
so $\Delta_{p,q}\equiv \mathbf 0$ only constrains a sum of a scaled score difference and a residual
$\boldsymbol{\delta}$ capturing distance--direction coupling. Unless one can separately control the scalar pre-conditioners
$\alpha_{\uppi}$ and the residuals $\boldsymbol{\delta}_{\uppi}$, \Cref{eq:ident_decomp_key} does not force
$\nabla\log p_\tau=\nabla\log q_\tau$, and therefore does not by itself certify $p_\tau=q_\tau$.
In particular, identifiability is generally $p$-dependent: the functions $\alpha_{p}$ and $\boldsymbol{\delta}_{p}$ are
determined by the kernel-reweighted local neighborhood law induced by $p$, and different targets can yield different patterns of
(pre-conditioner, residual) cancellation.

As an important example, for Laplace-type kernels commonly used in drifting-model implementations, the induced pre-conditioning is generally \emph{not} constant. In this case, both $\alpha_{\uppi}(\mathbf{x})$ and the residual $\boldsymbol{\delta}_{\uppi}(\mathbf{x})$ depend on the local kernel neighborhood and can vary with the underlying distribution $\uppi$.
As a result, even at the population level, the equilibrium condition $\Delta_{p,q}(\mathbf{x})\equiv \mathbf{0}$ does not necessarily imply score equality. Instead, it can be satisfied by a cancellation between a scaled score-mismatch and the residual term. Therefore, identifiability is not automatic and may require additional structural assumptions.

This concern can be amplified in practice because the drifting field is never evaluated at the population level. Implementations use finite-sample, mini-batch estimates with normalized (softmax) kernel weights, so the drift becomes a ratio-type Monte Carlo estimator, which can have non-negligible variance and bias. Empirically, the method often benefits from using more reference samples, consistent with this estimation sensitivity.
Moreover, many pipelines apply batch-dependent normalizations and heuristics, such as the balance between positive and negative sets, the choice and quality of the feature map, and tuning or aggregating across multiple temperatures~\cite{driftin2026}. These choices effectively reshape the kernel and hence the induced transport field. Finally, when the kernel is computed in a learned feature space, near-flat feature distances can make the softmax weights close to uniform, yielding a small drift magnitude even when $p$ and $q$ remain mismatched. Taken together, these effects imply that an observed near-equilibrium in practice does not uniquely certify $q\approx p$. It can also make training behavior sensitive to implementation choices that change the effective kernel and the resulting field.

In summary,
\begin{quote}
    {\color{BrickRed}\emph{Gaussian kernels lead to an identifiable score-based discrepancy. 
For Laplace-type kernels, identifiability is not automatic in general, because $\Delta_{p,q}$ can vanish through cancellation with the residual term even when the underlying scores differ.}
    }
\end{quote}

\section{Regimes Where Drifting Models Are Secretly Score-Based Models}\label{app:more-regimes}

\subsection{Geometry of the Residual for General Radial Kernels}
\label{app:radial_geometry}

This subsection derives the geometric formulas for the residual term
$\boldsymbol{\delta}_{\uppi}(\rvx)$ in \Cref{prop:radial_precond_decomp}, and in particular the Laplace-kernel identity used in the main text.

Recall that under the kernel-reweighted law $\rvy\sim\uppi_\tau(\cdot|\rvx)$,
\begin{equation*}
\tau^2\,\mathbf s_{\uppi,k_\tau}(\rvx)
=
\mathbb E\Big[b_\tau\big(\|\rvx-\rvy\|_2\big)(\rvy-\rvx)\Big].
\end{equation*}
Let
\begin{equation*}
r:=\|\rvx-\rvy\|_2,
\qquad
A:=b_\tau^{-1}(r)\in\mathbb R,
\qquad
\mathbf Z:=b_\tau(r)(\rvy-\rvx)\in\mathbb R^D.
\end{equation*}
Then
\begin{equation*}
\mathbf V_{\uppi,k_\tau}(\rvx)=\mathbb E[A\mathbf Z],
\qquad
\tau^2\,\mathbf s_{\uppi,k_\tau}(\rvx)=\mathbb E[\mathbf Z].
\end{equation*}
Therefore,
\begin{equation*}
\mathbf V_{\uppi,k_\tau}(\rvx)
=
\mathbb E[A]\mathbb E[\mathbf Z]+\operatorname{Cov}(A,\mathbf Z)
=
\tau^2\,\alpha_{\uppi}(\rvx)\,\mathbf s_{\uppi,k_\tau}(\rvx)
+
\boldsymbol{\delta}_{\uppi}(\rvx),
\end{equation*}
with
\begin{equation*}
\alpha_{\uppi}(\rvx)=\mathbb E[A],
\qquad
\boldsymbol{\delta}_{\uppi}(\rvx)=\operatorname{Cov}(A,\mathbf Z).
\end{equation*}

Conditioning on the radius $r$ yields a useful reformulation. Since $A=A(r)$ depends only on $r$,
\begin{equation}
\label{eq:delta_cond_radius_app}
\boldsymbol{\delta}_{\uppi}(\rvx)
=
\operatorname{Cov}(A,\mathbf Z)
=
\operatorname{Cov}\bigl(A(r),\,\mathbb E[\mathbf Z| r]\bigr).
\end{equation}
Thus the residual is nonzero only when the average pre-conditioned displacement varies systematically across radii.

To isolate the part that changes the direction relative to the score, assume $\mathbf s:=\mathbf s_{\uppi,k_\tau}(\rvx)\neq \mathbf 0$ and define
\begin{equation*}
\hat{\mathbf s}:=\frac{\mathbf s}{\|\mathbf s\|_2},
\qquad
\mathbf P_\perp:=\mathbf I-\hat{\mathbf s}\hat{\mathbf s}^\top,
\qquad
\boldsymbol{\delta}_\perp(\rvx):=\mathbf P_\perp\boldsymbol{\delta}_{\uppi}(\rvx).
\end{equation*}
Since $\mathbf P_\perp\mathbb E[\mathbf Z]=\mathbf 0$, we obtain
\begin{equation}
\label{eq:delta_perp_cond_radius_app}
\boldsymbol{\delta}_\perp(\rvx)
=
\operatorname{Cov}\bigl(A(r),\,\mathbb E[\mathbf P_\perp\mathbf Z| r]\bigr).
\end{equation}
Hence only the score-orthogonal part of the residual can change the direction of mean shift; the score-parallel component merely rescales the effective step size.

\paragraph{Laplace specialization.}
For the Laplace kernel, $b_\tau(r)=\tau/r$, so
\begin{equation*}
A(r)=\frac{r}{\tau}.
\end{equation*}
Let
\begin{equation*}
\rvu:=\frac{\rvy-\rvx}{\|\rvy-\rvx\|_2},
\qquad
\rvy-\rvx=r\,\rvu.
\end{equation*}
Then
\begin{equation*}
\mathbf Z=b_\tau(r)(\rvy-\rvx)=\frac{\tau}{r}(r\rvu)=\tau\,\rvu.
\end{equation*}
Writing
\begin{equation*}
\rvu_\perp:=\mathbf P_\perp\rvu,
\end{equation*}
we have $\mathbf P_\perp\mathbf Z=\tau\,\rvu_\perp$. Since $\mathbf P_\perp\mathbb E[\mathbf Z]=\mathbf 0$, it follows that $\mathbb E[\rvu_\perp]=\mathbf 0$, and \Cref{eq:delta_perp_cond_radius_app} reduces to
\begin{equation}
\label{eq:delta_perp_laplace_app}
\boldsymbol{\delta}_\perp(\rvx)
=
\operatorname{Cov}_{\rvy\sim\uppi_\tau(\cdot|\rvx)}(r,\rvu_\perp)
=
\mathbb E_{\rvy\sim\uppi_\tau(\cdot|\rvx)}[r\,\rvu_\perp].
\end{equation}

\Cref{eq:delta_perp_laplace_app} shows that the off-score component is a radius-weighted average of tangential directions. Therefore $\boldsymbol{\delta}_\perp(\rvx)$ is small whenever either
\begin{equation*}
\mathbb E[\rvu_\perp| r]\approx \mathbf 0
\quad\text{for all relevant }r,
\end{equation*}
so tangential contributions cancel within each distance band, or the kernel-reweighted radii concentrate so that $r$ is nearly constant and the weighting in $\mathbb E[r\,\rvu_\perp]$ is nearly uniform. The high-dimensional results in \Cref{sec:laplace} formalize the second mechanism: for Laplace kernels, kernel-reweighted radii concentrate in high dimension, which suppresses $\boldsymbol{\delta}_\perp(\rvx)$ and yields the alignment bounds in \Cref{thm:field_alignment_radial_full_xy,app:high-D}.

\begin{figure}[!t]
    \centering
    \begin{minipage}{0.32\textwidth}
        \centering
        \includegraphics[width=\linewidth]{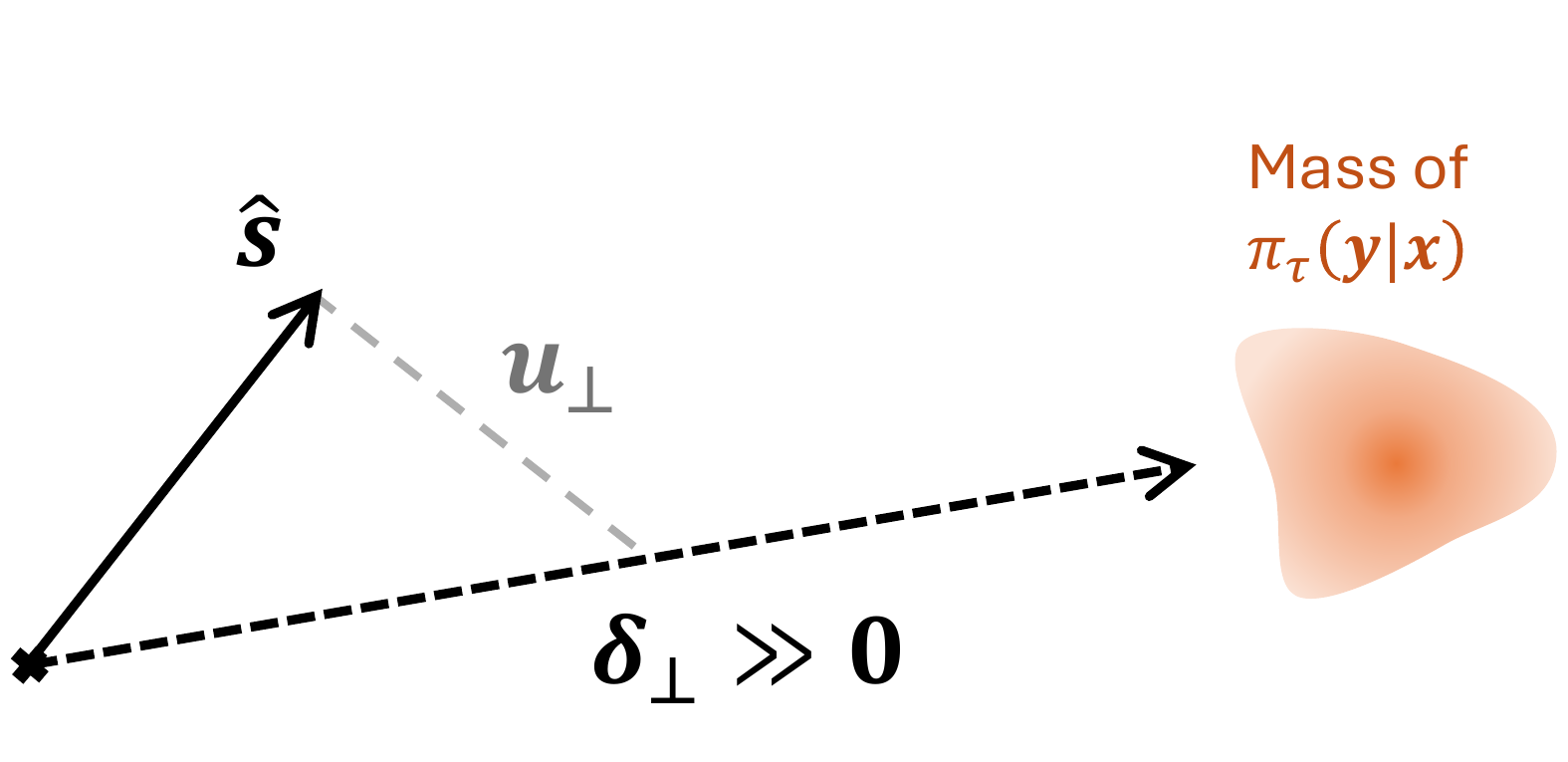}
    \end{minipage}
    \begin{minipage}{0.32\textwidth}
        \centering
        \includegraphics[width=\linewidth]{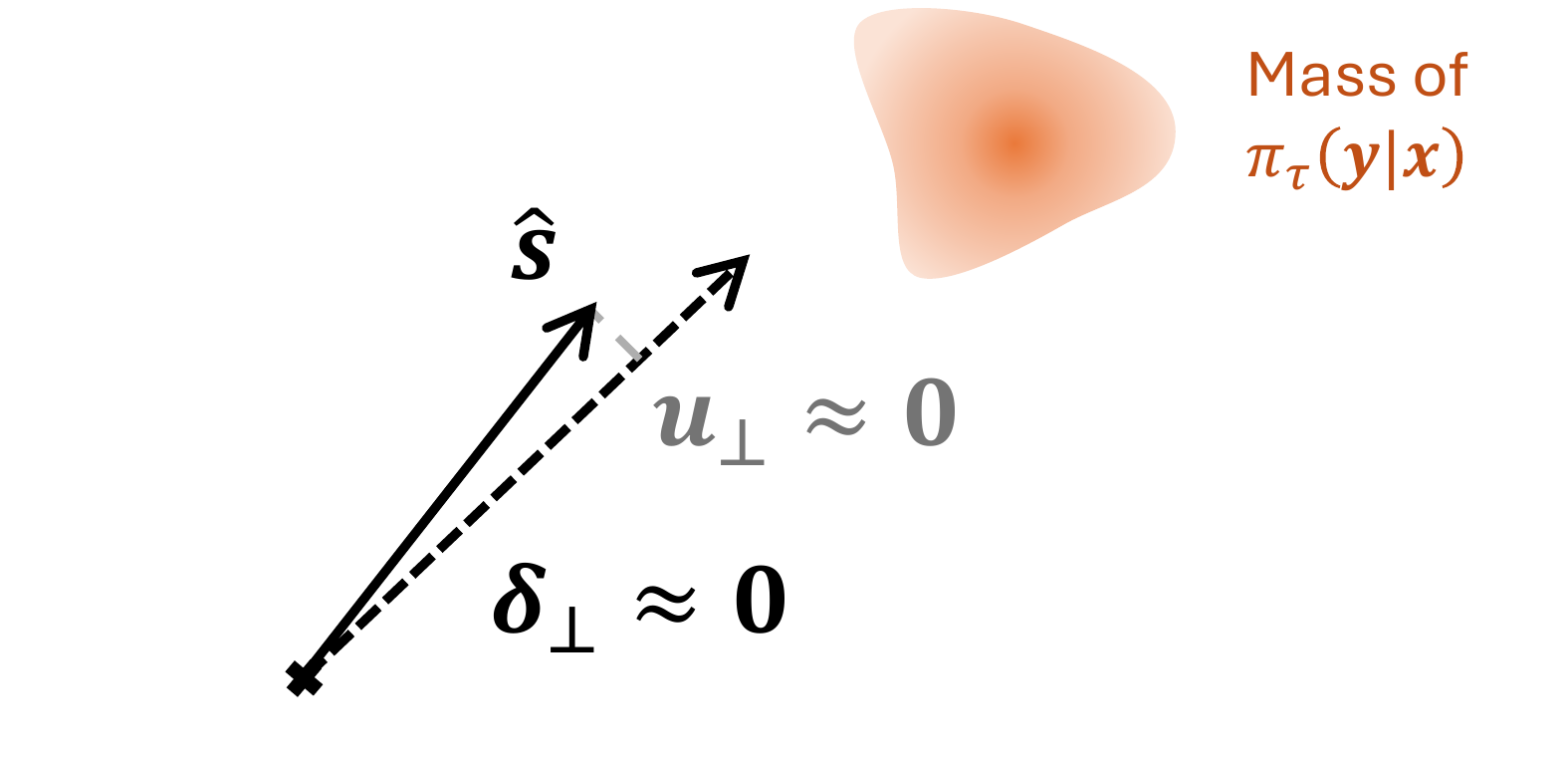}
    \end{minipage}
    \begin{minipage}{0.32\textwidth}
        \centering
        \includegraphics[width=0.8\linewidth]{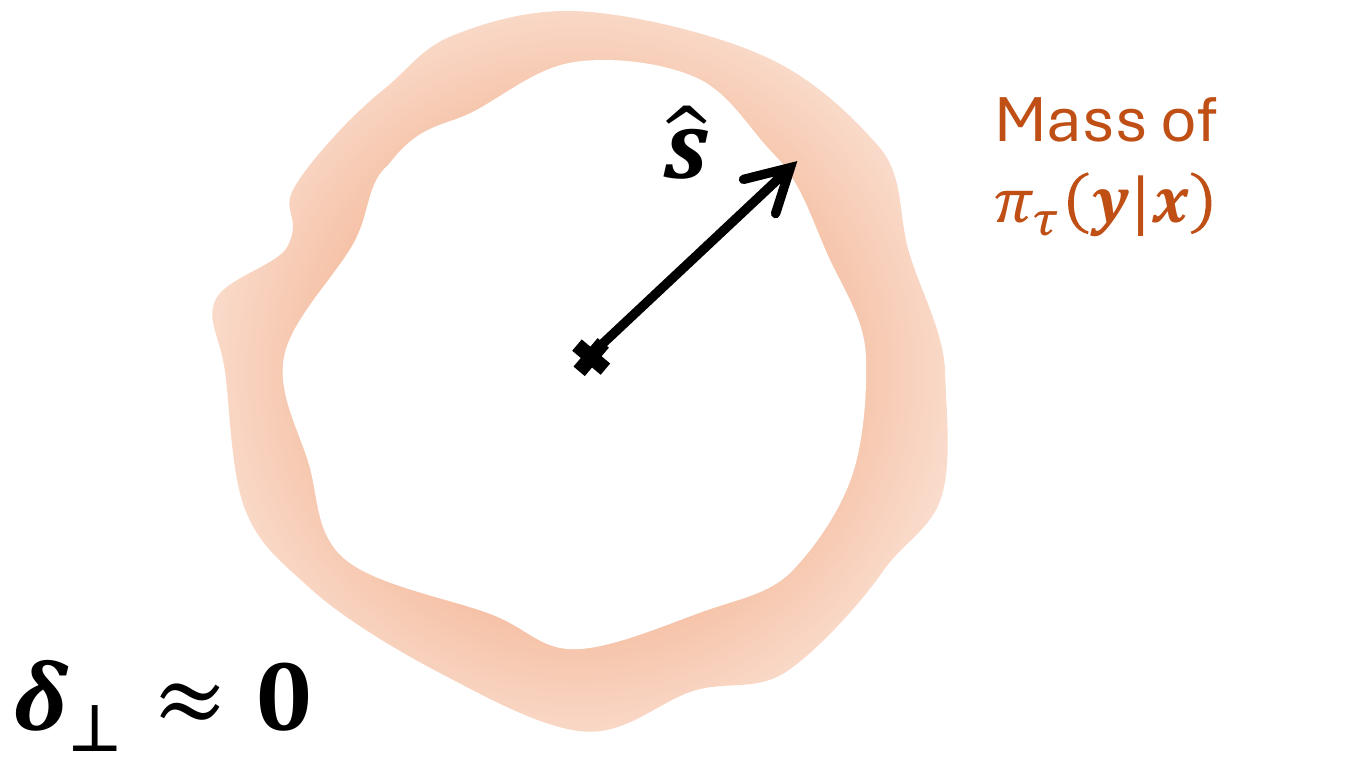}
    \end{minipage}
    \caption{\footnotesize{\textbf{Illustration of $\boldsymbol\delta_\perp(\rvx)$ in three illustrative examples.}
(a) $\boldsymbol\delta_\perp \gg \bm{0}$ because most of the mass of $\uppi_\tau(\cdot|\rvx)$ lies far away (large $r$) and in directions perpendicular to $\hat\rvs$;
(b) $\boldsymbol\delta_\perp \approx \bm{0}$ because most of the mass of $\uppi_\tau(\cdot|\rvx)$ lies closer (small $r$) and in directions parallel to $\hat\rvs$;
(c) $\boldsymbol\delta_\perp \approx \bm{0}$ because the contributions from different directions nearly cancel out.}} 
    \label{fig:illustration-of-delta}
\end{figure}

\subsection{High-Dimension Regime}\label{app:high-D}

In \Cref{sec:laplace}, we established the vector-field and objective-level alignment between drifting's mean-shift and score-mismatch. In this section, we develop two further notions of agreement: \emph{Algorithmic Gradient Alignment}, which concerns the actual update directions induced by the implemented stop-gradient solver, and \emph{Minimizer Alignment}, which concerns the agreement between the population optima of the two objectives. 

\paragraph{Algorithmic Gradient Alignment.}
The previous theorem compares the two discrepancy fields under the current model law $q_{\rvf}$. To connect this to the actual training dynamics, we now consider a parameterized generator $\rvf_\btheta$ and write $q_\btheta:=(\rvf_\btheta)_\# p_{\mathrm{prior}}$. 
The stop-gradient fixed-point loss used by drifting implements a transport--then--projection update: it first transports samples in data space and then projects the resulting targets back onto the generator family. The induced stop-gradient update direction is
\[
\mathbf g_{\mathrm{drift}}(\btheta)
:=
-2\,\mathbb E_{\beps}\Big[
\mathbf J_{\rvf_\btheta}(\beps)^\top\,
\eta\big(
\mathbf V_{p,k_\tau}(\rvf_\btheta(\beps))
-
\mathbf V_{q_\btheta,k_\tau}(\rvf_\btheta(\beps))
\big)
\Big],
\]
where $\eta>0$ is the transport step size in sample space and $\mathbf J_{\rvf_\btheta}$ is the Jacobian with respect to $\btheta$.

As a score-based comparator, we consider the analogous transport--then--projection update obtained by replacing the drift transport field with the scaled score-mismatch:
\[
\mathbf g_{\mathrm{ST}}(\btheta)
:=
-2\,\mathbb E_{\beps}\Big[
\mathbf J_{\rvf_\btheta}(\beps)^\top\,
\eta\,C\,
\big(
\rvs_{p,\tau}(\rvf_\btheta(\beps))
-
\rvs_{q_\btheta,\tau}(\rvf_\btheta(\beps))
\big)
\Big].
\]
This has the same Jacobian-transpose-times-vector form as DMD-style score-transport updates; see \Cref{sec:dmd}.

The next theorem shows that, in high dimension, the implemented drifting update becomes asymptotically indistinguishable from this score-transport update.

\begin{theorem}[(Informal) Large-$D$ Gradient Alignment]
\label{thm:grad_alignment_largeD}
Suppose that $k_\tau$ is a Laplace kernel. Assume the conditions of \Cref{thm:field_alignment_radial_full_xy} and a uniform second-moment bound on $\mathbf J_{\rvf_\btheta}$. Then for all $\btheta$ and all sufficiently large $D$,
\begin{mybox}
\[
\big\|
\mathbf g_{\mathrm{drift}}(\btheta)-\mathbf g_{\mathrm{ST}}(\btheta)
\big\|_2
=
\mathcal O(D^{-1/2}).
\]
\end{mybox}
Moreover, assuming that at least one of the two update directions has non-vanishing norm, we have
\begin{mybox}
\[
\cos\angle\big(
\mathbf g_{\mathrm{drift}}(\btheta),
\mathbf g_{\mathrm{ST}}(\btheta)
\big)
=
1-\mathcal O(D^{-1}).
\]
\end{mybox}
In particular, the hidden constant does not depend on learnable parameters.
\end{theorem}

 When the updates are small, the norm bound shows that stationarity for one update implies approximate stationarity for the other. When the updates are not small, the cosine bound shows that the two updates point in nearly the same direction, so they induce essentially the same optimization trajectory up to a rescaling.

\paragraph{Minimizer Alignment.}
Finally, we return to population optima in nonparametric form and study how close the minimizers are at the level of the induced densities.  
To quantify proximity between a model distribution and the data through their kernel-smoothed scores, we use the scale-$\tau$ reverse Fisher divergence
\[
\mathcal D_{\mathrm{rF}}(p\|q)
:=
\mathbb{E}_{\mathbf{x}\sim q}
\bigl\|\rvs_{p,\tau}(\mathbf{x})-\rvs_{q,\tau}(\mathbf{x})\bigr\|_2^2 .
\]
For a generator $\rvf$, this becomes
\[
\mathcal D_{\mathrm{rF}}(p\|q_{\rvf})
=
\mathbb E_{\mathbf{x}\sim q_{\rvf}}
\bigl\|\rvs_{p,\tau}(\mathbf{x})-\rvs_{q_{\rvf},\tau}(\mathbf{x})\bigr\|_2^2,
\]
which directly measures the squared mismatch between the smoothed data score and the smoothed model score under the model law.

Under the same high-dimensional regularity conditions as before, any population minimizer of the Laplace drifting objective induces a model distribution whose kernel-smoothed score is close to that of the data, with an error that decays polynomially in $D$. 
Although the Laplace mean-shift objective is not explicitly formulated as score matching, its minimizer nevertheless achieves vanishing reverse Fisher discrepancy at rate $\mathcal O\bigl(D^{-(1+2a)}\bigr)$. 
We summarize this high-dimensional agreement result below.

\begin{theorem}[(Informal) High-Dimensional Agreement Between Mean-Shift and Score Matching]\label{thm:D_consistency}
Suppose that $k_\tau$ is a Laplace kernel. Assume that distributions we consider concentrate on a common-radius shell and have controlled inner products and moments. Let
\[
\rvf^{\star}\in\arg\min_{\rvf} \mathcal{L}_{\mathrm{drift}}(\rvf),
\qquad
\rvg^{\star}\in\arg\min_{\rvg} \mathcal{L}_{\mathrm{SM}}(\rvg).
\]
Then as $D$ is large, 
\begin{mybox}
    \[
\mathcal{D}_{\mathrm{rF}}\bigl(p\| q_{\rvf^\star}\bigr)
= \mathcal{D}_{\mathrm{rF}}\bigl(q_{\rvg^\star}\| q_{\rvf^\star}\bigr)
=
\mathcal{O}\left(D^{-(1+2a)}\right).
\]
\end{mybox}
In particular, the hidden constant in $\mathcal{O}(\cdot)$ is independent of $D$ and of the learnable parameters used to parametrize $\rvf$ and $\rvg$.
\end{theorem}

To obtain similar bounds in other distributional divergences (e.g., KL or total variation) from a reverse-Fisher guarantee,
one needs an additional regularity condition on the reference distribution.
A standard sufficient assumption is that the reference distribution $\uppi$ satisfies a log-Sobolev inequality;
we do not pursue these implications further here.


\subsection{Practical-Implementation-Aligned High-Dimensional Regime}
\label{app:impl_highD}

The population results in \Cref{sec:laplace,app:high-D} compare the ideal Laplace discrepancy field
\[
\Delta_{p,q}(\mathbf{x})
:=
\mathbf{V}_{p,k_\tau}(\mathbf{x})-\mathbf{V}_{q,k_\tau}(\mathbf{x})
\]
with the corresponding score-mismatch
\[
\Delta \mathbf{s}_{p,q}(\mathbf{x})
:=
\mathbf{s}_{p,\tau}(\mathbf{x})-\mathbf{s}_{q,\tau}(\mathbf{x}).
\]
This population formulation is clean, but direct kernel estimation in high dimension is fragile in practice: distances concentrate, the kernel can become either too flat or too spiky unless the temperature is chosen carefully, accurate estimation typically requires large sample sets, and the method benefits from working in a structured pre-trained feature space rather than raw ambient space. Accordingly, practical implementations of drifting models are necessarily more carefully engineered. They compute the transport target in a frozen feature space, rescale distances by a batch-dependent scale, symmetrize affinities through the geometric mean of a softmax over targets and a softmax over queries, concatenate detached generated features with external negative and positive pools, sum force-normalized updates across temperatures to form one detached target per feature head, and finally sum the resulting per-head regression losses. 

The goal of this subsection is to extend \Cref{app:high-D,sec:laplace} by giving an implementation-aligned analogue of the high-dimensional field and semi-gradient alignment results for this practical procedure. We fix one frozen feature head and work entirely in that feature space, suppressing the pushforward notation. We also fix one class label $c$ and one temperature $\tau>0$. Here $\tau$ denotes the temperature after the same batch-scale normalization used in the practical implementation. Thus all vectors below live in the same normalized frozen feature space in which the drifting loss is evaluated at  $\tau$.

\paragraph{Implementation-Aligned Formulation at One Temperature.}
For one batch of generated queries
\[
\mathbf{x}_1,\dots,\mathbf{x}_{C_g}\in\mathbb R^D,
\]
let
\[
G:=\{\mathbf{y}_1^G,\dots,\mathbf{y}_{C_g}^G\},
\qquad
U:=\{\mathbf{y}_1^U,\dots,\mathbf{y}_{C_n}^U\},
\qquad
P:=\{\mathbf{y}_1^P,\dots,\mathbf{y}_{C_p}^P\}
\]
denote, respectively, the stop-gradient copy of the current generated batch, the additional unconditional negative pool used by classifier-free guidance~\cite{ho2021classifier}, and the class-positive real pool. The practical implementation concatenates them into the full target pool
\[
T:=G\cup U\cup P.
\]
In the current training loop, the target weights satisfy
\[
w_j=1\quad\text{for }\mathbf{z}_j\in G\cup P,
\qquad
w_j=w_{\mathrm{cfg}}\quad\text{for }\mathbf{z}_j\in U,
\]
where $w_{\mathrm{cfg}}$ is the classifier-free-guidance coefficient used by the practical implementation for the unconditional negatives.

Let $d_{ij}$ denote the normalized query--target distances that actually enter the logits after the same batch-scale normalization used by the implementation, with the self-pairs inside the $G$ block masked exactly as there. Writing $\mathbf{z}_j\in T$ for the concatenated targets, the practical implementation forms
\[
A_{ij}^{\mathrm{row}}
:=
\frac{\exp(-d_{ij}/\tau)}{\sum_{k\in T}\exp(-d_{ik}/\tau)},
\qquad
A_{ij}^{\mathrm{col}}
:=
\frac{\exp(-d_{ij}/\tau)}{\sum_{\ell=1}^{C_g}\exp(-d_{\ell j}/\tau)},
\]
and then the symmetrized affinity
\[
a_{ij}:=w_j\sqrt{A_{ij}^{\mathrm{row}}A_{ij}^{\mathrm{col}}}.
\]
This is the full affinity matrix used by the practical implementation at temperature $\tau$.

Because the practical implementation first splits the full affinity matrix into one negative block $G\cup U$ and one positive block $P$, we define
\[
m_i^-:=\sum_{j\in G\cup U}a_{ij},
\qquad
m_i^P:=\sum_{j\in P}a_{ij},
\]
and the corresponding block barycenters
\[
\boldsymbol{\mu}_i^-:=\frac{1}{m_i^-}\sum_{j\in G\cup U}a_{ij}\mathbf{y}_j,
\qquad
\boldsymbol{\mu}_i^P:=\frac{1}{m_i^P}\sum_{j\in P}a_{ij}\mathbf{y}_j^P.
\]
The one-temperature implementation force at one temperature is then
\begin{equation}
\label{eq:impl_force_def_main}
\mathbf{F}_{\tau}(\mathbf{x}_i)
:=
m_i^P m_i^-\bigl(\boldsymbol{\mu}_i^P-\boldsymbol{\mu}_i^-\bigr).
\end{equation}

The next proposition refines \Cref{eq:impl_force_def_main} into the more interpretable class-positive / self-generated / unconditional decomposition that we use in the score comparison.

\begin{proposition}[Exact Implementation-Level Decomposition]
\label{prop:impl_exact_decomp}
Define
\[
m_i^G:=\sum_{j\in G}a_{ij},
\qquad
m_i^U:=\sum_{j\in U}a_{ij},
\]
and
\[
\boldsymbol{\mu}_i^G:=\frac{1}{m_i^G}\sum_{j\in G}a_{ij}\mathbf{y}_j^G,
\qquad
\boldsymbol{\mu}_i^U:=\frac{1}{m_i^U}\sum_{j\in U}a_{ij}\mathbf{y}_j^U.
\]
Then the one-temperature implementation force in \Cref{eq:impl_force_def_main} admits the exact representation
\begin{mybox}
\begin{equation}
\label{eq:impl_exact_force_main}
\mathbf{F}_{\tau}(\mathbf{x}_i)
=
m_i^P m_i^G\bigl(\boldsymbol{\mu}_i^P-\boldsymbol{\mu}_i^G\bigr)
+
m_i^P m_i^U\bigl(\boldsymbol{\mu}_i^P-\boldsymbol{\mu}_i^U\bigr).
\end{equation}
\end{mybox}
Thus the practical implementation is exactly a sum of a class-positive versus self-generated transport term and a class-positive versus unconditional transport term.
\end{proposition}

\paragraph{Two Measurable Practical Implementation Deviations.}
To compare \Cref{eq:impl_exact_force_main} with a population score field, we isolate the two places where the practical implementation differs from the clean population construction.

First, remove only the query-axis softmax while keeping the same full target pool, the same normalized distances, the same self-mask, and the same target weights. This gives the row-softmax-only affinity
\[
\bar a_{ij}:=w_jA_{ij}^{\mathrm{row}}.
\]
Using the same blocks $G,U,P$, let
\[
\bar{\boldsymbol{\mu}}_i^G,
\qquad
\bar{\boldsymbol{\mu}}_i^U,
\qquad
\bar{\boldsymbol{\mu}}_i^P
\]
denote the corresponding row-softmax-only barycenters. The first practical implementation deviation is then
\[
\varepsilon_{\mathrm{sym},\tau}^2
:=
\mathbb E\Big[
\|\boldsymbol{\mu}_i^G-\bar{\boldsymbol{\mu}}_i^G\|_2^2
+
\|\boldsymbol{\mu}_i^U-\bar{\boldsymbol{\mu}}_i^U\|_2^2
+
\|\boldsymbol{\mu}_i^P-\bar{\boldsymbol{\mu}}_i^P\|_2^2
\Big].
\]
This is not a Monte Carlo error. It measures only the distortion introduced by the query-axis softmax in the affinity symmetrization.

Second, compare the row-softmax-only finite-pool barycenters with the corresponding population kernel barycenters. Let
\[
\boldsymbol{\mu}_{\pi,k_\tau}(\mathbf{x})
:=
\frac{\mathbb E_{\mathbf{y}\sim\pi}[k_\tau(\mathbf{x},\mathbf{y})\,\mathbf{y}]}
{\mathbb E_{\mathbf{y}\sim\pi}[k_\tau(\mathbf{x},\mathbf{y})]}
=
\mathbf{x}+\mathbf{V}_{\pi,k_\tau}(\mathbf{x}),
\qquad
k_\tau(\mathbf{x},\mathbf{y})=\exp(-d(\mathbf{x},\mathbf{y})/\tau),
\]
where $d(\cdot,\cdot)$ is the same normalized distance used above. The second deviation is
\[
\varepsilon_{\mathrm{mc},\tau}^2
:=
\mathbb E\Big[
\|\bar{\boldsymbol{\mu}}_i^G-\boldsymbol{\mu}_{q_c,k_\tau}(\mathbf{x}_i)\|_2^2
+
\|\bar{\boldsymbol{\mu}}_i^U-\boldsymbol{\mu}_{p_{\emptyset},k_\tau}(\mathbf{x}_i)\|_2^2
+
\|\bar{\boldsymbol{\mu}}_i^P-\boldsymbol{\mu}_{p_c,k_\tau}(\mathbf{x}_i)\|_2^2
\Big].
\]
This is the actual finite-pool-to-oracle kernel-estimation error. In practice it is controlled by the effective sample size of the kernel weights rather than by the raw pool size alone.

We now define the score-based comparator that mirrors \Cref{eq:impl_exact_force_main}. Let
\[
\Delta \mathbf{s}_{p_c,q_c}(\mathbf{x})
:=
\mathbf{s}_{p_c,\tau}(\mathbf{x})-\mathbf{s}_{q_c,\tau}(\mathbf{x}),
\qquad
\Delta \mathbf{s}_{p_c,p_{\emptyset}}(\mathbf{x})
:=
\mathbf{s}_{p_c,\tau}(\mathbf{x})-\mathbf{s}_{p_{\emptyset},\tau}(\mathbf{x}),
\]
and define the one-temperature guided score field by
\begin{equation}
\label{eq:impl_guided_score_main}
\boldsymbol{\Gamma}_{\tau}(\mathbf{x}_i)
:=
m_i^P m_i^G\,\Delta \mathbf{s}_{p_c,q_c}(\mathbf{x}_i)
+
m_i^P m_i^U\,\Delta \mathbf{s}_{p_c,p_{\emptyset}}(\mathbf{x}_i).
\end{equation}

The only population input now needed is the already-proved large-$D$ Laplace field-alignment theorem, applied to the two pairs that truly appear in the practical implementation, namely $(p_c,q_c)$ and $(p_c,p_{\emptyset})$. We write $C_{\tau}>0$ for the corresponding scale factor and
\[
K_{\tau}:=\max\{K_{PG,\tau},K_{PU,\tau}\}
\]
for the corresponding pairwise alignment constant.

Finally, because $0\le a_{ij}\le w_j$, we define the deterministic weight bound
\[
W_G:=\sum_{j\in G}w_j,
\qquad
W_U:=\sum_{j\in U}w_j,
\qquad
W_P:=\sum_{j\in P}w_j,
\qquad
M_{\tau}:=W_PW_G+W_PW_U.
\]
In the current practical implementation, $W_G$ and $W_P$ are determined by the positive and generated batch sizes, while $W_U$ also carries the classifier-free-guidance weight on the unconditional negatives.

\begin{theorem}[Practical-Implementation-Aligned One-Temperature Field Alignment]
\label{thm:impl_field_align}
Assume that the same feature-space large-$D$ Laplace field-alignment theorem from \Cref{sec:laplace} applies, at temperature $\tau$, to the two pairs $(p_c,q_c)$ and $(p_c,p_{\emptyset})$, with common scale $C_{\tau}$ and constant $K_{\tau}$. Then there exists a universal numerical constant $c_0>0$ such that
\begin{mybox}
\begin{equation}
\label{eq:impl_field_align_main}
\mathbb E_{\mathbf{x}_i\sim q_c}
\bigl\|
\mathbf{F}_{\tau}(\mathbf{x}_i)-C_{\tau}\boldsymbol{\Gamma}_{\tau}(\mathbf{x}_i)
\bigr\|_2^2
\le
c_0 M_{\tau}^2
\left(
\varepsilon_{\mathrm{sym},\tau}^2
+
\varepsilon_{\mathrm{mc},\tau}^2
+
\frac{K_{\tau}}{D}
\right).
\end{equation}
\end{mybox}
Therefore the actual one-temperature practical implementation is close to a guided score-mismatch field. The three terms on the right-hand side are, respectively, the symmetrization distortion, the finite-pool kernel-estimation error, and the population high-dimensional discrepancy.
\end{theorem}

The previous theorem is an $L^2$ closeness statement. As in the population case, it yields directional alignment once the guided score field does not vanish. Define the normalized one-temperature fields by
\[
\overline{\mathbf{F}}_{\tau}(\mathbf{x})
:=
\frac{\mathbf{F}_{\tau}(\mathbf{x})}
{\bigl(\mathbb E_{q_c}\|\mathbf{F}_{\tau}(\mathbf{x})\|_2^2\bigr)^{1/2}},
\qquad
\overline{\boldsymbol{\Gamma}}_{\tau}(\mathbf{x})
:=
\frac{\boldsymbol{\Gamma}_{\tau}(\mathbf{x})}
{\bigl(\mathbb E_{q_c}\|\boldsymbol{\Gamma}_{\tau}(\mathbf{x})\|_2^2\bigr)^{1/2}}.
\]

\begin{corollary}[Practical-Implementation-Aligned Cosine Similarity]
\label{cor:impl_cosine}
Assume the hypotheses of \Cref{thm:impl_field_align}. Suppose moreover that
\[
\mathbb E_{\mathbf{x}\sim q_c}\|\boldsymbol{\Gamma}_{\tau}(\mathbf{x})\|_2^2\ge \gamma_{\tau}>0,
\]
and that the field-alignment error from \Cref{eq:impl_field_align_main} is at most $C_{\tau}^2\gamma_{\tau}/4$. Then
\begin{mybox}
\[
\mathbb E_{\mathbf{x}\sim q_c}
\bigl[
\langle \overline{\mathbf{F}}_{\tau}(\mathbf{x}),\overline{\boldsymbol{\Gamma}}_{\tau}(\mathbf{x})\rangle
\bigr]
\ge
1-
\frac{c_1M_{\tau}^2}{\gamma_{\tau}}
\left(
\varepsilon_{\mathrm{sym},\tau}^2
+
\varepsilon_{\mathrm{mc},\tau}^2
+
\frac{K_{\tau}}{D}
\right)
\]
\end{mybox}
for a universal numerical constant $c_1>0$.
\end{corollary}

\paragraph{A Remark on Typical Scale of the Two Practical Deviations.}
The two deviations above have different meanings and therefore different typical scales. 
For the symmetrization term, writing
\[
Z_i^{\mathrm{row}}
:=
\sum_{k\in T}\exp(-d_{ik}/\tau),
\qquad
Z_j^{\mathrm{col}}
:=
\sum_{\ell=1}^{C_g}\exp(-d_{\ell j}/\tau),
\]
we have
\[
A_{ij}^{\mathrm{row}}=\frac{\exp(-d_{ij}/\tau)}{Z_i^{\mathrm{row}}},
\qquad
A_{ij}^{\mathrm{col}}=\frac{\exp(-d_{ij}/\tau)}{Z_j^{\mathrm{col}}},
\qquad
a_{ij}
=
\bar a_{ij}\sqrt{\frac{Z_i^{\mathrm{row}}}{Z_j^{\mathrm{col}}}}.
\]
Thus $\varepsilon_{\mathrm{sym},\tau}^2$ measures how much the extra factor
$\sqrt{Z_i^{\mathrm{row}}/Z_j^{\mathrm{col}}}$ perturbs the row-softmax-only barycenters. In particular, it is small when, within each block $\alpha\in\{G,U,P\}$, the column partition factors $Z_j^{\mathrm{col}}$ vary little across $j$, so that the query-axis softmax acts almost like a common rescaling inside that block.

By contrast, $\varepsilon_{\mathrm{mc},\tau}^2$ is a genuine finite-pool approximation error. For each block $\alpha\in\{G,U,P\}$, let
\[
\pi_G:=q_c,
\qquad
\pi_U:=p_{\emptyset},
\qquad
\pi_P:=p_c,
\]
and define the normalized row-softmax-only weights
\[
\bar w_{ij}^{\alpha,\tau}
:=
\frac{\bar a_{ij}}{\sum_{k\in\alpha}\bar a_{ik}},
\qquad j\in\alpha.
\]
Then
\[
\bar{\boldsymbol{\mu}}_i^\alpha
=
\sum_{j\in\alpha}\bar w_{ij}^{\alpha,\tau}\mathbf{y}_j^\alpha,
\]
so $\varepsilon_{\mathrm{mc},\tau}^2$ is a self-normalized kernel-estimation error. Its typical scale is governed not by the raw block size alone, but by the effective sample size
\[
N_{\mathrm{eff},i}^{\alpha,\tau}
:=
\left(\sum_{j\in\alpha}(\bar w_{ij}^{\alpha,\tau})^2\right)^{-1}.
\]
Schematically,
\[
\varepsilon_{\mathrm{mc},\tau}^2
\;\lesssim\;
\sum_{\alpha\in\{G,U,P\}}
\mathbb E_{\mathbf{x}_i\sim q_c}
\left[
\frac{\sigma_{\alpha,\tau}^2(\mathbf{x}_i)}
{N_{\mathrm{eff},i}^{\alpha,\tau}}
\right],
\]
where
\[
\sigma_{\alpha,\tau}^2(\mathbf{x})
:=
\frac{
\mathbb E_{\mathbf{y}\sim\pi_\alpha}
\!\left[
k_\tau(\mathbf{x},\mathbf{y})
\,
\|\mathbf{y}-\boldsymbol{\mu}_{\pi_\alpha,k_\tau}(\mathbf{x})\|_2^2
\right]
}{
\mathbb E_{\mathbf{y}\sim\pi_\alpha}\!\left[k_\tau(\mathbf{x},\mathbf{y})\right]
}
\]
is the corresponding kernel-posterior variance. This bound is only interpretive and is not used in the proofs in \Cref{app:proof_impl_highD}, but it makes the main point clear: increasing the raw pool size helps only insofar as it enlarges the effective sample size. When the kernel weights remain reasonably diffuse, $\varepsilon_{\mathrm{mc},\tau}^2$ is small; when they collapse onto only a few targets, it can remain non-negligible even for large pools.

\paragraph{From One Temperature to the Actual Stop-Gradient Update.}
The practical implementation computes such one-temperature forces for every temperature in its set $\mathcal T$, normalizes each one by its empirical RMS magnitude, sums these normalized forces to form a detached target for each feature head, and only then forms one regression loss per feature head. The total training loss is the sum of these per-head losses. This is important: the practical implementation does \emph{not} sum losses across temperatures.

Let $\mathcal L$ denote the set of frozen feature heads used by the practical implementation, and let $\mathcal T$ denote its temperature set. For each pair $(\ell,\tau)\in\mathcal L\times\mathcal T$, let
\[
\mathbf{F}_{\ell,\tau},
\quad
\boldsymbol{\Gamma}_{\ell,\tau},
\quad
C_{\ell,\tau},
\quad
K_{\ell,\tau},
\quad
M_{\ell,\tau},
\quad
\varepsilon_{\mathrm{sym},\ell,\tau}^2,
\quad
\varepsilon_{\mathrm{mc},\ell,\tau}^2
\]
denote the corresponding one-temperature quantities defined above for the $\ell$-th feature head. Let $\beta_{\ell,\tau}>0$ denote the inverse RMS normalization factor used by the practical implementation for that feature head and temperature. For each feature head $\ell$, define the aggregated one-head implementation field and guided score field by
\[
\widehat{\Delta}_{\boldsymbol{\theta},\ell}^{\mathrm{impl}}(\mathbf{x})
:=
\sum_{\tau\in\mathcal T}\beta_{\ell,\tau}\,\mathbf{F}_{\ell,\tau}(\mathbf{x}),
\qquad
\widehat{\Delta}_{\boldsymbol{\theta},\ell}^{\mathrm{guide}}(\mathbf{x})
:=
\sum_{\tau\in\mathcal T}\beta_{\ell,\tau}C_{\ell,\tau}\,\boldsymbol{\Gamma}_{\ell,\tau}(\mathbf{x}).
\]
The corresponding stop-gradient semi-gradients are
\[
\mathbf{g}_{\mathrm{impl}}(\boldsymbol{\theta})
:=
-2\,\mathbb E_{\boldsymbol{\epsilon}}
\Big[
\sum_{\ell\in\mathcal L}
\mathbf{J}_{\Psi_{\ell}\circ \mathbf{f}_{\boldsymbol{\theta}}}(\boldsymbol{\epsilon})^{\top}
\widehat{\Delta}_{\boldsymbol{\theta},\ell}^{\mathrm{impl}}(\mathbf{f}_{\boldsymbol{\theta}}(\boldsymbol{\epsilon}))
\Big],
\]
\[
\mathbf{g}_{\mathrm{guide}}(\boldsymbol{\theta})
:=
-2\,\mathbb E_{\boldsymbol{\epsilon}}
\Big[
\sum_{\ell\in\mathcal L}
\mathbf{J}_{\Psi_{\ell}\circ \mathbf{f}_{\boldsymbol{\theta}}}(\boldsymbol{\epsilon})^{\top}
\widehat{\Delta}_{\boldsymbol{\theta},\ell}^{\mathrm{guide}}(\mathbf{f}_{\boldsymbol{\theta}}(\boldsymbol{\epsilon}))
\Big].
\]
\begin{theorem}[Practical-Implementation-Aligned Semi-Gradient Alignment]
\label{thm:impl_semigrad}
Assume that \Cref{thm:impl_field_align} holds for every $(\ell,\tau)\in\mathcal L\times\mathcal T$. Assume furthermore that, for each feature head $\ell$, there exists $J_{\ell}<\infty$ such that
\[
\mathbb E_{\boldsymbol{\epsilon}}
\bigl[
\|\mathbf{J}_{\Psi_{\ell}\circ \mathbf{f}_{\boldsymbol{\theta}}}(\boldsymbol{\epsilon})\|_{\mathrm{op}}^2
\bigr]
\le J_{\ell}^2
\qquad
\text{for all }\boldsymbol{\theta}.
\]
Then
\begin{mybox}
\begin{equation}
\label{eq:impl_grad_bound_main}
\|\mathbf{g}_{\mathrm{impl}}(\boldsymbol{\theta})-\mathbf{g}_{\mathrm{guide}}(\boldsymbol{\theta})\|_2
\le
2\sum_{\ell\in\mathcal L}\sum_{\tau\in\mathcal T}
\beta_{\ell,\tau}J_{\ell}
\left[
c_0 M_{\ell,\tau}^2
\left(
\varepsilon_{\mathrm{sym},\ell,\tau}^2
+
\varepsilon_{\mathrm{mc},\ell,\tau}^2
+
\frac{K_{\ell,\tau}}{D}
\right)
\right]^{1/2}.
\end{equation}
\end{mybox}
Consequently, if the prefactor
\[
\sum_{\ell\in\mathcal L}\sum_{\tau\in\mathcal T}\beta_{\ell,\tau}J_{\ell}M_{\ell,\tau}
\]
remains uniformly bounded, then the practical implementation stop-gradient update is asymptotically aligned with a multitemperature guided score-transport semi-gradient:
\[
\|\mathbf{g}_{\mathrm{impl}}(\boldsymbol{\theta})-\mathbf{g}_{\mathrm{guide}}(\boldsymbol{\theta})\|_2
=
\mathcal O\!\left(
D^{-1/2}
+
\max_{\ell,\tau}\varepsilon_{\mathrm{sym},\ell,\tau}
+
\max_{\ell,\tau}\varepsilon_{\mathrm{mc},\ell,\tau}
\right).
\]
\end{theorem}

The message of \Cref{thm:impl_semigrad} is that the practical implementation retains the same score-based mechanism after the implementation-specific modifications are taken into account. In particular, after accounting for the two measurable practical deviations, namely the affinity symmetrization distortion and the finite-pool kernel-estimation error, the actual stop-gradient update is close to a guided multitemperature score-transport semi-gradient in the same frozen feature space.


\subsection{Low-Temperature Regime}\label{app:low-tem}
Beyond the high-dimensional regime, we also show that in the low-temperature regime, with fixed dimension $D$ and small $\tau$, the population minimizer of drifting induces a distribution close to the data distribution, mirroring score matching.
When $\tau$ is small, the Laplace kernel $k_\tau(\rvx,\rvy)=\exp(-\|\rvx-\rvy\|_2/\tau)$ concentrates its mass near $\rvx$,
so the mean-shift direction $\mathbf V_{\uppi,k_\tau}(\rvx)$ becomes a purely local statistic of the density $\uppi$.
A Taylor expansion then reveals that this local displacement is proportional to the kernel-smoothed score $\rvs_{\uppi,\tau}(\rvx)$, up to higher-order terms.

The only remaining issue is that mean shift is defined through a ratio.
Specifically, we write $\mathbf V_{\uppi,k_\tau}(\rvx)=B_\tau(\rvx)/A_\tau(\rvx)$, where
\[
A_\tau(\rvx):=\int_{\mathbb R^D} e^{-\|\rvx-\rvy\|_2/\tau}\,\uppi(\rvy)\,\diff\rvy,
\qquad
B_\tau(\rvx):=\int_{\mathbb R^D} e^{-\|\rvx-\rvy\|_2/\tau}\,(\rvy-\rvx)\,\uppi(\rvy)\,\diff\rvy.
\]
Thus we need to ensure the denominator $A_\tau(\rvx)$ does not become too small and that the Taylor remainder stays well behaved after dividing by $A_\tau(\rvx)$.
\Cref{ass:small_tau_envelope_uniform} provides exactly this by requiring  moment control of the local smoothness of $\log\uppi$ and of how much $\uppi(\rvx)$ can differ from nearby values.

Under this mild regularity, the theorem below states that any population drifting minimizer induces a distribution whose kernel-smoothed score matches that of the data up to $\mathcal O(\tau^2)$ pointwise,
and hence the scale-$\tau$ Fisher divergence decays as $\mathcal O(\tau^4)$. In short, we have the following:

\begin{theorem}[(Informal) Small-$\bar\tau$ Agreement between Mean-Shift and Score Matching]
\label{thm:small_bar_tau_consistency}
Suppose that $k_\tau$ is a Laplace kernel. Let $\tau_0>0$.
For each $\tau\in(0,\tau_0]$, pick any population drifting minimizer
\[
\rvf^{\star}(\tau)\in\arg\min_{\rvf}\mathcal L_{\mathrm{drift}}(\rvf),
\qquad
\rvg^{\star}(\tau)\in\arg\min_{\rvg} \mathcal{L}_{\mathrm{SM}}(\rvg).
\]
Assume that, uniformly over such drifting-model minimizers, the local derivatives and local density ratios of $p$ and $q_{\rvf^\star(\tau)}$ admit integrable (moment) control. Then, for small $\tau$,
\begin{mybox}
    \[
    \mathcal D_{\mathrm{rF}}\bigl(q_{\rvg^\star(\tau)}\| q_{\rvf^\star(\tau)}\bigr) = \mathcal D_{\mathrm{rF}}\bigl(p\|q_{\rvf^\star(\tau)}\bigr)  
    =\mathcal O(\tau^4).
    \]
\end{mybox}
The hidden constant is independent of $\bar\tau$ and of the learnable parameters.
\end{theorem}
We refer to \Cref{app:proof-small-tau} for a rigorous statement and the full proof.

Even though we also derive an analogous minimizer-level alignment in the small-temperature regime, one could in principle establish the same three components there as well: objective-level equivalence, (semi-)gradient alignment, and minimizer agreement. In practice, however, the small-$\tau$ analysis requires additional $\tau$-uniform local smoothness and non-degeneracy conditions. These are needed to control Taylor remainders and denominators arising in ratio expansions, but they make the presentation substantially more technical.
To keep the main narrative focused and readable, we therefore emphasize the high-dimensional results in this article.

\section{Empirical Setup, Additional Results, and Generated Samples}\label{app:gen_samples}

\subsection{Experimental Setup for Synthetic Datasets}\label{app:toy-setup}
This appendix records the exact oracle constructions and finite-sample estimators used in \Cref{subsec:emp-toy}.

\paragraph{Synthetic Distributions.}
For each ambient dimension $D$, we construct two Gaussian-mixture pairs $(p,q)$ and treat them as oracle samplers. In all cases, query points are sampled from $q$, and independent reference sets are drawn from both $p$ and $q$.

\subparagraph{(A) Ring MoG.}
Both $p$ and $q$ are six-mode mixtures of Gaussians in $\mathbb{R}^D$. For each $D$, we first choose a random two-dimensional plane and place six mode centers equally spaced on a ring of radius $R=3$ inside that plane. To draw a sample, we choose one mode uniformly and add isotropic Gaussian noise with standard deviation $0.40$. We define $q$ by rotating the entire ring by a fixed angle of $\pi/6$ within the same plane and sampling in the same way. Thus, the mismatch between $p$ and $q$ is controlled by a fixed angular offset, while the overall geometric structure is preserved as $D$ increases.

\subparagraph{(B) Raw MoG.}
Here $p$ and $q$ are mixtures with different numbers of modes and different radius profiles. For $p$, we place six mode centers at radii $\{1.5,\,2.5,\,3.0,\,4.0,\,5.0,\,6.0\}$ in random directions. For $q$, we place four mode centers at radii $\{2.0,\,3.5,\,4.0,\,5.5\}$ in independent random directions. Samples are generated by choosing a mode uniformly within each mixture and adding isotropic Gaussian noise with standard deviation $0.5$. This produces broader and more heterogeneous norms, together with a mismatched mode layout between $p$ and $q$, and therefore serves as a stress test beyond the shell-concentration setting.

In both datasets, we set the Laplace bandwidth adaptively as
\[
\tau = \bar\tau\cdot\overline{\|\rvx\|_2},
\qquad
\bar\tau=0.3,
\]
where $\overline{\|\rvx\|_2}$ denotes the mean norm of the query batch. Since $\overline{\|\rvx\|_2}$ grows like $\sqrt{R^2+\sigma^2 D}\propto \sqrt{D}$ for large $D$, this choice corresponds to $a\approx\tfrac12$ in the dimension-aware scaling $\tau=\bar\tau D^a$ used in the theory.

For each distribution, we draw a finite i.i.d.\ reference set of size $N=3{,}000$:
\[
\{\mathbf{y}^{(p)}_j\}_{j=1}^{N}\sim p,
\qquad
\{\mathbf{y}^{(q)}_j\}_{j=1}^{N}\sim q.
\]
All expectations over $\mathbf{x}\sim q$ are approximated empirically using query batches sampled from $q$.

\paragraph{Nonparametric Field Estimators.}
Fix a query point $\mathbf{x}\in\mathbb{R}^D$. We use the Laplace kernel
\[
k_\tau(\rvx,\rvy)=\exp\!\left(-\frac{\|\rvx-\rvy\|_2}{\tau}\right),
\]
and normalize the kernel values over each reference set. For $p$, the weights are
\[
w^{(p)}_j(\mathbf{x})
=
\frac{k_\tau(\mathbf{x},\mathbf{y}^{(p)}_j)}
{\sum_{\ell=1}^{N} k_\tau(\mathbf{x},\mathbf{y}^{(p)}_\ell)},
\qquad j=1,\dots,N,
\]
and $w^{(q)}_j(\mathbf{x})$ is defined analogously using $\{\mathbf{y}^{(q)}_j\}_{j=1}^{N}$.

These weights define the nonparametric mean-shift fields
\[
\widehat{\mathbf{V}}_{p,k}(\mathbf{x})
=
\sum_{j=1}^{N} w^{(p)}_j(\mathbf{x})\big(\mathbf{y}^{(p)}_j-\mathbf{x}\big),
\qquad
\widehat{\mathbf{V}}_{q,k}(\mathbf{x})
=
\sum_{j=1}^{N} w^{(q)}_j(\mathbf{x})\big(\mathbf{y}^{(q)}_j-\mathbf{x}\big).
\]
They estimate the kernel-reweighted displacement from $\mathbf{x}$ toward nearby samples under $p$ and $q$.

The corresponding kernel-score estimators are
\[
\widehat{\mathbf{s}}_{p,k}(\mathbf{x})
=
\nabla_{\mathbf{x}}\log\Big(\sum_{j=1}^{N} k_\tau(\mathbf{x},\mathbf{y}^{(p)}_j)\Big)
=
\frac{1}{\tau}\sum_{j=1}^{N} w^{(p)}_j(\mathbf{x})\,
\frac{\mathbf{y}^{(p)}_j-\mathbf{x}}{\|\mathbf{y}^{(p)}_j-\mathbf{x}\|_2},
\]
and analogously for $\widehat{\mathbf{s}}_{q,k}(\mathbf{x})$.

Finally, we form the drifting and score discrepancies
\[
\Delta_{p,q}(\mathbf{x})
=
\widehat{\mathbf{V}}_{p,k}(\mathbf{x})-\widehat{\mathbf{V}}_{q,k}(\mathbf{x}),
\qquad
\Delta \mathbf{s}_{p,q}(\mathbf{x})
=
\widehat{\mathbf{s}}_{p,k}(\mathbf{x})-\widehat{\mathbf{s}}_{q,k}(\mathbf{x}).
\]
All quantities are finite-sample Monte Carlo estimators: they depend only on sampled query/reference points and kernel evaluations, with no closed-form access to $p$ or $q$.

\subsection{Additional Diagnostics for the Laplace Mechanism on Synthetic Data}\label{app:toy-diagnostics}
Beyond the primary alignment curves in \Cref{fig:emp-D-decay}, we report diagnostics tied directly to the pre-conditioned-score decomposition in \Cref{prop:radial_precond_decomp}. These diagnostics test not only directional alignment, but also the predicted scaling and the vanishing-residual mechanism behind it.

\paragraph{pre-conditioners, Residuals, and Predicted Scale.}
For each query point $\mathbf{x}$, we define the kernel-weighted mean distances
\[
\alpha_p(\mathbf{x})
:=
\sum_j w^{(p)}_j(\mathbf{x})\,\|\mathbf{y}^{(p)}_j-\mathbf{x}\|_2,
\qquad
\alpha_q(\mathbf{x})
:=
\sum_j w^{(q)}_j(\mathbf{x})\,\|\mathbf{y}^{(q)}_j-\mathbf{x}\|_2.
\]
Their query-averaged versions are
\[
\bar\alpha_p := \mathbb E_{\mathbf{x}\sim q}\big[\alpha_p(\mathbf{x})\big],
\qquad
\bar\alpha_q := \mathbb E_{\mathbf{x}\sim q}\big[\alpha_q(\mathbf{x})\big].
\]
In high dimension, the theory predicts that these pre-conditioners become nearly constant and nearly equal. We therefore form the empirical scale
\begin{align}\label{eq:c-theory}
\rho := \tfrac12\big(\bar\alpha_p+\bar\alpha_q\big),
\qquad
C_{\text{theory}} := \rho\,\tau.
\end{align}
This is the theory-predicted scale relating the drifting discrepancy to the score discrepancy.

We also compute the covariance residuals
\[
\boldsymbol{\delta}_p(\mathbf{x})
:=
\mathbf{V}_{p,k}(\mathbf{x})-\big(\alpha_p(\mathbf{x})\,\tau\big)\,\mathbf{s}_{p,k}(\mathbf{x}),
\qquad
\boldsymbol{\delta}_q(\mathbf{x})
:=
\mathbf{V}_{q,k}(\mathbf{x})-\big(\alpha_q(\mathbf{x})\,\tau\big)\,\mathbf{s}_{q,k}(\mathbf{x}),
\]
and their gap $\boldsymbol{\delta}_p(\mathbf{x})-\boldsymbol{\delta}_q(\mathbf{x})$.

The proof of the high-dimensional theorem suggests two coupled effects: the pre-conditioners $\alpha_p(\mathbf{x})$ and $\alpha_q(\mathbf{x})$ should concentrate to the same scale, and the residual gap should vanish.

To test the predicted magnitude calibration directly, we also compute an oracle least-squares scale
\[
C^*
:=
\arg\min_{C\in\mathbb{R}}
\mathbb E_{\mathbf{x}\sim q}\Big[\big\|\Delta_{p,q}(\mathbf{x})-C\,\Delta \mathbf{s}_{p,q}(\mathbf{x})\big\|_2^2\Big]
=
\frac{\mathbb E_{\mathbf{x}\sim q}\big[\langle \Delta_{p,q}(\mathbf{x}),\,\Delta \mathbf{s}_{p,q}(\mathbf{x})\rangle\big]}
{\mathbb E_{\mathbf{x}\sim q}\big[\|\Delta \mathbf{s}_{p,q}(\mathbf{x})\|_2^2\big]}.
\]
Unlike the cosine similarity, which measures only directional agreement, comparing $C^*$ with $C_{\text{theory}}$ tests whether the magnitude scaling predicted by the theory is correct.

\begin{figure}[t]
    \centering
    \includegraphics[width=0.95\linewidth]{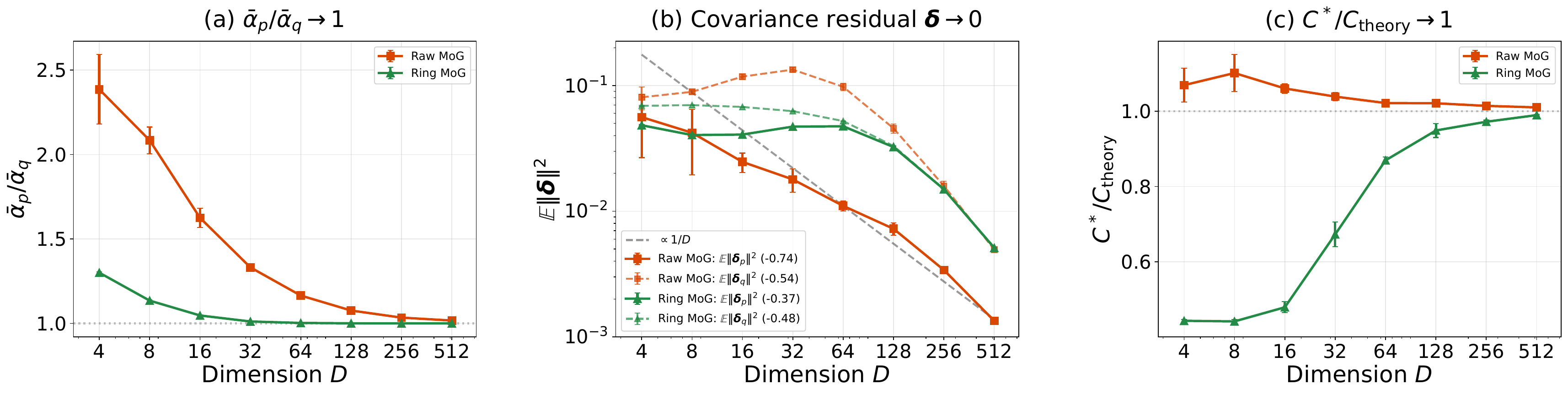}
\caption{\footnotesize{\textbf{Empirical diagnostics for the Laplace-kernel mechanism.}
(a) The kernel-reweighted pre-conditioners concentrate and become indistinguishable, with $\bar\alpha_p/\bar\alpha_q\to 1$.
(b) The residual-gap energy $\mathbb E_{\rvx\sim q}\|\boldsymbol{\delta}_p(\mathbf{x})-\boldsymbol{\delta}_q(\mathbf{x})\|_2^2$ decays with dimension, indicating a vanishing covariance residual.
(c) The theory-predicted scale $C_{\mathrm{theory}}=\rho\tau$ matches the oracle least-squares scale $C^*$, with $C^*/C_{\mathrm{theory}}\to 1$.
Together, these diagnostics support the mechanism behind \Cref{prop:radial_precond_decomp,thm:field_alignment_radial_full_xy} and show that the Laplace mean-shift discrepancy is not only directionally aligned with the score discrepancy, but also correctly calibrated in magnitude.}}
    \label{fig:emp-precon-scale}
\end{figure}

\paragraph{Empirical Findings.}
\Cref{fig:emp-precon-scale} supports all three predictions. First, panel (a) shows that the averaged pre-conditioners become indistinguishable as $D$ grows, with $\bar\alpha_p/\bar\alpha_q\to 1$. Second, panel (b) shows that the residual-gap energy $\mathbb E_{\mathbf{x}\sim q}\|\boldsymbol{\delta}_p(\mathbf{x})-\boldsymbol{\delta}_q(\mathbf{x})\|_2^2$ decays with dimension, indicating that the covariance residual becomes negligible in the large-$D$ regime. Third, panel (c) shows that the oracle scale $C^*$ agrees increasingly well with the theory-predicted scale $C_{\text{theory}}$, with $C^*/C_{\text{theory}}\to 1$.

Together, these diagnostics sharpen the message of \Cref{subsec:emp-toy}. The high-dimensional picture is not merely that the two discrepancy fields become directionally aligned; rather, the pre-conditioned-score decomposition also predicts the correct scale, and the residual term becomes progressively irrelevant. In this sense, the Laplace mean-shift field behaves increasingly like a globally rescaled score field.

Finally, although \Cref{thm:field_alignment_radial_full_xy,thm:grad_alignment_largeD} are proved under concentration-type assumptions as sufficient conditions, the Raw MoG stress test still exhibits the same qualitative trends. This suggests that the alignment mechanism may extend beyond the idealized shell-concentration setting covered by the current theory.

\paragraph{Training Experiments on 2D Synthetic Data.}
\begin{figure}[th!]
    \centering
    \includegraphics[width=\linewidth]{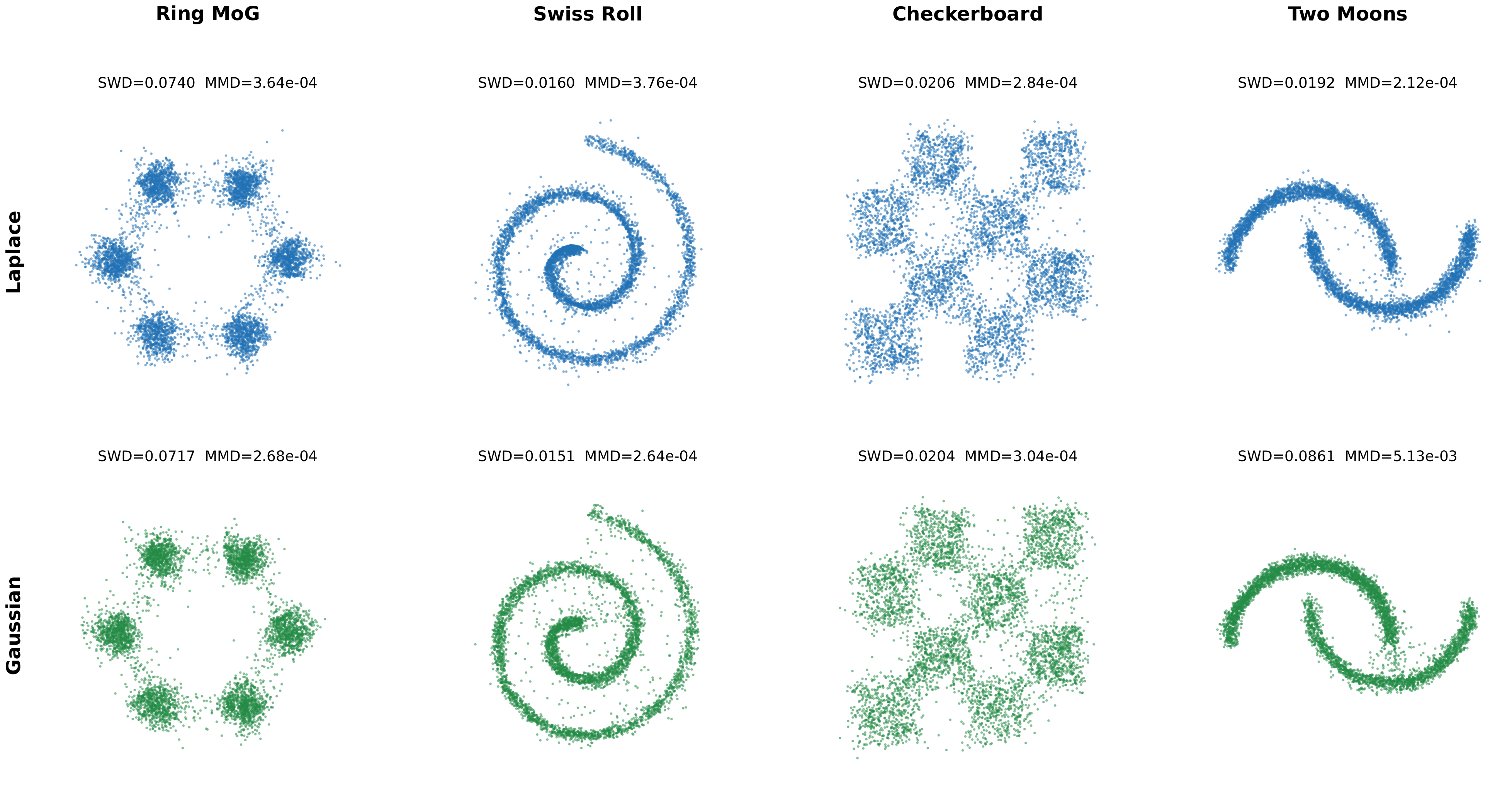}
    \caption{\footnotesize{\textbf{Illustration of 2D generation across different synthetic datasets.} We compare the generation quality of drifting models using Laplace and Gaussian kernels, and evaluate them using Sliced Wasserstein Distance (SWD) and MMD. The two kernels achieve nearly identical performance on both metrics across the four datasets. This suggests that, even in low dimension, the pre-conditioning and covariance-residual terms in \Cref{prop:radial_precond_decomp} have little practical effect on generation quality.
    }}
    \label{fig:toy-all}
\end{figure}
We train drifting models with Laplace kernels (mean-shift) and Gaussian kernels (score-mismatch) to test whether the performance gap between the two is negligible on synthetic datasets.

We follow the setup in the drifting model's demo implementation\footnote{\url{https://lambertae.github.io/projects/drifting/}} and train a one-step generator on four 2D targets (Ring MoG, Swiss Roll, Checkerboard, Two Moons).
The generator is a 4-layer MLP mapping $\rvz\sim\mathcal N(\mathbf 0,\mathbf I_{32})$ to $\mathbb R^2$, trained with batch size $2048$ for both data and model samples.
At each step, we form a transported target $\rvx+\widehat{\mathbf V}_{p,q}(\rvx)$ using finite-sample mean shift with either a Laplace kernel
$k_\tau(\rvx,\rvy)=\exp(-\|\rvx-\rvy\|_2/\tau)$ or a Gaussian kernel
$k_\sigma(\rvx,\rvy)=\exp(-\|\rvx-\rvy\|_2^2/(2\sigma^2))$ (the Gaussian case coincides with a score-mismatch transport direction).
We use kernel scales $(\tau,\sigma)=(0.30,0.30)$ for Ring MoG and $(\tau,\sigma)=(0.05,0.05)$ for the other datasets.
We evaluate sample quality by sliced Wasserstein distance (SWD) with 200 random projections and RBF-MMD on $5,000$ generated and $5,000$ real samples.
\Cref{fig:toy-all} reports the generations together with SWD and MMD.
Across all four targets, Laplace and Gaussian kernels achieve nearly identical performance, suggesting that, even in these low-dimensional settings, the pre-conditioning and covariance-residual corrections in \Cref{prop:radial_precond_decomp} have limited impact on sample quality.

\subsection{Experimental Setup and Results on Realistic Image Datasets}\label{app:real-setup}

\paragraph{CIFAR-10.}
We adapt the official JAX drifting codebase\footnote{\url{https://github.com/lambertae/drifting}} from its original TPU-oriented release to GPU execution, while keeping the training logic, model family, memory-bank mechanism, loss construction, and evaluation protocol as close as possible to the original implementation. All CIFAR-10 experiments are conducted in pixel space at resolution $32\times 32$, using the CIFAR-10 training split for optimization and the standard repository evaluation pipeline for FID on $50$K generated samples, with training run on up to 8 A100 80GB GPUs.

Our CIFAR-10 generator is a DiT-style conditional transformer~\cite{peebles2023scalable} with input size $32$, patch size $2$, hidden size $512$, depth $8$, $8$ attention heads, MLP ratio $4.0$, conditioning dimension $384$, $8$ class tokens, $32$ noise classes, and $16$ noise coordinates. We enable QK normalization, SwiGLU, RoPE, and RMSNorm, following the implementation defaults. Training uses AdamW with learning rate $2\times 10^{-4}$, linear warmup for $5$K steps followed by a constant schedule, $(\beta_1,\beta_2)=(0.9,0.95)$, weight decay $0.01$, gradient clipping at $2.0$, and EMA decay $0.999$. We train for $100$K iterations, save checkpoints every $5$K steps, and evaluate every $5$K steps. The data-loader batch size is $256$ for both training and evaluation.

For the drifting loss, we follow the same fixed-point template as the official implementation. We use $32$ positives and $8$ negatives per sample, a positive memory bank of size $128$, a negative memory bank of size $1000$, and $64$ generated samples per label with classifier-free guidance values drawn from $[1,4]$. To compare the theoretically exact and approximate cases under matched conditions, we test both Gaussian and Laplace kernels by changing only the kernel choice and its bandwidth list while keeping the rest of the pipeline fixed. For Laplace we use the standard repository choice $\bar\tau\in\{0.2,0.05,0.02\}$, with kernel weights of the form $\exp(-r/\bar\tau)$. For Gaussian we use $\bar\tau\in\{0.5,0.25,0.125\}$, with weights $\exp(-(r/\bar\tau)^2)$.Since the implementation first normalizes pairwise distances by an empirical feature-space scale, these $\bar\tau$ values are dimensionless. In this sense, they serve as the practical counterpart of the theoretical scaling $\tau=\bar\tau D^a$ for some $\bar\tau> 0 $ and $a \ge 0 $, with the dimension dependence absorbed by the normalization. The larger Gaussian list is chosen so that, after squaring in the exponent, the resulting affinities operate on a comparable effective distance scale to the Laplace case rather than becoming overly concentrated.

Since drifting on natural images relies on a frozen pre-trained feature space, we evaluate two feature backbones. For ConvNeXt-based runs, we use a frozen ConvNeXt-V2 backbone, resize CIFAR-10 images from $32\times 32$ to $224\times 224$ by bilinear interpolation, apply the standard ImageNet normalization used by the codebase, and extract the same multi-scale activation statistics as in the official pipeline; the theorem-facing post-hoc diagnostics use the global ConvNeXt feature \texttt{convnext\_global}. For MAE-based runs, we replace ConvNeXt with the frozen pre-trained checkpoint \texttt{hf://mae\_pixel\_640}, namely the ResNet-style MAE feature encoder released with the same codebase. Because this encoder is pre-trained on $256\times 256$ images, we resize CIFAR-10 inputs to $256\times 256$ before feature extraction and otherwise keep the drifting pipeline unchanged. The corresponding global MAE feature \texttt{mae\_global} is used for post-hoc analysis. No feature backbone is finetuned on CIFAR-10. Unless otherwise stated, all reported CIFAR-10 results keep these settings fixed and vary only the kernel. For theorem-facing post-hoc diagnostics, we use the same estimator computations as in \Cref{subsec:emp-toy,app:toy-diagnostics}.

\begin{figure}[h!]
    \centering
    \subfloat[Laplace-ConvNeXt]{%
        \includegraphics[width=0.45\linewidth]{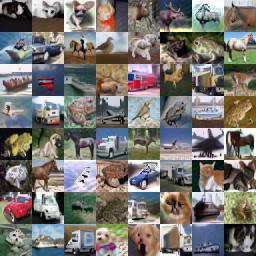}%
    }\hspace{0.3cm}
    \subfloat[Gaussian-ConvNeXt]{%
        \includegraphics[width=0.45\linewidth]{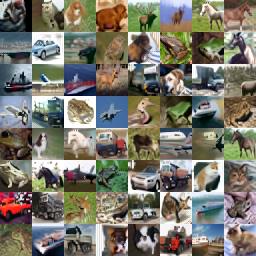}%
    }\hspace{0.3cm}
    \subfloat[Laplace-MAE]{%
        \includegraphics[width=0.45\linewidth]{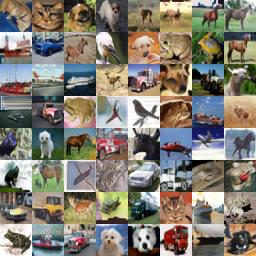}%
    }\hspace{0.3cm}
    \subfloat[Gaussian-MAE]{%
        \includegraphics[width=0.45\linewidth]{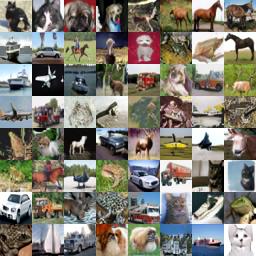}%
    }
    \caption{\footnotesize{\textbf{Comparison of generation on CIFAR-10 under different feature embeddings and kernels.} Single-step unconditional generation on CIFAR-10 at $32\times 32$ resolution using Laplace and Gaussian kernels, with the drifting loss evaluated in either ConvNeXt or MAE feature space. The first two panels show the ConvNeXt-based models, where Gaussian achieves FID 8.38 and Laplace achieves FID 8.66. The last two panels show the MAE-based models, where Gaussian achieves FID 4.72 and Laplace achieves FID 4.79. Across both feature spaces, the two kernels achieve near-parity under an otherwise identical training pipeline, suggesting that the Laplace-specific pre-conditioning and residual terms have limited effect on end-to-end sample quality in these settings. This observation is also consistent with concurrent findings on CelebA-HQ reported by~\cite{li2026long}.}}
    \label{fig:cifar10}
\end{figure}

\paragraph{ImageNet 256$\times$256.}
Unless otherwise stated, the overall code path, fixed-point evaluation template, and theorem-facing diagnostics are the same as in the CIFAR-10 setup. For high-resolution validation, however, we follow the official latent ImageNet-256 release rather than retraining from scratch. Specifically, we use the released Drift-L checkpoint \texttt{hf://latent\_L\_sota} together with the frozen feature encoder \texttt{hf://mae\_latent\_640}. In the official \texttt{latent\_sota\_L} configuration, training is carried out in cached VAE latent space at model resolution $32\times 32$ with $4$ channels, while the drifting loss is defined in the frozen latent MAE feature space. Relative to CIFAR-10, the main changes are a larger DiT backbone (hidden size $1024$, depth $24$, and $16$ attention heads), a longer schedule ($200$K steps), and a stronger fixed-point configuration with $64$ positives and $32$ negatives per sample; the optimizer hyperparameters and Laplace bandwidth list remain the same as above. The released checkpoint is reported at FID $1.53$ on ImageNet-256 using $50$K samples at CFG $1.0$.
\newpage

\section{Proof of the pre-conditioning Score Decomposition}\label{app:proof-precond-score}

\subsection{Proof to \Cref{thm:gaussian_meanshift_score_matching}}

\begin{proof}[Proof to \Cref{thm:gaussian_meanshift_score_matching}]
We first prove the identity that relates the mean-drift and score function. Let $k_\tau$ be a Gaussian kernel. By definition,
\[
p_\tau(\rvx)=(p*k_\tau)(\rvx)=\mathbb E_{\rvy\sim p}\big[k_\tau(\rvx, \rvy)\big].
\]
Differentiating under the expectation and using
$\nabla_{\rvx}k_\tau(\rvx, \rvy)=-(\rvx-\rvy)k_\tau(\rvx, \rvy)/\tau^2$ yields
\[
\nabla_{\rvx}p_\tau(\rvx)
=
\frac{1}{\tau^2}\mathbb E_{\rvy\sim p}\big[k_\tau(\rvx, \rvy)(\rvy-\rvx)\big].
\]
Dividing by $p_\tau(\rvx)$ gives
\[
\nabla_{\rvx}\log p_\tau(\rvx)
=
\frac{\mathbb E_{\rvy\sim p}\big[k_\tau(\rvx, \rvy)(\rvy-\rvx)\big]}
{\tau^2\,\mathbb E_{\rvy\sim p}\big[k_\tau(\rvx, \rvy)\big]}
=
\frac{\mathbf V_{p,k_\tau}(\rvx)}{\tau^2},
\]
where the last equality uses the definition of $\mathbf V_{p,k_\tau}$ in the mean-shift operator paragraph.
The same argument applies to $q$.
Substituting 
\begin{equation*}
\mathbf V_{\uppi,k_\tau}(\rvx)=\tau^2 \mathbf s_{\uppi,\tau}(\rvx)
\qquad\text{for } \uppi\in\{p,q\}.
\end{equation*}
into $\Delta_{p,q}(\rvx)=\eta(\mathbf V_{p,k_\tau}(\rvx)-\mathbf V_{q,k_\tau}(\rvx))$
gives the discrepancy formula defined by the mean-shift. Finally, \Cref{eq:drift_value_equals_score} follows from \Cref{eq:value_equiv}.
\end{proof}

\subsection{Proof to \Cref{prop:radial_precond_decomp}}
\begin{proof}[Proof to \Cref{prop:radial_precond_decomp}]
    Fix $\rvx$ and write $b:=b_\tau(\|\rvx-\rvy\|_2)$. Under $\rvy\sim \uppi_\tau(\cdot| \rvx)$, define
\[
A(\rvy):=b^{-1},\qquad B(\rvy):=b(\rvy-\rvx).
\]
Then $A(\rvy)B(\rvy)=\rvy-\rvx$, and by the definition of $\mathbf V_{\uppi,k_\tau}$,
\[
\mathbf V_{\uppi,k_\tau}(\rvx)=\mathbb E_{\rvy\sim \uppi_\tau(\cdot| \rvx)}[A(\rvy)B(\rvy)].
\]
Applying the covariance identity $\mathbb E[AB]=\mathbb E[A]\mathbb E[B]+\operatorname{Cov}(A,B)$ under $\rvy\sim \uppi_\tau(\cdot| \rvx)$ yields
\[
\mathbf V_{\uppi,k_\tau}(\rvx)
=
\alpha_{\uppi}(\rvx)\,\mathbb E_{\rvy\sim \uppi_\tau(\cdot| \rvx)}[B(\rvy)]
+
\boldsymbol{\delta}_{\uppi}(\rvx).
\]
Finally, \Cref{eq:radial_weighted_score_ratio} is exactly
\[
\mathbb E_{\rvy\sim \uppi_\tau(\cdot| \rvx)}\big[b_\tau(\|\rvx-\rvy\|_2)(\rvy-\rvx)\big]
=\tau^2\,\mathbf s_{\uppi,k_\tau}(\rvx),
\]
so $\mathbb E[B(\rvy)]=\tau^2\mathbf s_{\uppi,k_\tau}(\rvx)$ and the claim follows.
\end{proof}


\section{Proofs of the High-Dimensional Theorems}\label{app:proof-large-D}
In this section, we prove \Cref{thm:D_consistency}, which establishes a proxy relationship between score matching and the drifting model at their optimal distributions in the high-dimensional regime. We state a rigorous version of this relationship below.

\begingroup
\let\thetheoremorig\thetheorem
\addtocounter{theorem}{-4} 
\renewcommand{\thetheorem}{\thetheoremorig\ensuremath{'}}
\begin{theorem}[High-Dimensional Agreement between Mean-Shift and Score Matching]
\label{thm:D_consistency_prime}
Assume \Cref{ass:shell_R0}--\Cref{ass:bounded_norm}, \Cref{ass:ac}, and \Cref{ass:realizable_fixed}.
Fix $\bar\tau>0$ and $a\ge 0$, and set $\tau=\bar\tau D^{a}$.
For each $D$, let
\[
\rvf^{\star}_D \in \arg\min_{\rvf\in\mathcal F}\mathcal{L}_{\mathrm{drift}}(\rvf).
\]
Then
\[
\mathcal{D}_{\mathrm{rF}}\!\bigl(p\| q_{\rvf^\star_D}\bigr)
=
\mathcal{O}\!\left(D^{-(1+2a)}\right).
\]
Moreover, by realizability (\Cref{ass:realizable_fixed}) there exists $\rvg_{\mathrm{data},D}\in\mathcal G$ such that
$q_{\rvg_{\mathrm{data},D}}=p$, and hence
\[
\mathcal{D}_{\mathrm{rF}}\!\bigl(q_{\rvg_{\mathrm{data},D}}\| q_{\rvf^\star_D}\bigr)
=
\mathcal{D}_{\mathrm{rF}}\!\bigl(p\| q_{\rvf^\star_D}\bigr)
=
\mathcal{O}\!\left(D^{-(1+2a)}\right).
\]
The hidden constant in $\mathcal{O}(\cdot)$ is independent of $D$ and of the learnable parameters.
\end{theorem}
\endgroup

\subsection{Proof Roadmap and Technical Assumptions}

\paragraph{Proof Roadmap.}
The proof has two conceptual pieces. First, under the shell and inner-product moment conditions, independent samples from any two distributions in our model/data family have pairwise distance tightly concentrated around the common scale $\rho=\sqrt2R_0$ at rate $1/\sqrt D$ (Step~1). Second, drifting does not sample neighbors uniformly: it reweights neighbors by the Laplace kernel. With bounded feature norms, this kernel reweighting cannot distort distances too much, so the \emph{kernel-reweighted} neighbor radius $r=\|\rvy-\rvx\|_2$ still concentrates around $\rho$ with $\mathbb E(r-\rho)^2=\mathcal O(1/D)$, uniformly in $D$ and without any restriction on $\bar\tau$ (Step~2). These two concentration facts are then plugged into the pre-conditioned-score decomposition (Step~3), which writes mean shift as a scaled kernel score plus a covariance residual; concentration makes both the pre-conditioner fluctuation and the residual $\mathcal O(1/\sqrt D)$ in $L^2$, yielding the field alignment $\Delta_{p,q_\rvf}(\rvx)\approx (\tau\rho)\,\Delta\rvs_{p, q_\rvf}(\rvx)$ with mean-square error $\mathcal O(1/D)$. Finally, at the distribution-level optimum drifting enforces $\mathbf V_{\rvf^\star}\equiv 0$, so alignment implies $\mathbb E_{q_{\rvf^\star}}\|\Delta\rvs_{\rvf^\star}\|_2^2=\mathcal O((\tau\rho)^{-2}D^{-1})=\mathcal O(D^{-(1+2a)})$; score matching identifies $q_{\rvg^\star}=p$, giving the stated Fisher-divergence decay (Step~4).

\paragraph{Technical Assumptions.} The assumptions below assert that sample norms concentrate around a common scale, and that pairwise inner products between independent samples are small on average.

\begin{assumption}[Shell around a Common Scale $R_0$]
\label{ass:shell_R0}
There exists $\sigma\ge 0$ independent of $D$ such that for every $\mu$ in the above family, if $\rvx\sim\mu$ then
\[
\mathbb{E}\big(\|\mathbf{x}\|_2^2-R_0^2\big)^2 \le \frac{\sigma^2 R_0^4}{D}.
\]
\end{assumption}

\begin{assumption}[Mixed Inner-Product Second Moment]
\label{ass:ip_L2}
There exists $\kappa\ge 0$ independent of $D$ such that for any $\mu,\nu$ in the above family,
if $\rvx\sim\mu$ and $\rvy\sim\nu$ are independent then
\[
\mathbb{E}\langle \mathbf{x},\mathbf{y}\rangle^2 \le \frac{\kappa R_0^4}{D}.
\]
\end{assumption}

\begin{assumption}[Fourth-Moment Controls]
\label{ass:fourth}
There exist constants $C_{\mathrm{norm},4},C_{\mathrm{ip},4}<\infty$ independent of $D$ such that:
\begin{enumerate}
\item For every $\mu$ in the above family, if $\rvx\sim\mu$ then
\[
\mathbb{E}\big(\|\mathbf{x}\|_2^2-R_0^2\big)^4 \le \frac{C_{\mathrm{norm},4}R_0^8}{D^2}.
\]
\item For any $\mu,\nu$ in the above family, if $\rvx\sim\mu$ and $\rvy\sim\nu$ are independent then
\[
\mathbb{E}\langle \mathbf{x},\mathbf{y}\rangle^4 \le \frac{C_{\mathrm{ip},4}R_0^8}{D^2}.
\]
\end{enumerate}
\end{assumption}

\begin{assumption}[Bounded (Feature) Norm]
\label{ass:bounded_norm}
There exists $B<\infty$ independent of $D$ such that for every $\mu$ in the above family, if $\rvx\sim\mu$ then
\[
\|\rvx\|_2 \le B \qquad \text{almost surely}.
\]
\end{assumption}

Drifting-model pipelines that rely on pre-trained feature maps typically enforce explicit norm control, for instance via $\ell_2$-normalization, LayerNorm followed by clipping, or direct norm clipping. Under such preprocessing, \Cref{ass:bounded_norm} holds naturally with a known constant $B$ (often $B=1$ for $\ell_2$-normalized embeddings). Thus, this assumption is realistic in practice.

\subsection{Auxiliary Tools and the Main Proof}
\paragraph{Step 1: Mixed Distances Concentrate.}

\begin{lemma}[Mixed Distance Concentrates in $L^2$ around $\rho=\sqrt2 R_0$]
\label{lem:mixed_dist_L2_radial_full_xy}
Assume \Cref{ass:shell_R0,ass:ip_L2}.
Draw $\rvx\sim\mu$ and $\rvy\sim\nu$ independently (any $\mu,\nu$ from the family), and set $S:=\|\rvx-\rvy\|_2$.
Then
\[
\mathbb{E}(S-\rho)^2 \le \frac{C_{\mathrm{dist}}}{D},
\qquad
C_{\mathrm{dist}}:=3\big(2\sigma^2+4\kappa\big)R_0^2.
\]
\end{lemma}

\begin{proof}
Write
\[
S^2=\|\rvx\|_2^2+\|\rvy\|_2^2-2\langle \rvx,\rvy\rangle,
\qquad
\rho^2=2R_0^2.
\]
For $b>0$, $(\sqrt a-\sqrt b)^2\le (a-b)^2/b$, hence
\[
(S-\rho)^2\le \frac{(S^2-\rho^2)^2}{\rho^2}.
\]
Now
\[
S^2-\rho^2
=
(\|\rvx\|_2^2-R_0^2)+(\|\rvy\|_2^2-R_0^2)-2\langle \rvx,\rvy\rangle.
\]
Using $(a+b+c)^2\le 3(a^2+b^2+c^2)$,
\[
(S^2-\rho^2)^2
\le
3\Big((\|\rvx\|_2^2-R_0^2)^2+(\|\rvy\|_2^2-R_0^2)^2+4\langle \rvx,\rvy\rangle^2\Big).
\]
Taking expectations and applying \Cref{ass:shell_R0,ass:ip_L2} gives
\[
\mathbb{E}(S^2-\rho^2)^2
\le
3\Big(\tfrac{\sigma^2R_0^4}{D}+\tfrac{\sigma^2R_0^4}{D}+4\tfrac{\kappa R_0^4}{D}\Big)
=
\frac{3(2\sigma^2+4\kappa)R_0^4}{D}.
\]
Divide by $\rho^2=2R_0^2$ and absorb the factor $1/2$ into the constant.
\end{proof}

\begin{lemma}[Fourth Moment Bound for Mixed Distances]
\label{lem:mixed_dist_L4_radial_full_xy}
Assume \Cref{ass:fourth}. Draw $\rvx\sim\mu$ and $\rvy\sim\nu$ independently,
set $S:=\|\rvx-\rvy\|_2$ and $\rho=\sqrt2R_0$.
Then there exists $C_{\mathrm{dist},4}$ independent of $D$ such that
\[
\mathbb{E}(S-\rho)^4 \le \frac{C_{\mathrm{dist},4}}{D^2}.
\]
\end{lemma}

\begin{proof}
Let $A:=\|\rvx\|_2^2-R_0^2$, $B:=\|\rvy\|_2^2-R_0^2$, $W:=\langle \rvx,\rvy\rangle$.
Then $S^2-\rho^2=A+B-2W$.
Using $(a+b+c)^4\le 27(a^4+b^4+c^4)$ with $c=-2W$,
\[
(S^2-\rho^2)^4\le 27\big(A^4+B^4+16W^4\big).
\]
Taking expectations and applying \Cref{ass:fourth} yields
\[
\mathbb E(S^2-\rho^2)^4
\le
27\Big(\tfrac{C_{\mathrm{norm},4}R_0^8}{D^2}+\tfrac{C_{\mathrm{norm},4}R_0^8}{D^2}+16\tfrac{C_{\mathrm{ip},4}R_0^8}{D^2}\Big)
=
\frac{27(2C_{\mathrm{norm},4}+16C_{\mathrm{ip},4})R_0^8}{D^2}.
\]
Finally, for $b>0$, $(\sqrt a-\sqrt b)^4\le (a-b)^4/b^2$ with $a=S^2$ and $b=\rho^2=2R_0^2$ gives
\[
(S-\rho)^4\le \frac{(S^2-\rho^2)^4}{\rho^4}=\frac{(S^2-\rho^2)^4}{4R_0^4}.
\]
Combine the bounds.
\end{proof}

\paragraph{Step 2: Kernel-Reweighting Preserves the $\mathcal O(1/D)$ Scale.}

In the remainder we specialize to the Laplace kernel
\[
k_\tau(\rvx,\rvy)=\exp\!\Big(-\frac{\|\rvx-\rvy\|_2}{\tau}\Big),
\qquad\text{so that}\qquad
b_\tau(r)=\frac{\tau}{r}.
\]

\begin{lemma}[Kernel-reweighted radii still concentrate (no constraint on $\bar\tau$)]
\label{lem:reweight_L2_radial_full_xy}
Assume \Cref{ass:bounded_norm} and \Cref{lem:mixed_dist_L2_radial_full_xy}.
Fix any $\mu,\nu$ from the family and draw $\rvx\sim \nu$.
Conditionally on $\rvx$, draw $\rvy$ from the kernel-reweighted conditional density
\[
\mu_\tau(\rvy| \rvx)
:=
\frac{k_\tau(\rvx,\rvy)\,\mu(\rvy)}{\mathbb E_{\rvy'\sim \mu}[k_\tau(\rvx,\rvy')]}.
\]
Let $r:=\|\rvy-\rvx\|_2$ and $\rho=\sqrt2R_0$. Then for every $\tau>0$,
\[
\mathbb E(r-\rho)^2 \le \exp\!\Big(\frac{2B}{\tau}\Big)\cdot \frac{C_{\mathrm{dist}}}{D},
\]
where $B$ is from \Cref{ass:bounded_norm} and $C_{\mathrm{dist}}$ is from \Cref{lem:mixed_dist_L2_radial_full_xy}.
In particular, if $\tau=\bar\tau D^a$ with $a\ge 0$, then $\tau\ge \bar\tau$ for all $D\ge 1$ and hence
\[
\mathbb E(r-\rho)^2 \le \exp\!\Big(\frac{2B}{\bar\tau}\Big)\cdot \frac{C_{\mathrm{dist}}}{D},
\]
so the hidden constant is independent of $D$ with \emph{no lower bound requirement} on $\bar\tau>0$.
\end{lemma}

\begin{proof}
Condition on $\rvx$. By definition of $\mu_\tau(\cdot| \rvx)$,
\[
\mathbb E[(r-\rho)^2| \rvx]
=
\frac{
\mathbb E_{\rvy'\sim\mu}\Big[(\|\rvy'-\rvx\|_2-\rho)^2 e^{-\|\rvy'-\rvx\|_2/\tau}\,\Big|\,\rvx\Big]
}{
\mathbb E_{\rvy'\sim\mu}\Big[e^{-\|\rvy'-\rvx\|_2/\tau}\,\Big|\,\rvx\Big]
}.
\]
Since $e^{-u/\tau}\le 1$, the numerator is at most
$\mathbb E[(\|\rvy'-\rvx\|_2-\rho)^2| \rvx]$.
For the denominator, Jensen gives
\[
\mathbb E\Big[e^{-\|\rvy'-\rvx\|_2/\tau}\,\Big|\,\rvx\Big]
\ge
\exp\!\Big(-\frac{1}{\tau}\mathbb E[\|\rvy'-\rvx\|_2| \rvx]\Big).
\]
Therefore
\[
\mathbb E[(r-\rho)^2| \rvx]
\le
\mathbb E\big[(\|\rvy'-\rvx\|_2-\rho)^2| \rvx\big]\,
\exp\!\Big(\frac{1}{\tau}\mathbb E[\|\rvy'-\rvx\|_2| \rvx]\Big).
\]
Under \Cref{ass:bounded_norm}, $\|\rvx\|_2\le B$ and $\|\rvy'\|_2\le B$ almost surely, so
$\|\rvy'-\rvx\|_2\le 2B$ almost surely and hence $\mathbb E[\|\rvy'-\rvx\|_2| \rvx]\le 2B$.
Thus
\[
\mathbb E[(r-\rho)^2| \rvx]
\le
\exp\!\Big(\frac{2B}{\tau}\Big)\,
\mathbb E\big[(\|\rvy'-\rvx\|_2-\rho)^2| \rvx\big].
\]
Taking expectation over $\rvx\sim\nu$ and using independence of $\rvx\sim\nu$ and $\rvy'\sim\mu$ yields
\[
\mathbb E(r-\rho)^2
\le
\exp\!\Big(\frac{2B}{\tau}\Big)\,
\mathbb E\big[(\|\rvy'-\rvx\|_2-\rho)^2\big].
\]
Finally, \Cref{lem:mixed_dist_L2_radial_full_xy} bounds the right-hand side by
$\exp(2B/\tau)\cdot C_{\mathrm{dist}}/D$.
\end{proof}

\paragraph{Step 3: Field Alignment.}

For Laplace, $\nabla_{\rvx}\log k_\tau(\rvx,\rvy)=\frac{1}{\tau}\frac{\rvy-\rvx}{\|\rvy-\rvx\|_2}$ (with the diagonal convention),
so the kernel-score identity gives
\[
\rvs_{\uppi,\tau}(\rvx)
=\frac{1}{\tau}\,
\mathbb E_{\rvy\sim \uppi_\tau(\cdot| \rvx)}
\Big[\frac{\rvy-\rvx}{\|\rvy-\rvx\|_2}\Big]
\qquad
\Longrightarrow\qquad
\|\rvs_{\uppi,\tau}(\rvx)\|_2\le \frac{1}{\tau}.
\]

\begin{lemma}[Residual Bound for $\boldsymbol\delta_{\uppi}$ of Laplace Kernel]
\label{lem:delta_bound_radial_full_xy}
Assume the second moment of $r=\|\rvy-\rvx\|_2$ under $\uppi_\tau(\cdot| \rvx)$ is finite. Then
\[
\|\boldsymbol\delta_{\uppi}(\rvx)\|_2 \le 2\sqrt{\mathrm{Var}(r| \rvx)}.
\]
\end{lemma}

\begin{proof}
For Laplace, $b_\tau(r)=\tau/r$, so $b_\tau(r)^{-1}=r/\tau$ and
\[
b_\tau(r)(\rvy-\rvx)=\tau\cdot u,
\qquad
u:=\frac{\rvy-\rvx}{\|\rvy-\rvx\|_2},
\qquad \|u\|_2\le 1.
\]
By definition,
\[
\boldsymbol\delta_{\uppi}(\rvx)
=\operatorname{Cov}_{\rvy\sim \uppi_\tau(\cdot| \rvx)}\Big(r/\tau,~\tau u\Big).
\]
Cauchy--Schwarz for scalar--vector covariance yields
\[
\|\boldsymbol\delta_{\uppi}(\rvx)\|_2
\le
\sqrt{\mathrm{Var}(r/\tau| \rvx)}\cdot
\sqrt{\mathbb E\|\tau u-\mathbb E[\tau u]\|_2^2}.
\]
Since $\|\tau u-\mathbb E[\tau u]\|_2\le 2\tau$, the second factor is at most $2\tau$, hence
\[
\|\boldsymbol\delta_{\uppi}(\rvx)\|_2 \le 2\tau\sqrt{\mathrm{Var}(r/\tau| \rvx)}=2\sqrt{\mathrm{Var}(r| \rvx)}.
\]
\end{proof}

Now we prove the objective-alignment result in \Cref{thm:field_alignment_radial_full_xy_prime}. A rigorous statement is as follows:
\begingroup
\let\thetheoremorig\thetheorem
\addtocounter{theorem}{-3} 
\renewcommand{\thetheorem}{\thetheoremorig\ensuremath{'}}
\begin{theorem}[Large-$D$ Field Alignment at $1/D$ Rate]
\label{thm:field_alignment_radial_full_xy_prime}
Assume \Cref{ass:shell_R0,ass:ip_L2,ass:bounded_norm} and let $\rho=\sqrt2R_0$ and $C:=\tau\rho=\bar\tau\rho D^a$.
Then there exists a constant $K>0$ independent of $D$ and independent of learnable parameters such that
for all sufficiently large $D$ and all $\rvf\in\mathcal F$,
\[
\mathbb E_{\rvx\sim q_{\rvf}}
\Big\|\Delta_{p,q_\rvf}(\rvx)-C\,\Delta\rvs_{p, q_\rvf}(\rvx)\Big\|_2^2
\le \frac{K}{D}.
\]
\end{theorem}
\endgroup

\begin{proof}
Apply \Cref{prop:radial_precond_decomp} to $\uppi=p$ and $\uppi=q_{\rvf}$ and subtract:
\[
\Delta_{p,q_\rvf}
=
\tau^2\big(\alpha_{p}\rvs_{p,\tau}-\alpha_{q_{\rvf}}\rvs_{q_{\rvf},\tau}\big)
+\big(\boldsymbol\delta_{p}-\boldsymbol\delta_{q_{\rvf}}\big).
\]
Add and subtract $(\rho/\tau)\rvs_{p,\tau}$ and $(\rho/\tau)\rvs_{q_{\rvf},\tau}$:
\[
\Delta_{p,q_\rvf}
=
C\,\Delta\rvs_{p, q_\rvf}
+\tau^2\Big((\alpha_{p}-\rho/\tau)\rvs_{p,\tau}-(\alpha_{q_{\rvf}}-\rho/\tau)\rvs_{q_{\rvf},\tau}\Big)
+\big(\boldsymbol\delta_{p}-\boldsymbol\delta_{q_{\rvf}}\big).
\]
Using $(a+b)^2\le 2a^2+2b^2$ repeatedly, it suffices to control for $\uppi\in\{p,q_{\rvf}\}$:
\[
\mathbb E_{\rvx\sim q_{\rvf}}\Big\|\tau^2(\alpha_{\uppi}(\rvx)-\rho/\tau)\rvs_{\uppi,\tau}(\rvx)\Big\|_2^2
\quad\text{and}\quad
\mathbb E_{\rvx\sim q_{\rvf}}\|\boldsymbol\delta_{\uppi}(\rvx)\|_2^2,
\]
by $\mathcal O(1/D)$ with constants independent of $\rvf$.

\emph{(i) pre-conditioner term.}
For Laplace, $b_\tau(r)=\tau/r$, so $\alpha_{\uppi}(\rvx)=\mathbb E[r| \rvx]/\tau$, hence $\tau\alpha_{\uppi}(\rvx)=\mathbb E[r| \rvx]$.
Using $\|\rvs_{\uppi,\tau}(\rvx)\|_2\le 1/\tau$,
\[
\Big\|\tau^2(\alpha_{\uppi}-\rho/\tau)\rvs_{\uppi,\tau}\Big\|_2
\le
\tau^2\Big|\alpha_{\uppi}-\frac{\rho}{\tau}\Big|\cdot \frac{1}{\tau}
=
\Big|\tau\alpha_{\uppi}-\rho\Big|.
\]
By Jensen,
\[
\big(\tau\alpha_{\uppi}(\rvx)-\rho\big)^2
=
\big(\mathbb E[r| \rvx]-\rho\big)^2
\le
\mathbb E[(r-\rho)^2| \rvx].
\]
Averaging over $\rvx\sim q_{\rvf}$ and applying \Cref{lem:reweight_L2_radial_full_xy} yields
\[
\mathbb E_{\rvx\sim q_{\rvf}}\big(\tau\alpha_{\uppi}(\rvx)-\rho\big)^2
=
\mathcal O(1/D),
\qquad \uppi\in\{p,q_{\rvf}\}.
\]

\emph{(ii) Residual term.}
By \Cref{lem:delta_bound_radial_full_xy},
\[
\|\boldsymbol\delta_{\uppi}(\rvx)\|_2^2
\le 4\,\mathrm{Var}(r| \rvx)
\le 4\,\mathbb E[(r-\rho)^2| \rvx],
\]
since $\mathbb E[(r-\rho)^2| \rvx]=\mathrm{Var}(r| \rvx)+(\mathbb E[r| \rvx]-\rho)^2\ge \mathrm{Var}(r| \rvx)$.
Averaging over $\rvx\sim q_{\rvf}$ and applying \Cref{lem:reweight_L2_radial_full_xy} again gives $\mathcal O(1/D)$.

Combining the bounds yields $\mathbb E\|\Delta_{p,q_\rvf}-C\Delta\rvs_{p, q_\rvf}\|_2^2\le K/D$ for some $K$
depending only on $(\sigma,\kappa,R_0,B,\bar\tau,a)$ (and not on learnable parameters).
\end{proof}

\begin{corollary}[Drift Controls Score-Mismatch]
\label{cor:drift_controls_score_radial_full_xy}
Assume \Cref{thm:field_alignment_radial_full_xy_prime}. Then for every $\rvf\in\mathcal F$,
\[
\mathbb E_{\rvx\sim q_{\rvf}}\|\Delta\rvs_{p, q_\rvf}(\rvx)\|_2^2
\le
\frac{2}{C^2}\,\mathbb E_{\rvx\sim q_{\rvf}}\|\Delta_{p,q_\rvf}(\rvx)\|_2^2
+\frac{2K}{C^2D},
\qquad C=\tau\rho.
\]
\end{corollary}

\begin{proof}
From \Cref{thm:field_alignment_radial_full_xy_prime}, write
$\Delta_{p,q_\rvf}=C\Delta\rvs_{p, q_\rvf}+\boldsymbol\varepsilon_{\rvf}$ with
$\mathbb E\|\boldsymbol\varepsilon_{\rvf}\|_2^2\le K/D$.
Then
\[
C^2\|\Delta\rvs_{p, q_\rvf}\|_2^2=\|\Delta_{p,q_\rvf}-\boldsymbol\varepsilon_{\rvf}\|_2^2
\le 2\|\Delta_{p,q_\rvf}\|_2^2+2\|\boldsymbol\varepsilon_{\rvf}\|_2^2.
\]
Take expectation over $\rvx\sim q_{\rvf}$.
\end{proof}

\paragraph{Step 4: Score-Matching Minimizers Imply  $q_{\rvg^\star}=p$.}

\begin{lemma}[Kernel Score Identity for Laplace Smoothing]
\label{lem:kernel_score_identity_laplace_full_xy}
For any probability measure $\mu$ on $\mathbb R^D$, the function $\mu_{k_\tau}$ is Lipschitz and hence differentiable Lebesgue-a.e.
Moreover, wherever $\nabla \mu_{k_\tau}(\rvx)$ exists,
\[
\nabla \mu_{k_\tau}(\rvx)
=
\mathbb E_{\rvy\sim \mu}\big[\nabla_{\rvx}k_\tau(\rvx,\rvy)\big],
\qquad
\rvs_{\mu,\tau}(\rvx)=\nabla\log \mu_{k_\tau}(\rvx).
\]
\end{lemma}

\begin{proof}
For fixed $\rvy$, define $\nabla_{\rvx}k_\tau(\rvx,\rvy):=\frac{1}{\tau}k_\tau(\rvx,\rvy)\,\rvu(\rvx,\rvy)$
for $\rvy\neq \rvx$ and $0$ on the diagonal, where $\|\rvu(\rvx,\rvy)\|_2\le 1$.
Then $\|\nabla_{\rvx}k_\tau(\rvx,\rvy)\|_2\le 1/\tau$, so $\rvx\mapsto k_\tau(\rvx,\rvy)$ is $1/\tau$-Lipschitz,
and therefore $\mu_{k_\tau}$ is $1/\tau$-Lipschitz as an expectation.
By Rademacher, $\mu_{k_\tau}$ is differentiable Lebesgue-a.e.
At differentiability points, dominated convergence justifies differentiating under the expectation to obtain
$\nabla\mu_{k_\tau}(\rvx)=\mathbb E_{\rvy\sim\mu}[\nabla_{\rvx}k_\tau(\rvx,\rvy)]$.
Dividing by $\mu_{k_\tau}(\rvx)>0$ yields $\nabla\log \mu_{k_\tau}=\rvs_{\mu,\tau}$ at those points.
\end{proof}

\begin{lemma}[Injectivity of Laplace Smoothing]
\label{lem:laplace_injective_full_xy}
Let $k_\tau(\rvz)=\exp(-\|\rvz\|_2/\tau)$ and let $\mu,\nu$ be finite Borel measures on $\mathbb R^D$.
If $\mu*k_\tau=\nu*k_\tau$ as functions, then $\mu=\nu$ as measures.
\end{lemma}

\begin{proof}
Take Fourier transforms:
\[
\widehat\mu(\boldsymbol\omega)\,\widehat{k_\tau}(\boldsymbol\omega)
=
\widehat\nu(\boldsymbol\omega)\,\widehat{k_\tau}(\boldsymbol\omega).
\]
For the radial Laplace kernel, $\widehat{k_\tau}(\boldsymbol\omega)>0$ for all $\boldsymbol\omega$
(e.g.\ $\widehat{k_\tau}(\boldsymbol\omega)=c_{D,\tau}(1+\tau^2\|\boldsymbol\omega\|_2^2)^{-(D+1)/2}$ with $c_{D,\tau}>0$),
so $\widehat\mu=\widehat\nu$ and hence $\mu=\nu$.
\end{proof}

\begin{proposition}[SM Identification: $\mathcal L_{\mathrm{SM}}(\rvg)=0 \Rightarrow q_{\rvg}=p$]
\label{prop:SM_minimizer_is_p_radial_full_xy}
Assume \Cref{ass:ac}. If $\mathcal L_{\mathrm{SM}}(\rvg)=0$, i.e.
\[
\mathbb E_{\rvx\sim q_{\rvg}}\|\rvs_{p,\tau}(\rvx)-\rvs_{q_{\rvg},\tau}(\rvx)\|_2^2=0,
\]
then $q_{\rvg}=p$ as distributions on $\mathbb R^D$.
\end{proposition}

\begin{proof}
$\mathcal L_{\mathrm{SM}}(\rvg)=0$ implies $\rvs_{p,\tau}(\rvx)=\rvs_{q_{\rvg},\tau}(\rvx)$ for $q_{\rvg}$-a.e.\ $\rvx$.
By \Cref{ass:ac}, $q_{\rvg}$ has a density strictly positive Lebesgue-a.e., hence the equality holds Lebesgue-a.e.

By \Cref{lem:kernel_score_identity_laplace_full_xy}, wherever both gradients exist,
\[
\nabla\log p_{k_\tau}(\rvx)=\nabla\log (q_{\rvg})_{k_\tau}(\rvx)
\quad\text{for Lebesgue-a.e.\ }\rvx.
\]
Thus $\nabla h(\rvx)=0$ Lebesgue-a.e.\ for $h(\rvx):=\log p_{k_\tau}(\rvx)-\log (q_{\rvg})_{k_\tau}(\rvx)$.
Since both smoothed functions are positive and Lipschitz, $h$ is locally Lipschitz.
A locally Lipschitz function with zero gradient a.e.\ is constant, so $h(\rvx)\equiv \log c$ and
$p_{k_\tau}(\rvx)=c\,(q_{\rvg})_{k_\tau}(\rvx)$ for Lebesgue-a.e.\ $\rvx$.
By continuity, the equality holds for all $\rvx$.

Integrate both sides over $\mathbb R^D$ and use Fubini/translation invariance:
\[
\int_{\mathbb R^D}\mu_{k_\tau}(\rvx)\,\diff\rvx
=
\Big(\int_{\mathbb R^D}k_\tau(\rvz)\,\diff\rvz\Big)\int_{\mathbb R^D}\mu(\rvy)\,\diff\rvy,
\]
so $\int p_{k_\tau}=\int (q_{\rvg})_{k_\tau}$ and therefore $c=1$.
Hence $p_{k_\tau}=(q_{\rvg})_{k_\tau}$, and \Cref{lem:laplace_injective_full_xy} gives $p=q_{\rvg}$.
\end{proof}

\paragraph{Main Proof to \Cref{thm:D_consistency_prime}.}

Let
\[
\rvf^{\star}\in\arg\min_{\rvf\in\mathcal F} \mathcal{L}_{\mathrm{drift}}(\rvf),
\qquad
\rvg^{\star}\in\arg\min_{\rvg\in\mathcal G} \mathcal{L}_{\mathrm{SM}}(\rvg).
\]
Fix $a\ge 0$ and use the dimension-aware bandwidth $\tau=\bar\tau D^{a}$ with $\bar\tau>0$. By realizability, $\inf_{\rvg\in\mathcal G}\mathcal L_{\mathrm{SM}}(\rvg)=0$, hence $\mathcal L_{\mathrm{SM}}(\rvg^\star)=0$
and \Cref{prop:SM_minimizer_is_p_radial_full_xy} gives $q_{\rvg^\star}=p$.

By realizability, $\inf_{\rvf\in\mathcal F}\mathcal L_{\mathrm{drift}}(\rvf)=0$, hence any minimizer satisfies $\mathcal L_{\mathrm{drift}}(\rvf^\star)=0$, i.e.
\[
\mathbb E_{\rvx\sim q_{\rvf^\star}}\|\mathbf V_{\rvf^\star}(\rvx)\|_2^2=0.
\]
Apply \Cref{cor:drift_controls_score_radial_full_xy} to $\rvf^\star$:
\[
\mathbb E_{\rvx\sim q_{\rvf^\star}}\|\Delta\rvs_{\rvf^\star}(\rvx)\|_2^2
\le \frac{2K}{C^2D},
\qquad
C=\tau\rho=\bar\tau\rho D^a.
\]
Since $\Delta\rvs_{\rvf^\star}=\rvs_{p,\tau}-\rvs_{q_{\rvf^\star},\tau}$, by definition
\[
\mathcal D_{\mathrm{rF}}(p\|q_{\rvf^\star})
=
\mathbb E_{\rvx\sim q_{\rvf^\star}}\|\Delta\rvs_{\rvf^\star}(\rvx)\|_2^2
\le
\frac{2K}{(\rho\bar\tau)^2}\cdot \frac{1}{D^{1+2a}}.
\]
The equivalent statement with $q_{\rvg^\star}$ follows from $q_{\rvg^\star}=p$.

\subsection{Gradient-Level Alignment}

The stop-gradient solver optimizes the drifting loss by treating the transported target
$\mathrm{sg}(\rvx+\Delta_{p,q_\btheta}(\rvx))$ as constant for backpropagation.
Consequently, the update direction is the \emph{semi-gradient}
\[
\mathbf g_{\mathrm{drift}}(\btheta):= \nabla_\btheta \mathcal L_{\mathrm{drift}}(\btheta)
=
-2\,\mathbb E_{\beps}\Big[
\mathbf J_{\rvf_\btheta}(\beps)^\top\,\Delta_{p,q_\btheta}\big(\rvf_\btheta(\beps)\big)
\Big],
\qquad
q_\btheta := (\rvf_\btheta)_\# p_{\beps}.
\]
This differs from the full gradient of the population functional
$\btheta\mapsto \mathbb E_{\rvx\sim q_\btheta}\|\Delta_{p,q_\btheta}(\rvx)\|_2^2$,
which would additionally differentiate through the dependence of $\Delta_{p,q_\btheta}$ on $q_\btheta$.
Therefore, the appropriate notion of “gradient-level equivalence” for the implemented drifting algorithm is
\emph{alignment between implemented semi-gradient directions}.

In the Laplace-kernel setting, our field-alignment results imply that the drifting discrepancy field
\[
\Delta_{p,q_\btheta}(\rvx)
=
\eta\big(\mathbf V_{p,k_\tau}(\rvx)-\mathbf V_{q_\btheta,k_\tau}(\rvx)\big)
\]
is close to a scaled score-mismatch field. Recall the definition of score $\rvs_{\uppi,\tau}(\rvx):=\nabla_{\rvx}\log \uppi_{k_\tau}(\rvx)$, we define the scale-$\tau$ score-mismatch
\[
\Delta\rvs_{p, q_{\btheta}}(\rvx)
:=
\eta\,C\,\big(\rvs_{p,\tau}(\rvx)-\rvs_{q_\btheta,\tau}(\rvx)\big),
\qquad
C:=\tau\rho.
\]
To compare the implemented drifting update to an update driven by score-mismatch, we introduce the
\emph{score-transport fixed-point loss}
\begin{align}\label{eq:st-projf}
    \mathcal L_{\mathrm{fp\mbox{-}score}}(\btheta)
:=
\mathbb E_{\beps}\Big[
\big\|\rvf_\btheta(\beps)
-\mathrm{sg}\big(\rvf_\btheta(\beps)+\Delta\rvs_{p, q_{\btheta}}(\rvf_\btheta(\beps))\big)
\big\|_2^2\Big],
\end{align}
whose semi-gradient is
\[
\mathbf g_{\mathrm{ST}}(\btheta):=\nabla_\btheta \mathcal L_{\mathrm{fp\mbox{-}score}}(\btheta)
=
-2\,\mathbb E_{\beps}\Big[
\mathbf J_{\rvf_\btheta}(\beps)^\top\,\Delta\rvs_{p, q_{\btheta}}\big(\rvf_\btheta(\beps)\big)
\Big].
\]

We remark that the population score-matching objective is
\[
\mathcal{L}_{\mathrm{SM}}(\btheta)
:=
\mathbb E_{\rvx\sim q_{\rvg_\btheta}}
\big\|\rvs_{p,\tau}(\rvx)-\rvs_{q_{\rvg_\btheta},\tau}(\rvx)\big\|_2^2,
\]
and it involves no stop-gradient. The full gradient $\nabla_\btheta \mathcal{L}_{\mathrm{SM}}(\btheta)$
differentiates through the $\btheta$-dependence of the smoothed score $\rvs_{q_{\rvg_\btheta},\tau}$ and is
not, in general, equal to the semi-gradient of $\mathcal L_{\mathrm{fp\mbox{-}score}}$.
The role of $\mathcal L_{\mathrm{fp\mbox{-}score}}$ is purely to provide an algorithm-level comparator
that matches the transport--projection structure of the drifting stop-gradient update.

First, we state the Jacobian control needed to turn field alignment into gradient alignment:
\begin{assumption}[Uniform Jacobian Second Moment]
\label{ass:jacobian_L2}
There exists $J_2<\infty$ (uniform in $\btheta$ and in $D$ for the considered model class) such that
\[
\mathbb E_{\beps}\big[\|\mathbf J_{\rvf_\btheta}(\beps)\|_{\mathrm{op}}^2\big]\le J_2
\qquad\text{for all }\btheta.
\]
\end{assumption}

Next, we first state a rigorous version of \Cref{thm:grad_alignment_largeD} and then prove it.
\begingroup
\let\thetheoremorig\thetheorem
\addtocounter{theorem}{0} 
\renewcommand{\thetheorem}{\thetheoremorig\ensuremath{'}}
\begin{theorem}[Large-$D$ Alignment]
\label{thm:grad_alignment_largeD_prime}
Assume the hypotheses of \Cref{thm:field_alignment_radial_full_xy}, and additionally assume \Cref{ass:jacobian_L2}.
Let $\Delta_{p,q_\btheta}(\rvx)=\eta(\mathbf V_{p,k_\tau}(\rvx)-\mathbf V_{q_\btheta,k_\tau}(\rvx))$
and $\Delta\rvs_{p, q_{\btheta}}(\rvx)=\eta\,C\,\big(\rvs_{p,\tau}(\rvx)-\rvs_{q_\btheta,\tau}(\rvx)\big)$ with $C=\tau\rho$.
Then there exists $D_0\in\mathbb N$ such that for all $D\ge D_0$ and all $\btheta$,
\[
\big\|
\mathbf g_{\mathrm{drift}}(\btheta)
-
\mathbf g_{\mathrm{ST}}(\btheta)
\big\|_2
\;\le\;
2\,\eta\,\sqrt{J_2}\,\sqrt{\frac{K}{D}},
\]
where $K$ is the constant from \Cref{thm:field_alignment_radial_full_xy}.
\end{theorem}
\endgroup

\begin{proof}
By definition of the two semi-gradients,
\[
\mathbf g_{\mathrm{drift}}(\btheta)
-
\mathbf g_{\mathrm{ST}}(\btheta)
=
-2\,\mathbb E_{\beps}\Big[
\mathbf J_{\rvf_\btheta}(\beps)^\top
\Big(\Delta_{p,q_\btheta}(\rvx)-\Delta\rvs_{p, q_{\btheta}}(\rvx)\Big)
\Big],
\qquad \rvx=\rvf_\btheta(\beps).
\]
Using $\|\rmA^\top \rvv\|_2\le \|\rmA\|_{\mathrm{op}}\|\rvv\|_2$ and Cauchy--Schwarz,
\begin{align*}
\Big\|
\mathbf g_{\mathrm{drift}}(\btheta)
-
\mathbf g_{\mathrm{ST}}(\btheta)
\Big\|_2
&\le
2\sqrt{\mathbb E_{\beps}\|\mathbf J_{\rvf_\btheta}(\beps)\|_{\mathrm{op}}^2}\;
\sqrt{\mathbb E_{\beps}\big\|\Delta_{p,q_\btheta}(\rvx)-\Delta\rvs_{p, q_{\btheta}}(\rvx)\big\|_2^2}.
\end{align*}
By \Cref{ass:jacobian_L2}, the first factor is at most $\sqrt{J_2}$.

For the second factor,
\[
\Delta_{p,q_\btheta}(\rvx)-\Delta\rvs_{p, q_{\btheta}}(\rvx)
=
\eta\Big(
(\mathbf V_{p,k_\tau}(\rvx)-\mathbf V_{q_\btheta,k_\tau}(\rvx))
- C\,\big(\rvs_{p,\tau}(\rvx)-\rvs_{q_\btheta,\tau}(\rvx)\big)
\Big).
\]
Applying \Cref{thm:field_alignment_radial_full_xy} (with $\rvf=\rvf_\btheta$) yields that for all $D\ge D_0$,
\[
\mathbb E_{\rvx\sim q_\btheta}
\Big\|
(\mathbf V_{p,k_\tau}(\rvx)-\mathbf V_{q_\btheta,k_\tau}(\rvx))
- C\,\big(\rvs_{p,\tau}(\rvx)-\rvs_{q_\btheta,\tau}(\rvx)\big)
\Big\|_2^2
\le
\frac{K}{D}.
\]
Therefore,
\[
\mathbb E_{\beps}\big\|\Delta_{p,q_\btheta}(\rvx)-\Delta\rvs_{p, q_{\btheta}}(\rvx)\big\|_2^2
=
\mathbb E_{\rvx\sim q_\btheta}\big\|\Delta_{p,q_\btheta}(\rvx)-\Delta\rvs_{p, q_{\btheta}}(\rvx)\big\|_2^2
\le
\eta^2\frac{K}{D},
\]
and the claim follows.
\end{proof}


\begin{corollary}[Cosine Similarity of Update Directions]
\label{cor:cosine_alignment_largeD}
Assume the hypotheses of \Cref{thm:grad_alignment_largeD}.
Fix any $\gamma>0$ and consider parameters $\btheta$ such that $\min\{\|\mathbf g_{\mathrm{drift}}(\btheta)\|_2, \|\mathbf g_{\mathrm{ST}}(\btheta)\|_2\}\ge \gamma$.
Then for all $D\ge D_0$,
\[
\cos\angle\big(\mathbf g_{\mathrm{drift}}(\btheta),\mathbf g_{\mathrm{ST}}(\btheta)\big)
\;\ge\;
1-\frac{16\,\eta^2\,J_2\,K}{\gamma^2\,D}.
\]
In particular, if $\|\mathbf g_{\mathrm{ST}}(\btheta)\|_2$ is bounded below by a constant $\gamma>0$ independent of $D$,
then the two update directions become asymptotically parallel with
$\cos\angle(\mathbf g_{\mathrm{drift}},\mathbf g_{\mathrm{ST}})\ge 1-\mathcal O(D^{-1})$.
\end{corollary}

\begin{proof}
We may assume $\|\mathbf g_{\mathrm{ST}}(\btheta)\|_2\ge \gamma$.
Let $\mathbf a:=\mathbf g_{\mathrm{ST}}(\btheta)$ and $\mathbf b:=\mathbf g_{\mathrm{drift}}(\btheta)$.
From \Cref{thm:grad_alignment_largeD}, $\|\mathbf b-\mathbf a\|_2\le \delta_D$ with
$\delta_D:=2\eta\sqrt{J_2}\sqrt{K/D}$.
Write $\mathbf b=\mathbf a+\mathbf e$ with $\|\mathbf e\|\le \delta_D$,
and set $t:=\delta_D/\gamma$.

Let $\theta:=\angle(\mathbf a,\mathbf b)$.
The component of $\mathbf e$ perpendicular to $\mathbf a$ satisfies
$\|\mathbf P_\perp \mathbf e\|\le \|\mathbf e\|\le \delta_D$,
and the reverse triangle inequality gives
$\|\mathbf b\|\ge \|\mathbf a\|-\|\mathbf e\|\ge \gamma-\delta_D$.
Therefore
\[
\sin\theta
=\frac{\|\mathbf P_\perp \mathbf e\|}{\|\mathbf b\|}
\le \frac{\delta_D}{\gamma-\delta_D}
=\frac{t}{1-t}.
\]
For all $D$ large enough that $t\le 1/2$
(equivalently $\delta_D\le \gamma/2$, which holds for $D\ge D_1:=16\eta^2 J_2 K/\gamma^2$),
we have $\theta\le \pi/2$.
For $\theta\in[0,\pi/2]$,
\[
1-\cos\theta
=\frac{\sin^2\theta}{1+\cos\theta}
\le \sin^2\theta,
\]
since $1+\cos\theta\ge 1$.
Combining the two bounds and using $(1-t)^2\ge 1/4$ for $t\le 1/2$,
\[
1-\cos\angle\big(\mathbf g_{\mathrm{drift}}(\btheta),\mathbf g_{\mathrm{ST}}(\btheta)\big)
\le \sin^2\theta
\le \frac{t^2}{(1-t)^2}
\le 4t^2
=\frac{4\delta_D^2}{\gamma^2}
=\frac{16\,\eta^2\,J_2\,K}{\gamma^2\,D}.
\]
\end{proof}

\section{Proofs of Practical-Implementation-Aligned High-Dimensional Regime}
\label{app:proof_impl_highD}

In this section, we prove the practical-implementation-aligned high-dimensional statements from \Cref{app:impl_highD}. The logic mirrors the population proof, but now begins from the exact force computed by the practical implementation at one temperature.

\subsection{Proof Roadmap and Technical Objects}

Fix one frozen feature head, one class label $c$, and one temperature $\tau>0$. Let
\[
\mathbf{x}_1,\dots,\mathbf{x}_{C_g}\in\mathbb R^D
\]
be the generated queries, and let
\[
G:=\{\mathbf{y}_1^G,\dots,\mathbf{y}_{C_g}^G\},
\qquad
U:=\{\mathbf{y}_1^U,\dots,\mathbf{y}_{C_n}^U\},
\qquad
P:=\{\mathbf{y}_1^P,\dots,\mathbf{y}_{C_p}^P\}
\]
denote the detached generated block, the unconditional negative block, and the class-positive block, respectively. Write
\[
T:=G\cup U\cup P
\]
for the full concatenated target pool, ordered exactly as in the practical implementation.

Let $d_{ij}=d(\mathbf{x}_i,\mathbf{z}_j)$ denote the normalized distances that actually enter the logits after the same batch-scale normalization and diagonal self-mask used by the practical implementation, and let $w_j>0$ be the target weight attached to $\mathbf{z}_j\in T$. The row-softmax and query-softmax affinities are
\[
A_{ij}^{\mathrm{row}}
=
\frac{\exp(-d_{ij}/\tau)}{\sum_{k\in T}\exp(-d_{ik}/\tau)},
\qquad
A_{ij}^{\mathrm{col}}
=
\frac{\exp(-d_{ij}/\tau)}{\sum_{\ell=1}^{C_g}\exp(-d_{\ell j}/\tau)},
\]
and the practical implementation uses the symmetrized affinity
\[
a_{ij}=w_j\sqrt{A_{ij}^{\mathrm{row}}A_{ij}^{\mathrm{col}}}.
\]

The corresponding negative and positive block masses are
\[
m_i^-:=\sum_{j\in G\cup U}a_{ij},
\qquad
m_i^P:=\sum_{j\in P}a_{ij},
\]
with barycenters
\[
\boldsymbol{\mu}_i^-:=\frac{1}{m_i^-}\sum_{j\in G\cup U}a_{ij}\mathbf{y}_j,
\qquad
\boldsymbol{\mu}_i^P:=\frac{1}{m_i^P}\sum_{j\in P}a_{ij}\mathbf{y}_j^P.
\]
The exact one-temperature implementation force is
\[
\mathbf{F}_{\tau}(\mathbf{x}_i)
:=
m_i^Pm_i^-\bigl(\boldsymbol{\mu}_i^P-\boldsymbol{\mu}_i^-\bigr).
\]
To separate the two parts of the negative block, define
\[
m_i^G:=\sum_{j\in G}a_{ij},
\qquad
m_i^U:=\sum_{j\in U}a_{ij},
\]
with barycenters
\[
\boldsymbol{\mu}_i^G:=\frac{1}{m_i^G}\sum_{j\in G}a_{ij}\mathbf{y}_j^G,
\qquad
\boldsymbol{\mu}_i^U:=\frac{1}{m_i^U}\sum_{j\in U}a_{ij}\mathbf{y}_j^U.
\]

\paragraph{The Two Measurable Practical Deviations.}
Removing only the query-axis softmax yields
\[
\bar a_{ij}:=w_jA_{ij}^{\mathrm{row}},
\]
with row-softmax-only barycenters
\[
\bar{\boldsymbol{\mu}}_i^G,
\qquad
\bar{\boldsymbol{\mu}}_i^U,
\qquad
\bar{\boldsymbol{\mu}}_i^P.
\]
The symmetrization distortion is
\[
\varepsilon_{\mathrm{sym},\tau}^2
:=
\mathbb E\Big[
\|\boldsymbol{\mu}_i^G-\bar{\boldsymbol{\mu}}_i^G\|_2^2
+
\|\boldsymbol{\mu}_i^U-\bar{\boldsymbol{\mu}}_i^U\|_2^2
+
\|\boldsymbol{\mu}_i^P-\bar{\boldsymbol{\mu}}_i^P\|_2^2
\Big].
\]
This term is purely implementation-induced.

Let
\[
\boldsymbol{\mu}_{\pi,k_\tau}(\mathbf{x})
:=
\frac{\mathbb E_{\mathbf{y}\sim\pi}[k_\tau(\mathbf{x},\mathbf{y})\mathbf{y}]}
{\mathbb E_{\mathbf{y}\sim\pi}[k_\tau(\mathbf{x},\mathbf{y})]}
=
\mathbf{x}+\mathbf{V}_{\pi,k_\tau}(\mathbf{x}),
\qquad
k_\tau(\mathbf{x},\mathbf{y})=\exp(-d(\mathbf{x},\mathbf{y})/\tau).
\]
The finite-pool kernel-estimation error is
\[
\varepsilon_{\mathrm{mc},\tau}^2
:=
\mathbb E\Big[
\|\bar{\boldsymbol{\mu}}_i^G-\boldsymbol{\mu}_{q_c,k_\tau}(\mathbf{x}_i)\|_2^2
+
\|\bar{\boldsymbol{\mu}}_i^U-\boldsymbol{\mu}_{p_{\emptyset},k_\tau}(\mathbf{x}_i)\|_2^2
+
\|\bar{\boldsymbol{\mu}}_i^P-\boldsymbol{\mu}_{p_c,k_\tau}(\mathbf{x}_i)\|_2^2
\Big].
\]
A useful interpretation is that each row-softmax-only barycenter is a self-normalized kernel ratio estimator. Under bounded feature norms and a denominator bounded away from zero, standard self-normalized concentration gives the schematic control
\[
\mathbb E\!\left[
\bigl\|\bar{\boldsymbol{\mu}}_i^\alpha-\boldsymbol{\mu}_{\pi_\alpha,k_\tau}(\mathbf{x}_i)\bigr\|_2^2
\,\middle|\,\mathbf{x}_i
\right]
\lesssim
\frac{\sigma_{\alpha,\tau}^2(\mathbf{x}_i)}{N_{\mathrm{eff},\alpha}(\mathbf{x}_i;\tau)},
\qquad
\alpha\in\{G,U,P\},
\]
where $\sigma_{\alpha,\tau}^2(\mathbf{x}_i)$ is the conditional kernel-posterior variance and
\[
N_{\mathrm{eff},\alpha}(\mathbf{x}_i;\tau)
:=
\frac{1}{\sum_{j\in\alpha}\bar w_{ij}^2}
\]
is the effective sample size of the row-softmax weights. Thus the practical implementation is sample hungry only when the effective kernel neighborhood is too small or unstable.

\paragraph{Deterministic Weight Bound.}
Since $0\le a_{ij}\le w_j$, define
\[
W_G:=\sum_{j\in G}w_j,
\qquad
W_U:=\sum_{j\in U}w_j,
\qquad
W_P:=\sum_{j\in P}w_j,
\qquad
M_\tau:=W_PW_G+W_PW_U.
\]
Then
\[
m_i^G\le W_G,
\qquad
m_i^U\le W_U,
\qquad
m_i^P\le W_P
\]
for every query $i$.

\paragraph{Population High-$D$ Input.}
The only population theorem used below is the already-proved large-$D$ Laplace field-alignment theorem, applied to the two pairs that actually appear in the practical implementation:
\begin{equation}
\label{eq:impl_pair_pg}
\mathbb E_{\mathbf{x}\sim q_c}
\bigl\|
\Delta_{p_c,q_c}(\mathbf{x})-C_\tau\Delta\mathbf{s}_{p_c,q_c}(\mathbf{x})
\bigr\|_2^2
\le
\frac{K_{PG,\tau}}{D},
\end{equation}
\begin{equation}
\label{eq:impl_pair_pu}
\mathbb E_{\mathbf{x}\sim q_c}
\bigl\|
\Delta_{p_c,p_{\emptyset}}(\mathbf{x})-C_\tau\Delta\mathbf{s}_{p_c,p_{\emptyset}}(\mathbf{x})
\bigr\|_2^2
\le
\frac{K_{PU,\tau}}{D},
\end{equation}
and we abbreviate
\[
K_\tau:=\max\{K_{PG,\tau},K_{PU,\tau}\}.
\]

\subsection{Auxiliary Tools and the Main Proof}

\begin{proof}[Proof of \Cref{prop:impl_exact_decomp}]
Since
\[
\boldsymbol{\mu}_i^- = \frac{m_i^G}{m_i^-}\boldsymbol{\mu}_i^G+\frac{m_i^U}{m_i^-}\boldsymbol{\mu}_i^U,
\]
we have
\begin{align*}
\mathbf{F}_{\tau}(\mathbf{x}_i)
&=
m_i^Pm_i^-\bigl(\boldsymbol{\mu}_i^P-\boldsymbol{\mu}_i^-\bigr)
\\
&=
m_i^Pm_i^-\left(\boldsymbol{\mu}_i^P-\frac{m_i^G}{m_i^-}\boldsymbol{\mu}_i^G-\frac{m_i^U}{m_i^-}\boldsymbol{\mu}_i^U\right)
\\
&=
m_i^Pm_i^G\bigl(\boldsymbol{\mu}_i^P-\boldsymbol{\mu}_i^G\bigr)
+
m_i^Pm_i^U\bigl(\boldsymbol{\mu}_i^P-\boldsymbol{\mu}_i^U\bigr).
\end{align*}
\end{proof}

We next define the population-guided force and the guided score field:
\[
\mathbf{F}_{\tau}^{\mathrm{pop}}(\mathbf{x}_i)
:=
m_i^P m_i^G\bigl(
\boldsymbol{\mu}_{p_c,k_\tau}(\mathbf{x}_i)-\boldsymbol{\mu}_{q_c,k_\tau}(\mathbf{x}_i)
\bigr)
+
m_i^P m_i^U\bigl(
\boldsymbol{\mu}_{p_c,k_\tau}(\mathbf{x}_i)-\boldsymbol{\mu}_{p_{\emptyset},k_\tau}(\mathbf{x}_i)
\bigr),
\]
\[
\boldsymbol{\Gamma}_{\tau}(\mathbf{x}_i)
:=
m_i^P m_i^G\,\Delta\mathbf{s}_{p_c,q_c}(\mathbf{x}_i)
+
m_i^P m_i^U\,\Delta\mathbf{s}_{p_c,p_{\emptyset}}(\mathbf{x}_i).
\]

\begingroup
\let\thetheoremorig\thetheorem
\addtocounter{theorem}{1} 
\renewcommand{\thetheorem}{\thetheoremorig\ensuremath{'}}
\begin{theorem}[Rigorous Practical-Implementation-Aligned One-Temperature Field Alignment]
\label{thm:impl_field_align_prime}
Under the assumptions stated in \Cref{app:impl_highD} and the pairwise population bounds \Cref{eq:impl_pair_pg}--\Cref{eq:impl_pair_pu}, there exists a universal numerical constant $c_0>0$ such that
\[
\mathbb E_{\mathbf{x}_i\sim q_c}
\bigl\|
\mathbf{F}_{\tau}(\mathbf{x}_i)-C_\tau\boldsymbol{\Gamma}_{\tau}(\mathbf{x}_i)
\bigr\|_2^2
\le
c_0 M_\tau^2
\left(
\varepsilon_{\mathrm{sym},\tau}^2
+
\varepsilon_{\mathrm{mc},\tau}^2
+
\frac{K_\tau}{D}
\right).
\]
\end{theorem}
\endgroup

\begin{proof}
We compare $\mathbf{F}_{\tau}$ with $C_\tau\boldsymbol{\Gamma}_{\tau}$ in three steps.

\paragraph{Step 1: Remove the Query-Axis Softmax.}
Define the row-softmax-only force
\[
\bar{\mathbf{F}}_{\tau}(\mathbf{x}_i)
:=
m_i^P m_i^G\bigl(\bar{\boldsymbol{\mu}}_i^P-\bar{\boldsymbol{\mu}}_i^G\bigr)
+
m_i^P m_i^U\bigl(\bar{\boldsymbol{\mu}}_i^P-\bar{\boldsymbol{\mu}}_i^U\bigr).
\]
Using the deterministic weight bound,
\begin{align*}
\|\mathbf{F}_{\tau}(\mathbf{x}_i)-\bar{\mathbf{F}}_{\tau}(\mathbf{x}_i)\|_2
&\le
m_i^Pm_i^G\,
\bigl\|
(\boldsymbol{\mu}_i^P-\boldsymbol{\mu}_i^G)-(\bar{\boldsymbol{\mu}}_i^P-\bar{\boldsymbol{\mu}}_i^G)
\bigr\|_2
\\
&\quad
+
m_i^Pm_i^U\,
\bigl\|
(\boldsymbol{\mu}_i^P-\boldsymbol{\mu}_i^U)-(\bar{\boldsymbol{\mu}}_i^P-\bar{\boldsymbol{\mu}}_i^U)
\bigr\|_2
\\
&\le
M_\tau\Bigl(
2\|\boldsymbol{\mu}_i^P-\bar{\boldsymbol{\mu}}_i^P\|_2
+
\|\boldsymbol{\mu}_i^G-\bar{\boldsymbol{\mu}}_i^G\|_2
+
\|\boldsymbol{\mu}_i^U-\bar{\boldsymbol{\mu}}_i^U\|_2
\Bigr).
\end{align*}
Since $(2a+b+c)^2\le 6(a^2+b^2+c^2)$,
\[
\mathbb E\|\mathbf{F}_{\tau}-\bar{\mathbf{F}}_{\tau}\|_2^2
\le
18M_\tau^2\varepsilon_{\mathrm{sym},\tau}^2.
\]

\paragraph{Step 2: Replace Finite-Pool Barycenters by Population Barycenters.}
By the same algebra,
\[
\mathbb E\|\bar{\mathbf{F}}_{\tau}-\mathbf{F}_{\tau}^{\mathrm{pop}}\|_2^2
\le
18M_\tau^2\varepsilon_{\mathrm{mc},\tau}^2.
\]

\paragraph{Step 3: Apply the Population High-$D$ Theorem Pairwise.}
Since
\[
\boldsymbol{\mu}_{\pi,k_\tau}(\mathbf{x})=\mathbf{x}+\mathbf{V}_{\pi,k_\tau}(\mathbf{x}),
\]
the $\mathbf{x}_i$ terms cancel in differences, so
\[
\boldsymbol{\mu}_{p_c,k_\tau}-\boldsymbol{\mu}_{q_c,k_\tau}=\Delta_{p_c,q_c},
\qquad
\boldsymbol{\mu}_{p_c,k_\tau}-\boldsymbol{\mu}_{p_{\emptyset},k_\tau}=\Delta_{p_c,p_{\emptyset}}.
\]
Therefore
\[
\mathbf{F}_{\tau}^{\mathrm{pop}}(\mathbf{x}_i)-C_\tau\boldsymbol{\Gamma}_{\tau}(\mathbf{x}_i)
=
m_i^Pm_i^G\,\boldsymbol{\xi}_{PG,\tau}(\mathbf{x}_i)
+
m_i^Pm_i^U\,\boldsymbol{\xi}_{PU,\tau}(\mathbf{x}_i),
\]
where
\[
\boldsymbol{\xi}_{PG,\tau}:=\Delta_{p_c,q_c}-C_\tau\Delta\mathbf{s}_{p_c,q_c},
\qquad
\boldsymbol{\xi}_{PU,\tau}:=\Delta_{p_c,p_{\emptyset}}-C_\tau\Delta\mathbf{s}_{p_c,p_{\emptyset}}.
\]
Using $(a+b)^2\le 2a^2+2b^2$, the deterministic weight bound, and \Cref{eq:impl_pair_pg}--\Cref{eq:impl_pair_pu},
\[
\mathbb E\|\mathbf{F}_{\tau}^{\mathrm{pop}}-C_\tau\boldsymbol{\Gamma}_{\tau}\|_2^2
\le
4M_\tau^2\frac{K_\tau}{D}.
\]

\paragraph{Step 4: Combine the Three Steps.}
Since
\[
\mathbf{F}_{\tau}-C_\tau\boldsymbol{\Gamma}_{\tau}
=
(\mathbf{F}_{\tau}-\bar{\mathbf{F}}_{\tau})
+
(\bar{\mathbf{F}}_{\tau}-\mathbf{F}_{\tau}^{\mathrm{pop}})
+
(\mathbf{F}_{\tau}^{\mathrm{pop}}-C_\tau\boldsymbol{\Gamma}_{\tau}),
\]
and $(a+b+c)^2\le 3(a^2+b^2+c^2)$, we obtain
\[
\mathbb E\|\mathbf{F}_{\tau}-C_\tau\boldsymbol{\Gamma}_{\tau}\|_2^2
\le
54M_\tau^2\varepsilon_{\mathrm{sym},\tau}^2
+
54M_\tau^2\varepsilon_{\mathrm{mc},\tau}^2
+
12M_\tau^2\frac{K_\tau}{D}.
\]
This gives the claim with, for example, $c_0=54$.
\end{proof}

\begin{proof}[Proof of \Cref{cor:impl_cosine}]
Let
\[
\mathbf{E}_{\tau}(\mathbf{x})
:=
\mathbf{F}_{\tau}(\mathbf{x})-C_\tau\boldsymbol{\Gamma}_{\tau}(\mathbf{x}).
\]
By \Cref{thm:impl_field_align_prime},
\[
\mathbb E\|\mathbf{E}_{\tau}(\mathbf{x})\|_2^2
\le
c_0M_\tau^2
\left(
\varepsilon_{\mathrm{sym},\tau}^2
+
\varepsilon_{\mathrm{mc},\tau}^2
+
\frac{K_\tau}{D}
\right).
\]
The additional assumption implies
\[
\|\mathbf{E}_{\tau}\|_{L^2(q_c)}\le \frac{1}{2}C_\tau\|\boldsymbol{\Gamma}_{\tau}\|_{L^2(q_c)}.
\]
Hence
\[
\|\mathbf{F}_{\tau}\|_{L^2(q_c)}
\ge
C_\tau\|\boldsymbol{\Gamma}_{\tau}\|_{L^2(q_c)}-\|\mathbf{E}_{\tau}\|_{L^2(q_c)}
\ge
\frac{1}{2}C_\tau\|\boldsymbol{\Gamma}_{\tau}\|_{L^2(q_c)}.
\]
Using
\[
\|\mathbf{F}_{\tau}-C_\tau\boldsymbol{\Gamma}_{\tau}\|_{L^2(q_c)}^2
=
\|\mathbf{F}_{\tau}\|_{L^2(q_c)}^2
+
C_\tau^2\|\boldsymbol{\Gamma}_{\tau}\|_{L^2(q_c)}^2
-
2C_\tau\langle \mathbf{F}_{\tau},\boldsymbol{\Gamma}_{\tau}\rangle_{L^2(q_c)},
\]
we obtain
\[
1-
\mathbb E\bigl[\langle \overline{\mathbf{F}}_{\tau},\overline{\boldsymbol{\Gamma}}_{\tau}\rangle\bigr]
\le
\frac{\|\mathbf{E}_{\tau}\|_{L^2(q_c)}^2}
{2C_\tau\|\mathbf{F}_{\tau}\|_{L^2(q_c)}\|\boldsymbol{\Gamma}_{\tau}\|_{L^2(q_c)}}
\le
\frac{\|\mathbf{E}_{\tau}\|_{L^2(q_c)}^2}
{C_\tau^2\|\boldsymbol{\Gamma}_{\tau}\|_{L^2(q_c)}^2}.
\]
Since $\|\boldsymbol{\Gamma}_{\tau}\|_{L^2(q_c)}^2\ge \gamma_\tau$, the result follows.
\end{proof}

\subsection{Semi-Gradient Alignment}

For each feature head $\ell\in\mathcal L$ and temperature $\tau\in\mathcal T$, let
\[
\widehat{\Delta}^{\mathrm{impl}}_{\boldsymbol{\theta},\ell}(\mathbf{x})
:=
\sum_{\tau\in\mathcal T}\beta_{\ell,\tau}\mathbf{F}_{\ell,\tau}(\mathbf{x}),
\qquad
\widehat{\Delta}^{\mathrm{guide}}_{\boldsymbol{\theta},\ell}(\mathbf{x})
:=
\sum_{\tau\in\mathcal T}\beta_{\ell,\tau}C_{\ell,\tau}\boldsymbol{\Gamma}_{\ell,\tau}(\mathbf{x}).
\]
The practical-implementation semi-gradient and its guided score comparator are
\[
\mathbf{g}_{\mathrm{impl}}(\boldsymbol{\theta})
:=
-2\,\mathbb E_{\boldsymbol{\epsilon}}
\Big[
\sum_{\ell\in\mathcal L}
\mathbf{J}_{\Psi_\ell\circ\mathbf{f}_{\boldsymbol{\theta}}}(\boldsymbol{\epsilon})^{\top}
\widehat{\Delta}^{\mathrm{impl}}_{\boldsymbol{\theta},\ell}(\mathbf{f}_{\boldsymbol{\theta}}(\boldsymbol{\epsilon}))
\Big],
\]
\[
\mathbf{g}_{\mathrm{guide}}(\boldsymbol{\theta})
:=
-2\,\mathbb E_{\boldsymbol{\epsilon}}
\Big[
\sum_{\ell\in\mathcal L}
\mathbf{J}_{\Psi_\ell\circ\mathbf{f}_{\boldsymbol{\theta}}}(\boldsymbol{\epsilon})^{\top}
\widehat{\Delta}^{\mathrm{guide}}_{\boldsymbol{\theta},\ell}(\mathbf{f}_{\boldsymbol{\theta}}(\boldsymbol{\epsilon}))
\Big].
\]
\begingroup
\let\thetheoremorig\thetheorem
\addtocounter{theorem}{0} 
\renewcommand{\thetheorem}{\thetheoremorig\ensuremath{'}}
\begin{theorem}[Rigorous Practical-Implementation-Aligned Semi-Gradient Alignment]
\label{thm:impl_semigrad_prime}
Assume that \Cref{thm:impl_field_align_prime} holds for every $(\ell,\tau)\in\mathcal L\times\mathcal T$. Assume furthermore that, for each feature head $\ell$, there exists $J_\ell<\infty$ such that
\[
\mathbb E_{\boldsymbol{\epsilon}}
\bigl[
\|\mathbf{J}_{\Psi_\ell\circ\mathbf{f}_{\boldsymbol{\theta}}}(\boldsymbol{\epsilon})\|_{\mathrm{op}}^2
\bigr]
\le J_\ell^2
\qquad\text{for all }\boldsymbol{\theta}.
\]
Then
\[
\|\mathbf{g}_{\mathrm{impl}}(\boldsymbol{\theta})-\mathbf{g}_{\mathrm{guide}}(\boldsymbol{\theta})\|_2
\le
2\sum_{\ell\in\mathcal L}\sum_{\tau\in\mathcal T}
\beta_{\ell,\tau}J_\ell
\left[
c_0M_{\ell,\tau}^2
\left(
\varepsilon_{\mathrm{sym},\ell,\tau}^2
+
\varepsilon_{\mathrm{mc},\ell,\tau}^2
+
\frac{K_{\ell,\tau}}{D}
\right)
\right]^{1/2}.
\]
\end{theorem}
\endgroup

\begin{proof}
Subtract the two semi-gradients and apply
\[
\|\mathbf{A}^{\top}\mathbf{v}\|_2\le \|\mathbf{A}\|_{\mathrm{op}}\|\mathbf{v}\|_2
\]
followed by Cauchy--Schwarz:
\begin{align*}
\|\mathbf{g}_{\mathrm{impl}}(\boldsymbol{\theta})-\mathbf{g}_{\mathrm{guide}}(\boldsymbol{\theta})\|_2
&\le
2\sum_{\ell\in\mathcal L}
\sqrt{\mathbb E\|\mathbf{J}_{\Psi_\ell\circ\mathbf{f}_{\boldsymbol{\theta}}}(\boldsymbol{\epsilon})\|_{\mathrm{op}}^2}
\sqrt{\mathbb E\bigl\|\widehat{\Delta}^{\mathrm{impl}}_{\boldsymbol{\theta},\ell}-\widehat{\Delta}^{\mathrm{guide}}_{\boldsymbol{\theta},\ell}\bigr\|_2^2}
\\
&\le
2\sum_{\ell\in\mathcal L}J_\ell
\sqrt{\mathbb E\bigl\|\widehat{\Delta}^{\mathrm{impl}}_{\boldsymbol{\theta},\ell}-\widehat{\Delta}^{\mathrm{guide}}_{\boldsymbol{\theta},\ell}\bigr\|_2^2}.
\end{align*}
Now expand the temperature sum and use the triangle inequality in $L^2$:
\[
\sqrt{\mathbb E\bigl\|\widehat{\Delta}^{\mathrm{impl}}_{\boldsymbol{\theta},\ell}-\widehat{\Delta}^{\mathrm{guide}}_{\boldsymbol{\theta},\ell}\bigr\|_2^2}
\le
\sum_{\tau\in\mathcal T}
\beta_{\ell,\tau}
\sqrt{\mathbb E\|\mathbf{F}_{\ell,\tau}-C_{\ell,\tau}\boldsymbol{\Gamma}_{\ell,\tau}\|_2^2}.
\]
Applying \Cref{thm:impl_field_align_prime} to each $(\ell,\tau)$ block yields the claim.
\end{proof}


\section{Proofs of the Low-Temperature Theorems}\label{app:proof-small-tau}
In this section, we prove the small-temperature result \Cref{thm:small_bar_tau_consistency}, which establishes a proxy relationship between score matching and the drifting model at their optimal distributions in the small-temperature regime. We state a rigorous version of this relationship below.

\begingroup
\let\thetheoremorig\thetheorem
\addtocounter{theorem}{0} 
\renewcommand{\thetheorem}{\thetheoremorig\ensuremath{'}}
\begin{theorem}[Small-$\bar\tau$ agreement between mean-shift and score matching]
\label{thm:small_bar_tau_consistency_prime}
Assume \Cref{ass:ac,ass:small_tau_envelope_uniform,ass:realizable_fixed}.
Fix $a\ge 0$ and set $\tau=\tau_D:=\bar\tau D^a$.
Let $\tau_0$ be as in \Cref{ass:small_tau_envelope_uniform}.
For each $\tau\in(0,\tau_0]$, let
\[
\rvf^{\star}(\tau)\in\arg\min_{\rvf\in\mathcal F}\mathcal L_{\mathrm{drift}}(\rvf).
\]
Then, as $\bar\tau\to 0$ (equivalently $\tau\to 0$),
\[
\mathcal D_{\mathrm{rF}}\!\bigl(p\|q_{\rvf^\star(\tau)}\bigr)
=
\mathcal O(\bar\tau^4).
\]
Moreover, by realizability (\Cref{ass:realizable_fixed}) there exists $\rvg_{\mathrm{data}}\in\mathcal G$ such that
$q_{\rvg_{\mathrm{data}}}=p$, and hence
\[
\mathcal D_{\mathrm{rF}}\!\bigl(q_{\rvg_{\mathrm{data}}}\|q_{\rvf^\star(\tau)}\bigr)
=
\mathcal D_{\mathrm{rF}}\!\bigl(p\|q_{\rvf^\star(\tau)}\bigr)
=
\mathcal O(\bar\tau^4).
\]
The hidden constant in $\mathcal{O}(\cdot)$ is independent of $\bar\tau$ and of the learnable parameters
(but may depend on $D,a,\eta$ and the uniform envelope moment bound in \Cref{ass:small_tau_envelope_uniform}).
\end{theorem}
\endgroup

\subsection{Setup and Technical Assumptions}

\paragraph{Setup and Formulations.}
We work in the $D$-dimensional (feature embedding) space and consider distributions induced by generators.
Let $\mathcal F,\mathcal G$ be classes of measurable maps $\rvf,\rvg:\mathbb R^m\to\mathbb R^D$, and define the pushforward distributions
\[
q_{\rvf}:=\rvf_\#\mathcal N(\mathbf 0,\rmI),
\qquad
q_{\rvg}:=\rvg_\#\mathcal N(\mathbf 0,\rmI),
\]
on $\mathbb R^D$. Let $p$ denote the data distribution in the same space.
To match the dimension-dependent scaling used in drifting implementations (and to avoid kernel degeneracy in high dimension), we use the bandwidth
\[
\tau=\tau_D:=\bar\tau D^a,
\qquad \bar\tau>0,\ a\ge 0,
\]
so $\tau$ grows with $D$ when $a>0$ and remains constant when $a=0$.
Given a kernel $k_\tau$ (e.g., the Laplace kernel $k_\tau(\rvx,\rvy)=\exp(-\|\rvx-\rvy\|_2/\tau)$ as a special case of \Cref{eq:radial_kernel}), recall the mean-shift direction
\[
\mathbf V_{\uppi,k_\tau}(\rvx)
:=
\frac{\mathbb E_{\rvy\sim \uppi}\!\big[k_\tau(\rvx,\rvy)\,(\rvy-\rvx)\big]}
{\mathbb E_{\rvy\sim \uppi}\!\big[k_\tau(\rvx,\rvy)\big]},
\]
and the kernel-induced score (at scale $\tau$) $\rvs_{\uppi,\tau}(\rvx):=\nabla_{\rvx}\log \uppi_{k_\tau}(\rvx)$ with $\uppi_{k_\tau}(\rvx)=\mathbb E_{\rvy\sim\uppi}[k_\tau(\rvx,\rvy)]$.
We measure how well a generator $\rvf$ matches $p$ by the \emph{drifting field}
\[
\Delta_{p,q_\rvf}(\rvx):=\mathbf V_{p,k_\tau}(\rvx)-\mathbf V_{q_{\rvf},k_\tau}(\rvx),
\]
and how well it matches the (smoothed) score by the \emph{score-mismatch}
\[
\Delta\rvs_{p, q_\rvf}(\rvx):=\rvs_{p,\tau}(\rvx)-\rvs_{q_{\rvf},\tau}(\rvx).
\]
The corresponding population objectives are
\[
\mathcal L_{\mathrm{drift}}(\rvf):=\mathbb E_{\rvx\sim q_{\rvf}}\|\Delta_{p,q_\rvf}(\rvx)\|_2^2,
\qquad
\mathcal L_{\mathrm{SM}}(\rvg):=\mathbb E_{\rvx\sim q_{\rvg}}\|\rvs_{p,\tau}(\rvx)-\rvs_{q_{\rvg},\tau}(\rvx)\|_2^2.
\]
Finally, the scale-$\tau$ Fisher divergence introduced above satisfies
\[
\mathcal D_{\mathrm{rF}}(p\|q_{\rvf})
=
\mathbb E_{\rvx\sim q_{\rvf}}\|\rvs_{p,\tau}(\rvx)-\rvs_{q_{\rvf},\tau}(\rvx)\|_2^2
=
\mathbb E_{\rvx\sim q_{\rvf}}\|\Delta\rvs_{p, q_\rvf}(\rvx)\|_2^2.
\]

\paragraph{Proof Roadmap.}
We write $\mathbf V_{\uppi,k_\tau}(\rvx)=B_\tau(\rvx)/A_\tau(\rvx)$, where
\[
A_\tau(\rvx):=\int_{\mathbb R^D} e^{-\|\rvx-\rvy\|_2/\tau}\,\uppi(\rvy)\,\diff\rvy,
\qquad
B_\tau(\rvx):=\int_{\mathbb R^D} e^{-\|\rvx-\rvy\|_2/\tau}\,(\rvy-\rvx)\,\uppi(\rvy)\,\diff\rvy.
\]
After the change of variables $\rvy=\rvx+\tau\rvz$, both $A_\tau$ and $B_\tau$ become Laplace-weighted local averages of $\uppi(\rvx+\tau\rvz)$.
We then Taylor expand $\uppi(\rvx+\tau\rvz)$ around $\rvx$ and control the far-tail region by the exponential kernel decay, yielding expansions of $A_\tau$, $B_\tau$, and $\nabla A_\tau$ up to order $\tau^4$.
Finally, a ratio expansion gives
\[
\mathbf V_{\uppi,k_\tau}(\rvx)
=
c_D\tau^2\nabla\log A_\tau(\rvx)+\tau^4\mathbf R_{\uppi,\tau}(\rvx)
=
c_D\tau^2\rvs_{\uppi,\tau}(\rvx)+\tau^4\mathbf R_{\uppi,\tau}(\rvx),
\]
and the local envelope condition ensures the remainder stays uniformly controlled after dividing by $A_\tau(\rvx)$.


\paragraph{Technical Assumptions.}
We follow the \emph{instance-noise} convention: whenever we evaluate drifting or score-matching objectives,
we implicitly add a small Gaussian perturbation. This ensures all relevant laws admit smooth, strictly
positive densities, avoiding singular pushforwards without changing the practical learning setup.

\begin{assumption}[Instance Noise]
\label{ass:ac}
We observe $\tilde{\rvx}=\rvx+\boldsymbol{\xi}$ where
$\boldsymbol{\xi}\sim\mathcal N(\mathbf 0,\eta^2\rmI)$ is independent and $\eta>0$ is fixed.
Equivalently, for every distribution $\mu$ in
\[
\Big\{p\Big\}\ \cup\ \Big\{q_{\rvf}:\rvf\in\mathcal F\Big\}\ \cup\ \Big\{q_{\rvg}:\rvg\in\mathcal G\Big\},
\]
we work with its noise-regularized version $\mu*\mathcal N(\mathbf 0,\eta^2\rmI)$.
In particular, each effective $p,q_{\rvf},q_{\rvg}$ admits a $C^\infty$ Lebesgue density that is
strictly positive everywhere on $\mathbb R^D$.
\end{assumption}
Throughout this subsection, we keep the same notation $p,q_{\rvf},q_{\rvg}$ for these effective
(instance-noised) laws (i.e., we drop the convolution notation). Moreover, we notice that there is a uniform $\|\cdot\|_\infty$ bound from instance noise. If $\nu=\mu*\mathcal N(\mathbf 0,\eta^2\rmI)$ for a probability measure $\mu$, then
\[
\|\nu\|_\infty \le (2\pi\,\eta^2)^{-D/2}.
\]
Indeed, writing $\varphi_\eta(\rvx)=(2\pi\,\eta^2)^{-D/2}\exp(-\|\rvx\|_2^2/(2\eta^2))$ for the Gaussian density,
\[
\nu(\rvx)=\int \varphi_\eta(\rvx-\rvy)\,\diff\mu(\rvy)\le \sup_{\rvz}\varphi_\eta(\rvz)=\varphi_\eta(\mathbf 0).
\]

As stated above, we write $\mathbf V_{\uppi,k_\tau}(\rvx)=B_\tau(\rvx)/A_\tau(\rvx)$.
A small-$\tau$ expansion is local: after $\rvy=\rvx+\tau\rvz$, the kernel weight becomes
$e^{-\|\rvz\|_2}$, so the main contribution comes from $\|\rvz\|_2=\mathcal O(1)$
(equivalently, $\|\rvy-\rvx\|_2=\mathcal O(\tau)$).
The only subtlety is that our ratio expansions divide by $A_\tau(\rvx)$, whose leading term is
$A_\tau(\rvx)\approx M_0\tau^D\,\uppi(\rvx)$ as $\tau\to 0$.
Thus we must control how small $\uppi(\rvx)$ can be relative to its local neighborhood.
We encode this via a local density ratio as stated in the following assumption.

\begin{assumption}[Uniform Integrable Local Envelope along Drifting Models' Minimizers]
\label{ass:small_tau_envelope_uniform}
Assume \Cref{ass:ac}. For any effective density $\uppi$ define:
\[
\mathfrak M_{\uppi}(\rvx)
:=1+\sum_{1\le |\alpha|\le 4}\ \sup_{\|\rvu-\rvx\|_2\le 1}\ \bigl|\partial^\alpha \log \uppi(\rvu)\bigr|,
\qquad
\mathfrak R_{\uppi}(\rvx)
:=\sup_{\|\rvu-\rvx\|_2\le 1}\ \frac{\uppi(\rvu)}{\uppi(\rvx)} \in [1,\infty),
\]
and the combined local envelope (for an integer $K\ge 1$, e.g.\ $K=4$)
\[
\mathfrak U_{\uppi}(\rvx) := (1+\mathfrak M_\uppi(\rvx))^{K}\,(1+\mathfrak R_\uppi(\rvx)).
\]

Fix $\tau_0\in(0,1]$. Assume that for every $\tau\in(0,\tau_0]$ the population minimizer set
$\arg\min_{\rvf\in\mathcal F}\mathcal L_{\mathrm{drift}}(\rvf)$ is nonempty.
Assume there exists an integer $K\ge 1$ (one may take $K=4$) such that
\[
\sup_{\tau\in(0,\tau_0]}\ \sup_{\rvf^\star\in\arg\min_{\rvf}\mathcal L_{\mathrm{drift}}(\rvf)}\
\mathbb E_{\mathbf X\sim q_{\rvf^\star}}
\Big[\mathfrak U_{p}(\mathbf X)^{2}+\mathfrak U_{q_{\rvf^\star}}(\mathbf X)^{2}\Big]
<\infty.
\]
\end{assumption}

This density-ratio factor is mild but necessary. The denominator of $\mathbf V_{\uppi,k_\tau}(\rvx)$,
\[
A_\tau(\rvx)=\mathbb E_{\rvy\sim\uppi}\!\big[k_\tau(\rvx,\rvy)\big],
\]
satisfies $A_\tau(\rvx)\asymp \tau^D\,\uppi(\rvx)$ as $\tau\to 0$ (under the local regularity conditions used in the small-$\tau$ expansion).
Hence any pointwise expansion of
$\mathbf V_{\uppi,k_\tau}(\rvx)=B_\tau(\rvx)/A_\tau(\rvx)$
must control division by $\uppi(\rvx)$.
The local ratio $\mathfrak R_\uppi(\rvx)$ quantifies how $\uppi(\rvx)$ compares to nearby values in
$\mathbb B(\rvx,1)$, and it avoids imposing any global lower bound or compact support.
Only moments under the drifting minimizers are required.

\begin{assumption}[Realizability in each class]
\label{ass:realizable_fixed}
There exist $\rvf_{\mathrm{data}}\in\mathcal F$ and $\rvg_{\mathrm{data}}\in\mathcal G$ such that
$q_{\rvf_{\mathrm{data}}}=p$ and $q_{\rvg_{\mathrm{data}}}=p$.
\end{assumption}

\noindent
We measure agreement via the Fisher divergence between \emph{kernel-smoothed scores}:
\[
\mathcal D_{\mathrm{rF}}(p\|q)
:=
\mathbb E_{\rvx\sim q}\Big[\big\|\rvs_{p,\tau}(\rvx)-\rvs_{q,\tau}(\rvx)\big\|_2^2\Big],
\qquad
\rvs_{\uppi,\tau}(\rvx):=\nabla_{\rvx}\log \uppi_{k_\tau}(\rvx),
\quad
\uppi_{k_\tau}(\rvx):=\mathbb E_{\rvy\sim \uppi}[k_\tau(\rvx,\rvy)].
\]


\subsection{Auxiliary Tools and the Main Proof}

For the Laplace kernel $k_\tau(\rvx,\rvy)=\exp(-\|\rvx-\rvy\|_2/\tau)$, define
\[
M_0:=\int_{\mathbb R^D} e^{-\|\rvz\|_2}\,\diff\rvz,
\qquad
M_2:=\int_{\mathbb R^D} z_1^2\, e^{-\|\rvz\|_2}\,\diff\rvz,
\qquad
c_D:=\frac{M_2}{M_0}\in(0,\infty).
\]

\begin{lemma}[Small-$\tau$ Expansion: Mean Shift Matches Kernel Score up to $\tau^4$]
\label{lem:small_tau_ms_score}
Assume \Cref{ass:ac,ass:small_tau_envelope_uniform}.
Fix any $\tau\in(0,\tau_0]$ and any minimizer
$\rvf^\star\in\arg\min_{\rvf\in\mathcal F}\mathcal L_{\mathrm{drift}}(\rvf)$, depending on $\tau$.
For $\uppi\in\{p,q_{\rvf^\star}\}$, there exists $C<\infty$ (depending only on $D,\eta$ and kernel moments,
but \emph{not} on $\tau$) such that for all $\rvx\in\mathbb R^D$,
\[
\mathbf V_{\uppi,k_\tau}(\rvx)
=
c_D\,\tau^2\,\rvs_{\uppi,\tau}(\rvx)
\;+\;
\tau^4\,\mathbf R_{\uppi,\tau}(\rvx),
\qquad
\|\mathbf R_{\uppi,\tau}(\rvx)\|_2 \le C\,\mathfrak U_\uppi(\rvx).
\]
Consequently,
\[
\sup_{\tau\in(0,\tau_0]}\ \sup_{\rvf^\star\in\arg\min \mathcal L_{\mathrm{drift}}}\
\mathbb E_{\mathbf X\sim q_{\rvf^\star}}
\bigl[\|\mathbf R_{\uppi,\tau}(\mathbf X)\|_2^2\bigr]
<\infty.
\]
\end{lemma}

\begin{proof}
Fix $\uppi$ and $\rvx\in\mathbb R^D$. Define
\[
A_\tau(\rvx):=\int_{\mathbb R^D} e^{-\|\rvx-\rvy\|_2/\tau}\,\uppi(\rvy)\,\diff\rvy,
\qquad
B_\tau(\rvx):=\int_{\mathbb R^D} e^{-\|\rvx-\rvy\|_2/\tau}\,(\rvy-\rvx)\,\uppi(\rvy)\,\diff\rvy.
\]
Then $\uppi_{k_\tau}(\rvx)=A_\tau(\rvx)$ and
$\mathbf V_{\uppi,k_\tau}(\rvx)=B_\tau(\rvx)/A_\tau(\rvx)$.

First, we examine the differentiability of $A_\tau$ (Laplace kernel).
Although $k_\tau(\rvx,\rvy)$ is not classically $C^1$ at $\rvx=\rvy$,
the map $\rvx\mapsto k_\tau(\rvx,\rvy)$ belongs to $W^{1,1}(\mathbb R^D)$ (in $\rvx$)
and admits an $L^1$ weak gradient, given for $\rvx\neq \rvy$ by
\[
\nabla_{\rvx}k_\tau(\rvx,\rvy)
=
-\frac{1}{\tau}\,e^{-\|\rvx-\rvy\|_2/\tau}\,
\frac{\rvx-\rvy}{\|\rvx-\rvy\|_2},
\]
and we set $\nabla_{\rvx}k_\tau(\rvy,\rvy):=\mathbf 0$.
Hence $A_\tau = k_\tau * \uppi \in C^1$ and
\[
\nabla A_\tau(\rvx)=\int_{\mathbb R^D}\nabla_{\rvx}k_\tau(\rvx,\rvy)\,\uppi(\rvy)\,\diff\rvy,
\qquad
\rvs_{\uppi,\tau}(\rvx)=\nabla\log A_\tau(\rvx)=\frac{\nabla A_\tau(\rvx)}{A_\tau(\rvx)}.
\]

Now, we apply the change of variables. Set $\rvy=\rvx+\tau\rvz$ to get
\[
A_\tau(\rvx)=\tau^D\int_{\mathbb R^D} e^{-\|\rvz\|_2}\,\uppi(\rvx+\tau\rvz)\,\diff\rvz,
\qquad
B_\tau(\rvx)=\tau^{D+1}\int_{\mathbb R^D} e^{-\|\rvz\|_2}\,\rvz\,\uppi(\rvx+\tau\rvz)\,\diff\rvz.
\]

\paragraph{Step 1: Local Taylor Expansion on $\|\rvz\|\le \tau^{-1}$.}
If $\|\rvz\|\le \tau^{-1}$ then $\rvx+t\tau\rvz\in \mathbb \mathbb B\rvx,1)$ for all $t\in[0,1]$.
Taylor's theorem (order $3$ with integral remainder) yields
\[
\uppi(\rvx+\tau\rvz)
=
\uppi(\rvx)
+\tau\langle \nabla\uppi(\rvx),\rvz\rangle
+\frac{\tau^2}{2}\rvz^\top \nabla^2\uppi(\rvx)\rvz
+\frac{\tau^3}{6}\sum_{i,j,k}\partial_{ijk}\uppi(\rvx)\,z_iz_jz_k
+\tau^4\,\widetilde R_4(\rvx,\tau\rvz),
\]
where $\widetilde R_4(\rvx,\tau\rvz)$ is a linear combination of $\partial^\alpha \uppi(\rvx+t\tau\rvz)$
with $|\alpha|=4$ times monomials of degree $4$ in $\rvz$.
Using $\partial^\alpha \uppi = \uppi\cdot P_\alpha(\{\partial^\beta\log\uppi\}_{1\le|\beta|\le 4})$ and the definitions of
$\mathfrak M_\uppi,\mathfrak R_\uppi$, we obtain for $\|\rvz\|\le \tau^{-1}$:
\[
\bigl|\widetilde R_4(\rvx,\tau\rvz)\bigr|
\;\le\;
C_1\,
\uppi(\rvx)\,\mathfrak R_\uppi(\rvx)\,
(1+\mathfrak M_\uppi(\rvx))^{K}\,
\|\rvz\|_2^4,
\]
for a constant $C_1$ depending only on $D$.

Plugging into $A_\tau,B_\tau$ and using symmetry of $e^{-\|\rvz\|_2}$ (odd moments vanish) gives
\begin{align*}
A_\tau(\rvx)
&=\tau^D\Big(M_0\,\uppi(\rvx)+\frac{\tau^2}{2}M_2\,\Delta \uppi(\rvx)\Big)
+\tau^{D+4}\,a_{\uppi,\tau}(\rvx),\\
B_\tau(\rvx)
&=\tau^{D+2}\Big(M_2\,\nabla\uppi(\rvx)\Big)
+\tau^{D+4}\,b_{\uppi,\tau}(\rvx),
\end{align*}
where $\Delta=\sum_{i=1}^D\partial_{ii}$ and
\[
|a_{\uppi,\tau}(\rvx)|+\|b_{\uppi,\tau}(\rvx)\|
\;\le\;
C_2\,\uppi(\rvx)\,\mathfrak R_\uppi(\rvx)\,(1+\mathfrak M_\uppi(\rvx))^{K},
\qquad
(\tau\in(0,\tau_0]).
\]

\paragraph{Step 2: Tail Region $\|\rvz\|>\tau^{-1}$.}
Since $\int_{\|\rvz\|>\tau^{-1}} e^{-\|\rvz\|_2}\,\diff\rvz \lesssim e^{-1/\tau}\tau^{-(D-1)}$
and $\|\uppi\|_\infty<\infty$ under instance noise, the tail contributions to $A_\tau,B_\tau$ are bounded by
$e^{-1/\tau}$ times a polynomial in $1/\tau$.
Hence they are $o(\tau^N)$ for every $N$ and can be absorbed into the $\tau^{D+4}$ remainders
by shrinking $\tau_0$ if needed.

\paragraph{Step 3: Expansion for $\nabla A_\tau$ and a Ratio Identity.}
By the same decomposition (main region + tail) we obtain
\[
\nabla A_\tau(\rvx)
=\tau^D\Big(M_0\,\nabla\uppi(\rvx)+\frac{\tau^2}{2}M_2\,\nabla\Delta \uppi(\rvx)\Big)
+\tau^{D+4}\,\tilde a_{\uppi,\tau}(\rvx),
\quad
\|\tilde a_{\uppi,\tau}(\rvx)\|
\le
C_3\,\uppi(\rvx)\,\mathfrak R_\uppi(\rvx)\,(1+\mathfrak M_\uppi(\rvx))^{K}.
\]

We use the explicit algebraic identity: if $A=a_0+\tau^2 a_2+\tau^4 a_4$ with $a_0>0$ and
$B=\tau^2(b_2+\tau^2 b_4)$, then
\[
\frac{B}{A}
= \tau^2\frac{b_2}{a_0}
+ \tau^4\frac{b_4 a_0-b_2 a_2}{a_0\,(a_0+\tau^2 a_2+\tau^4 a_4)}.
\]
Applying this to $B_\tau/A_\tau$ and to $\nabla A_\tau/A_\tau$, using $a_0(\rvx)=M_0\uppi(\rvx)$,
$b_2(\rvx)=M_2\nabla\uppi(\rvx)$, and $c_D=M_2/M_0$, yields
\[
\frac{B_\tau(\rvx)}{A_\tau(\rvx)}
=
c_D\,\tau^2\,\frac{\nabla A_\tau(\rvx)}{A_\tau(\rvx)}
+\tau^4\,\mathbf R_{\uppi,\tau}(\rvx)
=
c_D\,\tau^2\,\rvs_{\uppi,\tau}(\rvx)
+\tau^4\,\mathbf R_{\uppi,\tau}(\rvx).
\]
Using the bounds on $a_{\uppi,\tau},b_{\uppi,\tau},\tilde a_{\uppi,\tau}$ above together with
\[
\|\nabla\uppi(\rvx)\|\le \uppi(\rvx)\sup_{\|\rvu-\rvx\|\le 1}\|\nabla\log\uppi(\rvu)\|
\le \uppi(\rvx)\,\mathfrak M_\uppi(\rvx),
\quad
|\Delta\uppi(\rvx)|\le \uppi(\rvx)\,C(1+\mathfrak M_\uppi(\rvx))^2,
\]
one checks that all factors of $\uppi(\rvx)$ cancel in the ratio remainder, and thus
\[
\|\mathbf R_{\uppi,\tau}(\rvx)\|
\le C\, (1+\mathfrak M_\uppi(\rvx))^{K}\,(1+\mathfrak R_\uppi(\rvx))
= C\,\mathfrak U_\uppi(\rvx),
\]
for a constant $C$ depending only on $D,\eta$ and kernel moments.
This proves the pointwise expansion and bound.

Finally, the uniform $L^2$ statement follows immediately from \Cref{ass:small_tau_envelope_uniform}.
\end{proof}


\paragraph{Main Proof to \Cref{thm:small_bar_tau_consistency_prime}.}

Fix $a\ge 0$ and $\tau=\tau_D=\bar\tau D^a$. For each $\tau\in(0,\tau_0]$, let
\[
\rvf^{\star}(\tau)\in\arg\min_{\rvf\in\mathcal F} \mathcal{L}_{\mathrm{drift}}(\rvf),
\qquad
\rvg^{\star}(\tau)\in\arg\min_{\rvg\in\mathcal G} \mathcal{L}_{\mathrm{SM}}(\rvg).
\]

We first show that drifting optimality implies equality of mean-shift fields.
By \Cref{ass:realizable_fixed}, there exists $\rvf_{\mathrm{data}}\in\mathcal F$ with $q_{\rvf_{\mathrm{data}}}=p$.
Then $\mathcal L_{\mathrm{drift}}(\rvf_{\mathrm{data}})=0$ for every $\tau$, hence
$\min_{\rvf}\mathcal L_{\mathrm{drift}}(\rvf)=0$ and therefore
$\mathcal L_{\mathrm{drift}}(\rvf^\star(\tau))=0$.
Equivalently,
\[
\mathbf V_{p,k_\tau}(\rvx)=\mathbf V_{q_{\rvf^\star(\tau)},k_\tau}(\rvx)
\qquad\text{for }q_{\rvf^\star(\tau)}\text{-a.e.\ }\rvx.
\]

We then convert mean-shift equality into a bound on score-mismatch.
Apply \Cref{lem:small_tau_ms_score} with $\uppi=p$ and $\uppi=q_{\rvf^\star(\tau)}$.
For $q_{\rvf^\star(\tau)}$-a.e.\ $\rvx$,
\[
c_D\,\tau^2\,\rvs_{p,\tau}(\rvx)+\tau^4\mathbf R_{p,\tau}(\rvx)
=
c_D\,\tau^2\,\rvs_{q_{\rvf^\star(\tau)},\tau}(\rvx)+\tau^4\mathbf R_{q_{\rvf^\star(\tau)},\tau}(\rvx),
\]
hence
\[
\Delta\rvs_{\rvf^\star(\tau)}(\rvx)
=
\rvs_{p,\tau}(\rvx)-\rvs_{q_{\rvf^\star(\tau)},\tau}(\rvx)
=
\frac{\tau^2}{c_D}\Big(\mathbf R_{q_{\rvf^\star(\tau)},\tau}(\rvx)-\mathbf R_{p,\tau}(\rvx)\Big).
\]

Now we square and integrate under $q_{\rvf^\star(\tau)}$.
Let $\mathbf X\sim q_{\rvf^\star(\tau)}$. Using $(a+b)^2\le 2a^2+2b^2$,
\[
\mathcal D_{\mathrm{rF}}(p\|q_{\rvf^\star(\tau)})
=
\mathbb E\big[\|\Delta\rvs_{\rvf^\star(\tau)}(\mathbf X)\|_2^2\big]
\le
\frac{2\tau^4}{c_D^2}\,
\mathbb E\big[\|\mathbf R_{q_{\rvf^\star(\tau)},\tau}(\mathbf X)\|_2^2+\|\mathbf R_{p,\tau}(\mathbf X)\|_2^2\big].
\]
By \Cref{ass:small_tau_envelope_uniform} and \Cref{lem:small_tau_ms_score}, the expectation on the right-hand side
is bounded uniformly for $\tau\in(0,\tau_0]$, hence
\[
\mathcal D_{\mathrm{rF}}(p\|q_{\rvf^\star(\tau)})=\mathcal O(\tau^4),
\qquad \tau\to 0.
\]
Since $D$ is fixed and $\tau=\bar\tau D^a$, we have $\tau^4=D^{4a}\bar\tau^4$, so equivalently
\[
\mathcal D_{\mathrm{rF}}(p\|q_{\rvf^\star(\tau)})=\mathcal O(\bar\tau^4),
\qquad \bar\tau\to 0,
\]
with a constant independent of $\bar\tau$.

At last, we compare to the score-matching minimizer. By realizability and population identification of score matching (as used elsewhere in your paper),
$q_{\rvg^\star(\tau)}=p$. Therefore
\[
\mathcal D_{\mathrm{rF}}(q_{\rvg^\star(\tau)}\|q_{\rvf^\star(\tau)})
=
\mathcal D_{\mathrm{rF}}(p\|q_{\rvf^\star(\tau)})
=
\mathcal O(\bar\tau^4).
\]

\end{document}